\documentclass{article}
\usepackage[utf8]{inputenc}
\usepackage[margin=1.25in]{geometry}
\usepackage{graphicx}
\usepackage{amsmath,amsfonts,amssymb,amsthm}
\usepackage[linesnumbered,ruled,vlined]{algorithm2e}
\usepackage{thmtools}
\usepackage{mathtools}
\usepackage{stackengine}
\usepackage{eucal}
\usepackage{natbib}
\usepackage{enumitem}
\usepackage{multirow}
\usepackage{float}
\usepackage{blkarray}
\usepackage{bbm}
\usepackage{tikz}
\usepackage{hyperref}
\usepackage{comment}
\usepackage{trimclip}
\usepackage{authblk}

\newtheorem{thm}{Theorem}[section]
\newtheorem{prop}{Proposition}[section]
\newtheorem{lem}{Lemma}[section]
\newtheorem{cor}{Corollary}[section]

\theoremstyle{definition}
\newtheorem*{defn}{Definition}

\let \mc = \mathcal
\let \ms = \mathsf

\let \mbb = \mathbb
\let \eq = \equiv
\let \sub = \subset
\let \sube = \subseteq
\let \sm = \setminus
\let \es = \varnothing
\let \ra = \rightarrow
\let \la = \leftarrow
\let \lra = \leftrightarrow
\let \Ra = \Rightarrow
\let \La = \Leftarrow
\let \Lra = \Leftrightarrow

\newcommand{\g}[1][]{\ensuremath{\mc G_{#1}}}
\newcommand{\pa}[2]{\ensuremath{\textup{pa}_{#1}(#2)}}
\newcommand{\ch}[2]{\ensuremath{\textup{ch}_{#1}(#2)}}
\newcommand{\sib}[2]{\ensuremath{\textup{sib}_{#1}(#2)}}
\newcommand{\an}[2]{\ensuremath{\textup{an}_{#1}(#2)}}
\newcommand{\de}[2]{\ensuremath{\textup{de}_{#1}(#2)}}
\newcommand{\dis}[2]{\ensuremath{\textup{dis}_{#1}(#2)}}
\newcommand{\ba}[2]{\ensuremath{\textup{bar}_{#1}(#2)}}
\newcommand{\co}[2]{\ensuremath{\textup{col}_{#1}(#2)}}
\newcommand{\mb}[2]{\ensuremath{\textup{mb}_{#1}(#2)}}
\newcommand{\cl}[2]{\ensuremath{\textup{cl}_{#1}(#2)}}

\newcommand{\pre}[2]{\ensuremath{\textup{pre}^{\leq}_{#1}(#2)}}
\newcommand{\ml}[2]{\ensuremath{\textup{ml$\s$s}^{\leq}_{#1}(#2)}}
\newcommand{\ta}[2]{\ensuremath{\textup{tail}_{#1}(#2)}}

\newcommand{\seq}[1]{\ensuremath{\langle #1 \rangle}}
\newcommand{\flo}[1]{\ensuremath{\lfloor #1 \rfloor_{\leq}}}
\newcommand{\ceo}[1]{\ensuremath{\lceil #1 \rceil_{\leq}}}
\newcommand{\floor}[1]{\ensuremath{\lfloor #1 \rfloor}}
\newcommand{\ceil}[1]{\ensuremath{\lceil #1 \rceil}}

\newcommand{\s}[1][0.5]{\ensuremath{\mkern #1 mu}}
\newcommand{\istate}[4]{\ensuremath{#1 \perp \s[-9] \perp #2 \mid #3 \; [\s #4 \s]}}
\newcommand{\dstate}[4]{\ensuremath{#1 \not \s[-3] \perp \s[-9] \perp #2 \mid #3 \; [\s #4 \s]}}

\newcommand{\dom}[1]{\ensuremath{\textup{dom}( #1, \leq )}}

\DeclareMathOperator*{\argmax}{arg\s[2]max}

\title{The \textit{m}-connecting imset and factorization for ADMG models}

\author[1]{Bryan Andrews}
\author[2]{Gregory F. Cooper}
\author[3]{Thomas S. Richardson}
\author[1]{Peter Spirtes}
\affil[1]{Department of Philosophy, Carnegie Mellon University}
\affil[2]{Department of Biomedical Informatics, University of Pittsburgh}
\affil[3]{Department of Statistics, University of Washington}

\date{\today}

\begin{document}

\maketitle

\begin{abstract}
    \noindent Directed acyclic graph (DAG) models have become widely studied and applied in statistics and machine learning---indeed, their simplicity facilitates efficient procedures for learning and inference. Unfortunately, these models are not closed under marginalization, making them poorly equipped to handle systems with latent confounding. Acyclic directed mixed graph (ADMG) models characterize margins of DAG models, making them far better suited to handle such systems. However, ADMG models have not seen wide-spread use due to their complexity and a shortage of statistical tools for their analysis. 
    
    In this paper, we introduce the \textit{m}-connecting imset which provides an alternative representation for the independence models induced by ADMGs. Furthermore, we define the \textit{m}-connecting factorization criterion for ADMG models, characterized by a single equation, and prove its equivalence to the global Markov property. The \textit{m}-connecting imset and factorization criterion provide two new statistical tools for learning and inference with ADMG models. We demonstrate the usefulness of these tools by formulating and evaluating a consistent scoring criterion with a closed form solution.
\end{abstract}

\textit{Keywords:} ADMG models, imsetal conditional independence, factorization criterion

\section{Introduction}
\label{sec:intro}

Directed acyclic graph (DAG) models, also known as Bayesian networks, have become widely studied and applied in statistics and machine learning---indeed, their simplicity facilitates efficient procedures for learning and inference \citep{bishop2006pattern,darwiche2009modeling,koller2009probabilistic}. Furthermore, DAG models allow for causal inference under a few additional assumptions \citep{spirtes2000causation, pearl2009causality}. These models describe probability distributions that obey the conditional independence relations represented in a DAG, and may be characterized by a well-known recursive factorization. Moreover, multinomial and Gaussian DAG models form curved exponential families with known dimension \citep{geiger2001stratified}. While variables may be added to DAG models to account for latent confounding, doing so leads to several challenges \citep{richardson2002ancestral}: 
\begin{itemize}
    \item these models are not always identifiable;
    \item inference may be sensitive to assumptions made about the latent variables;
    \item these models are not curved exponential families \cite{geiger2001stratified};
    \item distributions associated with one of these models may be difficult to characterize \cite{settimi1998geometry, settimi2000geometry}.
\end{itemize}

Alternatively, we could consider models that describe probability distributions obeying the conditional independence relations represented in margins of DAGs. In this paradigm, we treat latent confounding as marginalized variables, thereby making no assumptions about the latent variables. However, DAG models, are not closed under marginalization. As a result, several authors have explored more general families of graphical models that are closed under marginalization \citep{richardson2002ancestral, sadeghi2013stable}.

\begin{figure}[H]
    \centering
    \includegraphics[page=1]{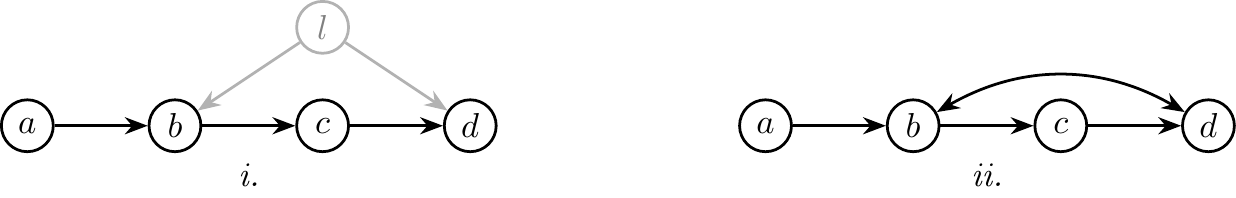}
    \caption{Graphical representations for a system with latent confounding: (\textit{i}) a DAG with an explicitly defined latent variable; (\textit{ii}) an ADMG for the induced marginal independence model.}
    \label{fig:dag_admg}
\end{figure}

In this work, we focus on acyclic directed mixed graph (ADMG) models due to their simple and intuitive derivation as latent projections of DAGs \citep{verma1990equivalence, evans2018markov}. Figure \ref{fig:dag_admg} illustrates a DAG with a latent variable and the corresponding ADMG. Notably, ADMGs generalize other families of graphical models that are closed under marginalization, such as directed (maximal) ancestral graphs \citep{sadeghi2013stable}. Accordingly, the results of this paper directly apply to those subfamilies.

ADMG models describe probability distributions that obey the conditional independence relations represented in an ADMG. Auspiciously, multinomial and Gaussian ADMG models form curved exponential families with known dimension \citep{richardson2002ancestral, evans2014markovian}. However, these models have not seen wide-spread use due to their complexity and a shortage of statistical tools for their analysis. Similarly, \cite{studeny2005probabilistic} introduced structural imsets as a non-graphical representation of conditional independence which, while capable of representing all probabilistic independence models, have not seen wide-spread use. Developments to the imsetal representation of conditional independence have been largely theoretical, except for in constrained scenarios, such as the restriction to DAG independence models---indeed, several methods use imsets to learn DAGs \citep{studeny2008mathematical, studeny2017towards, studeny2014learning}. A recent innovation in the space of imsetal independence models introduced the characteristic imset as a more natural representation of DAG independence models \citep{hemmecke2012characteristic}.

In this paper we introduce the \textit{m}-connecting imset, which is analogous to the characteristic imset for DAG models, in order to expand the tools available for analysis of ADMG models. Indeed, we show how the \textit{m}-connecting imset admits a novel factorization criterion, characterized by a single equation, that is equivalent to the global Markov property---we call this new criterion the \textit{m}-connecting factorization. Using the \textit{m}-connecting factorization, we formulate a consistent scoring criterion with a closed-form solution for curved exponential families whose independence models are described by ADMGs. We investigate asymptotic properties of our score and compare its ability to recover correct independence models against the well-known FCI algorithm and two of its variants

This paper details the methods developed in a recent PhD thesis \citep{andrews2022inducing} and is organized as follows. Section \ref{sec:back} introduces and discusses concepts relevant to ADMGs, imsets, and independence models. Section \ref{sec:m-imset} introduces the \textit{m}-connecting imset and factorization and proves the equivalence of the \textit{m}-connecting factorization to the global Markov property for ADMG models. Section \ref{sec:consist_score} formulates a consistent scoring criterion for ADMG curved exponential families and investigates its ability to recover correct independence models against the well-known FCI algorithm and two of its variants. Section \ref{sec:discussion} provides a discussion and states our conclusions. For readability, all proofs are postponed until the appendix.

\section{Background}
\label{sec:back}

In this section, we introduce and discuss concepts relevant to ADMGs, imsets, and independence models in order to facilitate the formulation of the \textit{m}-connecting imset and factorization; partially ordered sets and the M{\"o}bius inversion are of particular importance for the latter. Throughout this paper, $V$ denotes a non-empty finite set of variables that act as vertices in graphical contexts. We use the following notational conventions:
\begin{itemize}
    \item{\makebox[32mm]{$a,b,c,\ldots \: \in V$\hfill}} lowercase letters denote variables or singletons;
    \item{\makebox[32mm]{$A,B,C,\ldots \: \sube V$\hfill}} uppercase letters denote sets of variables; 
    \item{\makebox[32mm]{$\ms A,\ms B,\ms C,\ldots \: \sube \mc P(V)$\hfill}} uppercase sans-serif letters denote sets of sets;
\end{itemize}
where $\mc P(V)$ denotes the power set of $V$. Additionally, juxtapositions of sets denote their unions (e.g., $AB = A \cup B$).

\subsection{Acyclic Directed Mixed Graphs}

Let $\mc E(V)$ be a set of ordered and unordered pairs:
\[
    \mc E(V) \eq \{ (a,b) \; : \; a,b \in V \; (a \neq b) \} \; \cup \; \{ \{a,b\} \; : \; a,b \in V \; (a \neq b) \}.
\]
 
\begin{defn}[\textit{directed mixed graph}]
    A \textit{directed mixed graph} is an ordered pair $\g = (V, E)$ where $E \sube \mc E(V)$. Graphically, the elements of $V$ act as vertices and the elements of $E$ act as directed and bi-directed edges. Accordingly, for $a,b \in V$ we write:
    \[
        \left\{ \!\!
        \begin{array}{c}
            a \ra b \\
            a \lra b
        \end{array}
         \!\! \right\}
         \: \text{in} \s[7] \g \s[7] \text{if} \:
         \left\{ \!\!
        \begin{array}{c}
            \s[-4] (a,b) \s[-4] \\
            \s[-4] \{ a,b \} \s[-4]
        \end{array}
         \!\! \right\} \in E.
    \]
\end{defn}

\noindent A directed mixed graph is \textit{mixed} in the sense that it contains a mixture of directed and bi-directed edges, and \textit{directed} in the sense that it does not contain undirected edges. This is not to be confused with a \textit{directed graph}, which is a graph that only contains directed edges. A directed mixed graph is \textit{acyclic} if it does not contain a sequence of the following form:
\[
    v_1 \ra \dots \ra v_n \ra v_1 \s[18] \text{for} \: n \geq 2.
\]

In this paper, we focus on acyclic directed mixed graphs (ADMGs)---directed mixed graphs that are acyclic. Next, we review a few graphical concepts used to describe ADMGs.

\begin{defn}[\textit{induced subgraph}]
    Let $\g = (V, E)$ be an ADMG and $A \sube V$. The \textit{induced subgraph} of $\g$ with respect to $A$ is the ordered pair:
    \[
        \g[A] \eq (A, E \cap \mc E(A)).
    \]
\end{defn}

\begin{defn}[\textit{path}]
    A \textit{path} $\pi = \seq{v_1, e_1, \dots, v_n}$ $(n \geq 2)$ is an alternating sequence of distinct vertices and edges where $e_i = (v_{i}, v_{i+1})$ or $e_i = \{ v_{i}, v_{i+1} \}$ for $1 \leq i < n$. The \textit{endpoints} of $\pi$ are the first and last vertices $\{ v_1, v_n \}$ and are \textit{adjacent} if $(v_{1}, v_{n}) \in E$ or $\{ v_{1}, v_{n} \} \in E$.
\end{defn}

\noindent In this paper, we use sequences of vertices to describe paths in ADMGs even if there is more than one edge between a pair of vertices---the intended edge for a pair of vertices will either be obvious from the context or irrelevant.

\begin{defn}[\textit{triple}]
    A \textit{triple} is a path $\pi = \seq{v_1,v_2,v_3}$ with three vertices and is \textit{unshielded} if its endpoints are not adjacent.
\end{defn}

\begin{defn}[\textit{collider}]
    A \textit{collider} on a path $\pi = \seq{v_1, \dots, v_n}$ is a vertex $v_i$ such that:
    \[
        \{ v_{i-1} \; {*\s[-1]\clipbox{1.5mm -1mm -1mm -1mm}{$\ra$}} v_i {\clipbox{-1mm -1mm 1.5mm -1mm}{$\la$}\s[-1]*} \; v_{i+1} \} \s[7] \text{in} \s[7] \g
    \]
    where asterisks are used to denote edge marks that may either be an arrowhead or a tail. Moreover, the collider is an \textit{unshielded collider} if the corresponding triple is unshielded.
\end{defn}

\begin{defn}[\textit{non-collider}]
    A \textit{non-collider} on a path $\pi = \seq{v_1, \dots, v_n}$ is a vertex $v_i$ such that:
    \[
        \left\{ \!\! \begin{array}{c}
        v_{i-1} \la v_i \ra v_{i+1} \\
        v_{i-1} \; {*\s[-1]\clipbox{1.5mm -1mm -1mm -1mm}{$\ra$}} v_i \ra v_{i+1} \\
        v_{i-1} \la v_i {\clipbox{-1mm -1mm 1.5mm -1mm}{$\la$}\s[-1]*} \; v_{i+1}
        \end{array} \!\! \right\} \: \text{in} \s[7] \g
    \]
    where asterisks are used to denote edge marks that may either be an arrowhead or a tail. Moreover, the non-collider is an \textit{unshielded non-collider} if the corresponding triple is unshielded.
\end{defn}

\noindent ADMGs arise naturally from DAGs with latent variables via a process called \textit{latent projection} \citep{verma1990equivalence, evans2018markov}. For example, the ADMG in Figure \ref{fig:dag_admg} (\textit{ii}) illustrates the latent projection of the DAG in Figure \ref{fig:dag_admg} (\textit{i}). Notably, every DAG is an ADMG. We define latent projection with respect to ADMGs.

\vskip 5mm

\begin{algorithm}[H]
    \caption{\textsc{Latent Projection} $\textsc{LP}(\g, L)$}
    \label{alg:lp}
    \KwIn{ADMG: $\g$, \, vertex set $L$}
    \KwOut{ADMG: $\g$}
    \ForEach{$l \in L$}{
        \ForEach{\textup{triple} \seq{a,l,b} \textup{in} $\g$}{
            \If{$a \ra l \ra b$ \textup{in} $\g$ \textup{and} $a \ra b$ \textup{not in} $\g$}{
                Add $a \ra b$ to $\g$ \;
            }
            \If{$\left\{ \!\!
            \begin{array}{c}
                a \la l \ra b \\
                a \lra l \ra b
            \end{array}
            \!\! \right\}$ \textup{in} \g[] \textup{and} $a \lra b$ \textup{not in} $\g$}{
                Add $a \lra b$ to $\g$ \;
            }
        }
        Remove $l$ from $\g$ \;
    }    
\end{algorithm}

\vskip 5mm

Next, we review standard terminology for describing how vertices relate in an ADMG. Let $\g = (V, E)$ be an ADMG and $a \in V$:
\begin{align*}
    \pa{\g}{a} &\eq \{ b \in V \; : \; b \ra a \s[7] \text{in} \s[7] \g\}; \\
    \ch{\g}{a} &\eq \{ b \in V \; : \; b \la a \s[7] \text{in} \s[7] \g\}; \\
    \sib{\g}{a} &\eq \{ b \in V \; : \; b \lra a \s[7] \text{in} \s[7] \g\}
\end{align*}
are the \textit{parents}, \textit{children}, and \textit{siblings} of $a$, respectively. Similarly:
\begin{align*}
    \an{\g}{a} &\eq \{ b \in V \; : \; b \ra \cdots \ra a \s[7] \text{or} \s[7] b \ra a \s[7] \text{in} \s[7] \g \s[7] \text{or} \s[7] a = b \}; \\
    \de{\g}{a} &\eq \{ b \in V \; : \; b \la \cdots \la a \s[7] \text{or} \s[7] b \la a \s[7] \text{in} \s[7] \g \s[7] \text{or} \s[7] a = b \}; \\
    \dis{\g}{a} &\eq \{ b \in V \; : \; b \lra \cdots \lra a \s[7] \text{or} \s[7] b \lra a \s[7] \text{in} \s[7] \g \s[7] \text{or} \s[7] a = b \}
\end{align*}
are the \textit{ancestors}, \textit{descendants}, and \textit{district} of $a$, respectively. These functions are applied disjunctively to sets, in other words, applying one to a set of vertices is the union of the operation applied to each vertex in the set. For example, $A \sube V$ has parents and ancestors:
\[
\pa{\g}{A} \eq \bigcup_{a \in A} \pa{\g}{a} \s[72] \an{\g}{A} \eq \bigcup_{a \in A} \an{\g}{a}.
\]
% Notably, we use inclusive definitions for the latter set of functions: $a \in \an{\g}{A}$, $a \in \de{\g}{A}$, and $a \in \dis{\g}{A}$. 
On the other hand:
\[
    \co{\g}{a} \eq \{ b \in V \; : \left\{ \!\!
    \begin{array}{c}
        b \ra \; \lra \cdots \lra \; \la a \\
        b \ra \; \lra \cdots \lra a \\
        b \lra \cdots \lra \; \la a \\
        b \lra \cdots \lra a
    \end{array}
    \!\! \right\} \: \text{or} \: \left\{ \!\!
    \begin{array}{c}
        b \ra \; \la a \\
        b \ra a \\
        b \la a \\
        b \lra a
    \end{array}
    \!\! \right\}
    \: \text{in} \s[7] \g \s[7] \text{or} \s[7] a = b \}
\]
are the \textit{collider-connecting vertices} of $a$. This function is applied conjunctively to sets, in other words, applying it to a set of vertices is the intersection of the operation applied to each vertex in the set. For example, a set of vertices $A \sube V$ has collider-connecting vertices:
\[
    \co{\g}{A} \eq \bigcap_{a \in A} \co{\g}{a}.
\]
The use of conjunction rather than disjunction ensures that the collider-connecting vertices of a set are contiguously connected by colliders on some path.

Lastly, we review a few special types of sets that will be useful throughout this paper.

\begin{defn}[\textit{ancestral set}] 
    Let $\g = (V, E)$ be an ADMG and $A \sube V$. If $A = \an{\g}{A}$, then $A$ is an \textit{ancestral set}. In other words, $A$ contains its own ancestors. The set of all ancestral sets is defined:
    \[
        \mc A(\g) \eq \{ A \in \mc P(V) \; : \; A = \an{\g}{A} \}.
    \]
\end{defn}

\noindent Figure \ref{fig:set_ex} illustrates the nontrivial ancestral sets of an ADMG. Notably, the concepts of induced subgraph and latent projection are equivalent for ancestral sets.

\begin{prop}[\cite{evans2016graphs}]
    %Proposition 5.2
    \label{prop:induce_margin}
    Let $\g = (V, E)$ be an ADMG and $A \sube V$. If $A \in \mc A(\g)$, then\textup{:}
    \[
        \g[A] = \textsc{LP}(\g, V \sm A).
    \]
\end{prop}

\begin{defn}[\textit{collider-connecting set}]
    Let $\g = (V, E)$ be an ADMG and $C \sube V$. If $C = \co{\g[C]}{C}$, then $C$ is a \textit{collider-connecting set}. In other words, the members of $C$ are collider-connecting by paths using only members of $C$. The set of all collider-connecting sets is defined:
    \[
    \mc C(\g) \eq \{ C \in \mc P(V) \; : \; C = \co{\g[C]}{C} \}.
    \]
\end{defn}

\noindent Figure \ref{fig:set_ex} illustrates the nontrivial collider-connecting sets of an ADMG.

\begin{defn}[\textit{barren subset}]
    Let $\g = (V,E)$ be an ADMG and $B \sube V$. If $A = \smash{\bigcup \limits_{b \in B}} \an{\g}{b} \sm b$, then the \textit{barren subset} of $B$ is defined:
    \[
        \ba{\g}{B} \eq B \sm A.
    \]
In other words, the barren subset of $B$ contains the members of $B$ that are not ancestors of other members of $B$.
\end{defn}

\noindent Figure \ref{fig:set_ex} illustrates the nontrivial barren subsets of an ADMG.

\begin{figure}[H]
    \begin{minipage}{.35\textwidth}
        \centering
        \includegraphics[page=2]{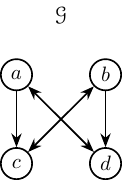}
        \vspace{10mm}
    \end{minipage} %
    \begin{minipage}{.6\textwidth}
        \centering
        \[
            \begin{blockarray}{c @{\hspace{5mm}} c @{\hspace{5mm}} c @{\hspace{5mm}} c}
                S & S \in \mc A(\g) & S \in \mc C(\g) & \ba{\g}{S} \\
                \begin{block}{c @{\hspace{5mm}} c @{\hspace{5mm}} c @{\hspace{5mm}} c}
                    \{a,b\} & 1 & 0 & \{a,b\} \\
                    \{a,c\} & 1 & 1 & \{c\} \\
                    \{a,d\} & 0 & 1 & \{a,d\} \\
                    \{b,c\} & 0 & 1 & \{b,c\} \\
                    \{b,d\} & 1 & 1 & \{d\} \\
                    \{c,d\} & 0 & 0 & \{c,d\} \\
                    \{a,b,c\} & 1 & 1 & \{b,c\} \\
                    \{a,b,d\} & 1 & 1 & \{a,d\} \\
                    \{a,c,d\} & 0 & 0 & \{c,d\} \\
                    \{b,c,d\} & 0 & 0 & \{c,d\} \\
                    \{a,b,c,d\} & 1 & 0 & \{c,d\} \\
                \end{block}
            \end{blockarray}
        \]
    \end{minipage}
    \caption{An ADMG and its nontrivial ancestral sets, collider-connecting sets, and barren subsets.}
    \label{fig:set_ex}
\end{figure}

\subsection{Partially Ordered Sets}

Throughout this paper, $\ms P$ denotes a finite \textit{partially ordered set}. Notably, the elements of $\ms P$ can be variables or sets of variables---hence our choice of notation.

\begin{defn}[\textit{partial order}]
    A \textit{partial order} is a binary relation $\leq$ over a set $\ms P$ such that $\leq$ is \textit{reflexive}, \textit{antisymmetric}, and \textit{transitive}. In other words, for every collection of elements $A,B,C \in \ms P$:
    \begin{enumerate}
        \item{\makebox[30mm]{reflexivity\hfill}} $A \leq A$;
        \item{\makebox[30mm]{antisymmetry\hfill}} $A \leq B$ and $B \leq A$ $\s[18] \Ra \s[18]$ $A = B$;
        \item{\makebox[30mm]{transitivity\hfill}} $A \leq B$ and $B \leq C$ $\s[18] \Ra \s[18]$ $A \leq C$.
    \end{enumerate}
\end{defn}

\begin{defn}[\textit{partially ordered set}]
    A \textit{partially ordered set}, \textit{poset} for short, is a set $\ms P$ with a partial order $\leq$. A pair of elements $A,B \in \ms P$ are \textit{comparable} if $A \leq B$ or $B \leq A$ and \textit{incomparable} otherwise. If every pair of elements is comparable, then $\leq$ is a \textit{total order}. 
\end{defn}

Two simple examples of partially ordered sets used throughout this paper are $V$ ordered by a total order \textit{consistent} with an ADMG $\g = (V,E)$ and $\mc P(V)$ ordered by \textit{inclusion}.

\begin{defn}[\textit{consistent order}]
    If $\g = (V, E)$ is an ADMG and $\leq$ is a total order where:
    \[
        a \leq b \s[18] \Ra \s[18] b \not \in \an{\g}{a} \s[18] \text{for all} \s[18] a,b \in V \: (a \neq b),
    \]
    then $\leq$ is a \textit{consistent} with $\g$.
\end{defn}

\begin{defn}[\textit{preceding vertices}]
    Let $\g = (V, E)$ be an ADMG with consistent order $\leq$. If $b \in V$, then the preceding vertices of $b$ with respect to $\leq$ are defined:
    \[
        \pre{\g}{b} \eq \{ a \in V \; : \; a \leq b \}.
    \]
    This function is applied disjunctively to sets, in other words, applying is to a set of vertices is the union of the operation applied to each vertex in the set. For example, $A \sube V$ has preceding vertices:
\[
    \pre{\g}{A} \eq \bigcup_{a \in A} \pre{\g}{a}.
\]
\end{defn}

\begin{defn}[\textit{inclusion order}]
If $\ms P = \mc P(V)$ and $\leq$ is the partial order where:
\[
    A \leq B \s[18] \Lra \s[18] A \sube B \s[18] \text{for all} \s[18] A,B \in \ms P
\]
then $\mc P(V)$ is ordered by \textit{inclusion}.
\end{defn}

\begin{figure}[H]
    \centering
    \includegraphics[page=3]{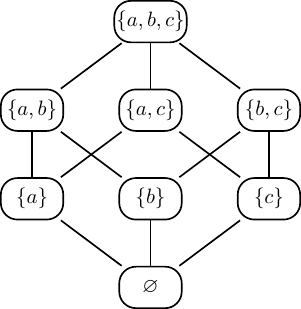}
    \caption{The Hasse diagram for the poset $\ms P = \mc P(\{ a,b,c \})$ ordered by inclusion.}
    \label{fig:powposet}
\end{figure}

\noindent Figure \ref{fig:powposet} illustrates the Hasse diagram for the poset $\ms P = \mc P(\{ a,b,c \})$ ordered by inclusion. Given a poset, we will also reference its \textit{maximal/minimal subsets}.

\begin{defn}[\textit{maximal/minimal subset}]
If $\ms P$ is a poset with partial order $\leq$ and $\ms S \sube \ms P$, then the maximal and minimal subsets of $S$ are defined:
    \[
        \ceo{\ms S} \eq \{ A \in \ms S \; : \; A \not \leq B \s[18] \text{for all} \; B \in \ms S \; (A \neq B)\};
    \]
    \[
        \flo{\ms S} \eq \{ A \in \ms S \; : \; B \not \leq A \s[18] \text{for all} \; B \in \ms S \; (A \neq B)\}.
    \]
\end{defn}

\noindent If we can guarantee that the maximal/minimal subset of $S$ is a unique singleton, such as when $\leq$ is a total order, then we abuse notation and interpret the output as the element of the singleton. Moreover, if no partial order is given, then the poset is assumed to be ordered by inclusion. For example:

\begin{itemize}
    \item $\ceo{V} = d$ in Figure \ref{fig:dag_admg}(\textit{ii}) where $\leq$ is a consistent order;
    \item $\flo{V} = a$ in Figure \ref{fig:dag_admg}(\textit{ii}) where $\leq$ is a consistent order;
    \item $\ceil{\mc P(\{ a, b, c \})} = \{ a, b, c \}$ in Figure \ref{fig:powposet};
    \item $\floor{\mc P(\{ a, b, c \}) \sm \es} = \{ \{ a \} , \{ b \}, \{ c \} \}$ in Figure \ref{fig:powposet}.
\end{itemize}

\noindent For the remainder of this paper, unless otherwise specified, $\ms P$ is the poset $\mc P(V)$ ordered by inclusion.

\subsection{M{\"o}bius Inversion}

Two useful functions for analyzing posets that will be fundamental to deriving the \textit{m}-connecting factorization are the \textit{zeta function} and \textit{M{\"o}bius function}. The zeta function $\zeta_{\ms P} : \ms P \times \ms P \ra \{0, 1\}$ is defined:
\[
    \zeta_{\ms P}(B,A) = 
    \begin{cases}
        \s[18] 1 \s[36] & B \sube A; \\
        \s[18] 0 \s[36] & B \not \sube A,
    \end{cases}
\]
and the M{\"o}bius function $\mu_{\ms P} \s : \ms P \times \ms P \ra \mbb Z$ is defined:
\[
    \mu_{\ms P}(B,A) = 
    \begin{cases}
        \s[6] (-1)^{|A \sm B|} \s[36] & B \sube A; \\
        \s[18] 0 \s[36] & B \not \sube A.
    \end{cases}
\]
The posets we consider are finite and may be thought of as matrices. Abusing notation, we interpret $\zeta_{\ms P}$ and $\mu_{\ms P}$ as matrices where the first and second arguments act as the row and column indices, respectively. Under this interpretation, the M{\"o}bius function is the inverse of the zeta function in the sense that $\mu_{\ms P} \s = \zeta_{\ms P}^{-1}$.

\begin{figure}[H]
    \[
        \begin{blockarray}{ccccccccc}
            \zeta_{\ms P} & \es & \{a\} & \{b\} & \{c\} & \{a,b\} & \{a,c\} & \{b,c\} & \{a,b,c\} \\
            \begin{block}{c[cccccccc]}
                \es & 1 & 1 & 1 & 1 & 1 & 1 & 1 & 1 \\
                \{a\} & 0 & 1 & 0 & 0 & 1 & 1 & 0 & 1 \\
                \{b\} & 0 & 0 & 1 & 0 & 1 & 0 & 1 & 1 \\
                \{c\} & 0 & 0 & 0 & 1 & 0 & 1 & 1 & 1 \\
                \{a,b\} & 0 & 0 & 0 & 0 & 1 & 0 & 0 & 1 \\
                \{a,c\} & 0 & 0 & 0 & 0 & 0 & 1 & 0 & 1 \\
                \{b,c\} & 0 & 0 & 0 & 0 & 0 & 0 & 1 & 1 \\
                \{a,b,c\} & 0 & 0 & 0 & 0 & 0 & 0 & 0 & 1 \\
            \end{block}
        \end{blockarray}
    \]
    \caption{The zeta function of a poset $\ms P = \mc P(\{ a,b,c \})$ ordered by inclusion---the first and second arguments of the zeta function act as row and column indices, respectively.}
    \label{fig:zeta}
\end{figure}

Figure \ref{fig:zeta} depicts the zeta function $\zeta_{\ms P}$ as a matrix for the poset depicted in Figure \ref{fig:powposet}. Notice that the matrix is invertible---it is an upper triangular matrix with non-zero entries on the main diagonal. In general, the rows and columns of the matrix corresponding to the zeta function of a poset can be rearranged in this manner. Accordingly, the matrix corresponding to the zeta function is invertible.

\begin{figure}[H]
    \[
        \begin{blockarray}{ccccccccc}
            \mu_{\ms P} \s & \es & \{a\} & \{b\} & \{c\} & \{a,b\} & \{a,c\} & \{b,c\} & \{a,b,c\} \\
            \begin{block}{c[cccccccc]}
                \es & 1 & \llap{$-$}1 & \llap{$-$}1 & \llap{$-$}1 & 1 & 1 & 1 & \llap{$-$}1 \\
                \{a\} & 0 & 1 & 0 & 0 & \llap{$-$}1 & \llap{$-$}1 & 0 & 1 \\
                \{b\} & 0 & 0 & 1 & 0 & \llap{$-$}1 & 0 & \llap{$-$}1 & 1 \\
                \{c\} & 0 & 0 & 0 & 1 & 0 & \llap{$-$}1 & \llap{$-$}1 & 1 \\
                \{a,b\} & 0 & 0 & 0 & 0 & 1 & 0 & 0 & \llap{$-$}1 \\
                \{a,c\} & 0 & 0 & 0 & 0 & 0 & 1 & 0 & \llap{$-$}1 \\
                \{b,c\} & 0 & 0 & 0 & 0 & 0 & 0 & 1 & \llap{$-$}1 \\
                \{a,b,c\} & 0 & 0 & 0 & 0 & 0 & 0 & 0 & 1 \\
            \end{block}
        \end{blockarray}
    \]  
    \caption{The M{\"o}bius function of a poset $\ms P = \mc P(\{ a,b,c \})$ ordered by inclusion---the first and second arguments of the M{\"o}bius function act as row and column indices, respectively.}
    \label{fig:mobius}
\end{figure}

\noindent Figure \ref{fig:mobius} depicts the M{\"o}bius function $\mu_{\ms P}$ as a matrix for the poset depicted in Figure \ref{fig:powposet}. Again, notice that $\mu_{\ms P}$ is invertible---it is an upper triangular matrix with non-zero entries on the main diagonal. We encourage the reader to check that the matrices depicted in Figures \ref{fig:zeta} and \ref{fig:mobius} are indeed inverses of each other. This relation holds in general and provides a good intuition for the \textit{M{\"o}bius inversion}.

\begin{prop}[\cite{rota1964foundations}]
    % Proposition 2 
    \label{prop:mob}
    Let $\ms P$ be a poset, and $g : \ms P \ra \mbb R$ and $h : \ms P \ra \mbb R$ be real-valued functions. The following expressions imply each other\textup{:}
    \begin{enumerate}
        \item $g(A) = \sum \limits_{B \in \ms P} h(B) \, \mu_{\ms P}(B,A) \s[18] \textup{for all} \; A \in \ms P$\textup{;}
        \item $h(A) = \sum \limits_{B \in \ms P} g(B) \, \zeta_{\ms P}(B,A) \s[18] \textup{for all} \; A \in \ms P$.
    \end{enumerate}
    Alternatively, if we abuse notation and treat $g$ and $h$ as column vectors, then the M{\"o}bius inversion states that $g = \mu_{\ms P}^{\top} h \, \Lra \, h = \zeta_{\ms P}^{\top} g$.
    If $\ms P = \mc P(V)$ is a poset ordered by inclusion, then the M{\"o}bius inversion simplifies to the following equivalent statements:
    \begin{enumerate}
        \item $g(A) = \sum \limits_{B \sube A} (-1)^{|A \sm B|} \, h(B) \s[18] \textup{for all} \; A \sube V$\textup{;}
        \item $h(A) = \sum \limits_{B \sube A} g(B) \s[18] \textup{for all} \; A \sube V$.
    \end{enumerate}
\end{prop}

\begin{cor}[\cite{rota1964foundations}]
    % Corollary 1
    \label{cor:mob_alt}
    Let $\ms P$ be a poset and $g : \ms P \ra \mbb R$ and $h : \ms P \ra \mbb R$ be real-valued functions. The following expressions imply each other\textup{:}
    \begin{enumerate}
        \item $g(A) = \sum \limits_{B \in \ms P} \mu_{\ms P}(A,B) \, h(B) \s[18] \textup{for all} \; A \in \ms P$\textup{;}
        \item $h(A) = \sum \limits_{B \in \ms P} \zeta_{\ms P}(A,B) \, g(B) \s[18] \textup{for all} \; A \in \ms P$.
    \end{enumerate}
    Alternatively, if we abuse notation and view $g$ and $h$ as column vectors, then the corollary states that $g = \mu_{\ms P} \s h \, \Lra \, h = \zeta_{\ms P} g$. If $V$ is a non-empty set of variables and $\ms P = \mc P(V)$ is a poset ordered by inclusion, then the corollary simplifies to the following equivalent statements:
    \begin{enumerate}
        \item $g(A) = \sum \limits_{\substack{B \sube V \\ A \sube B}} (-1)^{|B \sm A|} \, h(B) \s[18] \textup{for all} \; A \sube V$\textup{;}
        \item $h(A) = \sum \limits_{\substack{B \sube V \\A \sube B}} g(B) \s[18] \textup{for all} \; A \sube V$.
    \end{enumerate}
\end{cor}

\begin{figure}[H]
    \[
        \begin{blockarray}{c @{\hspace{10mm}} c @{\hspace{15mm}} c @{\hspace{10mm}} c @{\hspace{10mm}} c}
        S & g(S) & & S & h(S) \\
        \begin{block}{c @{\hspace{10mm}} [c] @{\hspace{15mm}} c @{\hspace{10mm}} c @{\hspace{10mm}} [c]}
            \es & 0 & & \es & 0 \\
            \{a\} & 0 & \smash{\overset{\zeta_{\ms P}^{\top}\;}{\longrightarrow}} & \{a\} & 0 \\
            \{b\} & 0 & & \{b\} & 0 \\
            \{c\} & 1 & & \{c\} & 0 \\
            \{a,b\} & 0 & & \{a,b\} & 1 \\
            \{a,c\} & \llap{$-$}1 & & \{a,c\} & 0 \\
            \{b,c\} & \llap{$-$}1 & \smash{\overset{\;\;\mu_{\ms P}^{\top}}{\longleftarrow}} & \{b,c\} & 0 \\
            \{a,b,c\} & 1 & & \{a,b,c\} & 1 \\
        \end{block}
    \end{blockarray}
    \] 
    \caption{An application of the M{\"o}bius inversion.}
    \label{fig:zmapp}
\end{figure}

Figure \ref{fig:zmapp} depicts an application of the M{\"o}bius inversion. If $g : \ms P \ra \mbb R$ and $h : \ms P \ra \mbb R$ are real-valued functions satisfying Proposition \ref{prop:mob}, then the zeta function depicted in Figure \ref{fig:zeta} applied to $g$ returns $h$ and the M{\"o}bius function depicted in Figure \ref{fig:mobius} applied to $h$ returns $g$. The M{\"o}bius inversion will be used in this manner to transition between imsetal representations of conditional independence. 

\subsection{Independence Models}

Conditional independence is often used synonymously with probabilistic conditional independence---conditional independence statements that are represented in a probability measure. We use this term more generally to include graphical and imsetal conditional independence---conditional independence statements that are represented in an ADMG and an imset, respectively.

\begin{defn}[\textit{disjoint triple}]
    If $A,B,C \sube V$ are disjoint, then $\seq{A,B \mid C}$ is the corresponding \textit{disjoint triple}. The set of all disjoint triples is defined:
    \[
        \mc T(V) \eq \{ \seq{A,B \mid C} \; : \; A,B,C \sube V \; \text{and} \; A \cap B = A \cap C = B \cap C = \es \}.
    \]
\end{defn}

\begin{defn}[\textit{independence statement}]
    If $\mc O$ is a mathematical object over $V$ and $\seq{A,B \mid C} \in \mc T(V)$, then the statement \textit{``$A$ and $B$ are independent given $C$ with respect to $\mc O$''} is an \textit{independence statement}. In this case, we say $\seq{A,B \mid C}$ is represented in $\mc O$ and write \istate{A}{B}{C}{\mc O}.
\end{defn}

\begin{defn}[\textit{independence model}]
    If $\mc O$ is a mathematical object over $V$, then the \textit{independence model} induced by $\mc O$ is defined:
    \[
        \mc I(\mc O) \eq \left\{ \seq{A,B \mid C} \in \mc T(V) \; : \;  \istate{A}{B}{C}{\mc O} \right\}.
    \]
\end{defn}

Several families of independence models have been axiomatically characterized. The independence model $\mc I(\mc O)$ is called a \textit{semi-graphoid} whenever conditions (\textit{i - v}) hold for every collection of disjoint sets $A,B,C,D \sube V$:
\begin{enumerate}
    \item{\makebox[28mm]{triviality\hfill}} \istate{A}{\es}{C}{\mc O};
    \item{\makebox[28mm]{symmetry\hfill}} \istate{A}{B}{C}{\mc O} $\s[18] \Ra \s[18]$ \istate{B}{A}{C}{\mc O};
    \item{\makebox[28mm]{decomposition\hfill}} \istate{A}{BD}{C}{\mc O} $\s[18] \Ra \s[18]$ \istate{A}{D}{C}{\mc O};
    \item{\makebox[28mm]{weak union\hfill}} \istate{A}{BD}{C}{\mc O}$\s[18] \Ra \s[18]$ \istate{A}{B}{CD}{\mc O};
    \item{\makebox[28mm]{contraction\hfill}} \istate{A}{B}{CD}{\mc O} and \istate{A}{D}{C}{\mc O} $\s[18] \Ra \s[18]$ \istate{A}{BD}{C}{\mc O}.
\end{enumerate}
Furthermore, $\mc I(\mc O)$ is called a \textit{graphoid} whenever conditions (\textit{i - vi}) hold for every collection of disjoint sets $A,B,C,D \sube V$:
\begin{enumerate}
    \item[\it vi.]{\makebox[28mm]{intersection\hfill}} \istate{A}{B}{CD}{\mc O} and \istate{A}{D}{BC}{\mc O} $\s[18] \Ra \s[18]$ \istate{A}{BD}{C}{\mc O}.
\end{enumerate}
Moreover, $\mc I(\mc O)$ is called a \textit{compositional graphoid} whenever conditions (\textit{i - vii}) hold for every collection of disjoint sets $A,B,C,D \sube V$:
\begin{enumerate}
    \item[\it vii.]{\makebox[28mm]{composition\hfill}} \istate{A}{B}{C}{\mc O} and \istate{A}{D}{C}{\mc O} $\s[18] \Ra \s[18]$ \istate{A}{BD}{C}{\mc O}.
\end{enumerate}

\subsection{Probabilistic Conditional Independence}

Throughout this paper, we consider a collection of random variables $X$ on a measurable space $(\ms X, \mc X)$ indexed by $V$, and a probability measure $P : \mc X \ra [0,1]$ having density $f : \ms X \ra [0, \infty)$ dominated by a $\sigma$-finite product measure $\mu$. If for $P$-almost all $x_{BC} \in \ms X_{BC}$:
\[
    P(X_A \in \Omega \mid X_{BC} = x_{BC}) = P(X_A \in \Omega \mid X_C = x_C) \s[18] \text{for all} \; \Omega \in \mc X_A
\]
\noindent or, alternatively, for $\mu$-almost all $x \in \ms X$:
\[
    f_{ABC}(x) f_{C}(x) = f_{AC}(x) f_{BC}(x),
\]
then we say \seq{A,B \mid C} is represented in $P$ and write \istate{A}{B}{C}{P}. The independence model induced by $P$ is denoted by $\mc I(P)$. Moreover, the independence model induced by a probability measure is a semi-graphoid. 

\begin{prop}[\cite{studeny2005probabilistic}]
    \label{prop:prob_semi_graphoid}
    If $P$ is a probability measure, then $\mc I(P)$ is a semi-graphoid.
\end{prop}

\noindent In this paper, we are interested in the ADMG models that describe \textit{positive measures}.

\begin{defn}[\textit{positive measure}]
    If probability measure $P$ has density $f : \ms X \ra (0, \infty)$, then $P$ is a positive measure.
\end{defn}

\noindent Notably, any independence model induced by a positive measure is a graphoid. 

\begin{prop}[\cite{pearl1988probabilistic}]
    % Chapter 3 Theorem 1
    \label{prop:pos_graphoid}
    If $P$ is a positive measure, then $\mc I(P)$ is a graphoid.
\end{prop}

\subsection{Graphical Conditional Independence}

Conditional independence in an ADMG is characterized by \textit{m-connecting} paths and the \textit{m-separation} criterion, which naturally extend \textit{d-connecting} paths and the \textit{d-separation} criterion from DAGs to ADMGs \citep{richardson2002ancestral, richardson2003markov}.

\begin{defn}[\textit{m-connecting path}]
If $\g = (V, E)$ is an ADMG and $\seq{a,b \mid C} \in \mc T(V)$ such that a path $\pi$  exists between $a$ and $b$ where:
\begin{enumerate}
    \item every non-collider on $\pi$ is not a member of $C$;
    \item every collider on $\pi$ is an ancestor of some $c \in C$;
\end{enumerate}
then $\pi$ is an \textit{m-connecting path} between $a$ and $b$ relative to $C$.
\end{defn}

\begin{defn}[\textit{m-separation}]
If $\g = (V, E)$ is an ADMG and $\seq{A,B \mid C} \in \mc T(V)$ such that no \textit{m}-connecting path exists between $a$ and $b$ relative to $C$ for all $a \in A$ and $b \in B$, then $A$ and $B$ are \textit{m-separated} by $C$.
\end{defn}

Let $\g = (V,E)$ be an ADMG and $\seq{A,B \mid C} \in \mc T(V)$. If $A$ and $B$ are \textit{m}-separated by $C$, then we say \seq{A,B \mid C} is represented in $\g$ and write \istate{A}{B}{C}{\g}. The independence model induced by $\g$ is denoted $\mc I(\g)$.  Moreover, the independence model induced by an ADMG is a compositional graphoid.

\begin{prop}[\cite{sadeghi2014markov}]
    % Theorem 1 
    \label{prop:admg_graphoid}
    If \g[] is an ADMG, then $\mc I(\g)$ is a compositional graphoid.
\end{prop}

\noindent Notably, multiple ADMGs can induce the same independence model. This phenomenon is called \textit{Markov equivalence}, and a complete set of Markov equivalent ADMGs is called a \textit{Markov equivalence class} (MEC).

\subsubsection{Graphical Markov Properties}

ADMGs and probability measures are linked by conditional independence using the \textit{global Markov property}.

\begin{defn}[\textit{global Markov property}]
    If $\g = (V, E)$ is an ADMG and $P$ is a probability measure where:
    \[
        \istate{A}{B}{C}{\g} \s[18] \Ra \s[18] \istate{A}{B}{C}{P} \s[18] \text{for all} \s[18] \seq{A,B \mid C} \in \mc T(V),
    \]
    or $\mc I(\g) \sube \mc I(P)$, then $P$ satisfies the \textit{global Markov property} for $\g$.
\end{defn}

\noindent However, many of the conditional independence statements characterized by the global Markov property are redundant---implied by the semi-graphoid axioms. Accordingly, the global Markov property can be simplified to a more succinct property using a consistent order and a few additional concepts. Let $\g = (V,E)$ be an ADMG. If $b \in V$, then:
\[
    \mb{\g}{b} \eq \co{\g}{b} \sm b \s[72] \cl{\g}{b} \eq \co{\g}{b}
\]
are the \textit{Markov boundary} and \textit{closure} of $b$, respectively. Together, they have the following property.

\begin{prop}
    \label{prop:ordered_markov_helper}
    Let $\g = (V, E)$ be an ADMG. If $b \in V$, then\textup{:}
    \[
        \istate{b}{V \sm \cl{\g}{b}}{\mb{\g}{b}}{\g}.
    \]
\end{prop}

\noindent In other words, $\mb{\g}{b}$ renders $b$ independent of the vertices not in $\cl{\g}{b}$. Moreover, these sets are the smallest sets with this property and lead directly to the \textit{ordered local Markov property}.

\begin{defn}[\textit{ordered local Markov property}]
    Let $\g = (V, E)$ be an ADMG with consistent order $\leq$. If $P$ is a probability measure where:
    \[
        \istate{b}{A \sm \cl{\g[A]}{b}}{\mb{\g[A]}{b}}{P} \s[18] \text{for all} \s[18] A \in \mc A(\g) \; \text{and} \; b = \ceo{A},
    \]
    then $P$ satisfies the \textit{ordered local Markov property} for $\g$ and $\leq$.
\end{defn}

\noindent Auspiciously, the global and ordered local Markov properties are equivalent.

\begin{thm}[\cite{richardson2003markov}]
    % Theorem 2 
    \label{thm:mprops}
    If $\g = (V, E)$ is an ADMG with consistent order $\leq$ and $P$ is a probability measure, then the following are equivalent\textup{:}
    \begin{enumerate}
        \item $P$ satisfies the global Markov property for $\g$\textup{;}
        \item $P$ satisfies the ordered local Markov property for $\g$ and $\leq$.
    \end{enumerate}
\end{thm}

\subsubsection{Factorization}

The global Markov property for DAGs can be characterized by a well-known \textit{recursive factorization} that encodes the independence statements represented in a DAG algebraically. If $\g = (V, E)$ is a DAG and $P$ is a probability measure, then the following are equivalent:
\begin{enumerate}
    \item $f(x) = \prod_{v \in V} f_{v \mid \pa{\g}{v}}(x) \s[18]$ for $\mu$-almost all $x \in \mc X$\textup{;}
    \item $\istate{A}{B}{C}{\g} \s[18] \Ra \s[18] \istate{A}{B}{C}{P} \s[18]$ for all $\seq{A,B \mid C} \in \mc T(V)$.
\end{enumerate}

\noindent \cite{richardson2009factorization} developed a similar factorization for ADMGs, however, in order to state his factorization, we first must review a few additional concepts. At a high level, Richardson's factorization utilizes a function to partition the variables into sets called \textit{heads}. The heads are then used to construct one or more equations. In these equations marginal densities over ancestral sets are set equal to the products of conditional densities of heads conditioned on accompanying sets called \textit{tails}.

% (not to be confused with \textit{recursive heads})

\begin{defn}[\textit{head}]
    Let $\g = (V, E)$ be an ADMG and $H \sube V$ $(H \neq \es)$. If $H = \ba{\g}{C}$ for some $C \in \mc C(\g)$, then $H$ is a \textit{head}. In other words, $H$ is the barren subset of a (non-empty) collider-connecting set. The set of all heads is defined:
    \[
        \mc H(\g) \eq \{ \ba{\g}{C} \; : \; C \in \mc C(\g) \} \sm \es
    \]
\end{defn}

\begin{defn}[\textit{tail}]
    Let $\g = (V, E)$ be an ADMG and $H \in \mc H(\g)$. If $A = \an{\g}{H}$, then the \textit{tail} of $H$ is defined:
    \[
        \ta{\g}{H} \eq \co{\g[A]}{H} \sm H. 
    \]
    In other words, the tail of $H$ contains the nontrivial collider-connecting vertices of $H$ connected by ancestors of $H$. 
\end{defn}

In order to define the function that will partition the variables into heads, we need to define a new partial order. For the remainder of this section, let $\g =(V, E)$ be an ADMG and $\leq$ be the partial order:
\[
    A \leq B \s[18] \Lra \s[18] A \sube \an{\g}{B} \s[18] \text{for all} \s[18] A,B \in \mc H(\g).
\]
The heads $\mc H(V)$ partition $V$ with the help of two functions: $\Phi_{\g} : \mc P(V) \ra \mc P(\mc H(\g))$ and $\Psi_{\g} : \mc P(V) \ra \mc P(V)$. For $S \sube V$:
\begin{align*}
    \Phi_{\g}(S) &\eq \ceo{\{ H \in \mc H(\g) \; : \; H \sube S \}}; \\
    \Psi_{\g}(S) &\eq S \sm \bigcup_{T \in \Phi_{\g}(S)} T.
\end{align*}
In other words, $\Phi_{\g}(S)$ returns the maximal heads contained in $S$ and $\Psi_{\g}(S)$ returns the variables in $S$ not contained within a set in $\Phi_{\g}(S)$. The partition function $[ \, \cdot \, ]_{\g} : \mc P(V) \ra \mc P(\mc H(\g))$ is defined recursively:
\[
    [\, S \,]_{\g} \eq 
    \begin{cases}
        \s[18] \es \s[36] & S = \es; \\
        \s[18] \Phi_{\g}(S) \cup [\Psi(S)]_{\g} \s[36] & S \neq \es.
    \end{cases}
\]
In other words, the partition function returns the maximal head sets contained in $S$ and is then reapplied to the remaining variables.

\begin{thm}[\cite{richardson2009factorization}]
% Theorem 4 
    Let $\g = (V, E)$ be an ADMG and $P$ be a probability measure. The following are equivalent\textup{:}
    \begin{enumerate}
        \item $f_A(x) = \prod_{H \in [A]_{\g}} f_{H \mid T}(x) \s[18]$ for all $A \in \mc A(\g)$ and $\mu$-almost all $x \in \mc X$\textup{;}
        \item $\istate{A}{B}{C}{\g} \s[18] \Ra \s[18] \istate{A}{B}{C}{P} \s[18]$ for all $\seq{A,B \mid C} \in \mc T(V)$.
    \end{enumerate}
\end{thm}

\begin{figure}[H]
    \centering
    \begin{minipage}{.6\textwidth}
        \centering
        \begin{tabular}{l|cccc}
            % \hline
            H & $\{a\}$ & $\{b\}$ & $\{c\}$ & $\{d\}$ \\
            \hline
            T & $\es$ & $\{a\}$ & $\{b\}$ & $\{a,b,c\}$ \\
            % \hline
        \end{tabular}
    \end{minipage}%
    \begin{minipage}{.35\textwidth}
        \centering
        \includegraphics[page=4]{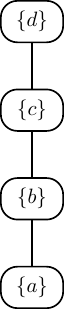}
    \end{minipage}
    \caption{The heads and tails of the ADMG in Figure \ref{fig:dag_admg} (\textit{ii}) and the Hasse diagram for the corresponding poset over the heads.}
    \label{fig:tr_faci}
\end{figure}

Figure \ref{fig:tr_faci} illustrates the heads and tails of the ADMG in Figure \ref{fig:dag_admg} (\textit{ii}) and the corresponding poset over the heads. The only ancestral set that admits a non-redundant factorization is, $\{a,b,c,d\}$ which partitions the variables into the following heads:
\begin{align*}
    [\, \{a,b,c,d\} \,]_{\g} &= \Phi_{\g}(\{a,b,c,d\}) \cup [\, \Psi(\{a,b,c,d\}) \,]_{\g} \\
    &= \{\{d\}\} \cup [\, \{a,b,c\} \,]_{\g} \\
    &= \{\{d\}\} \cup \Phi_{\g}(\{a,b,c\}) \cup [\, \Psi(\{a,b,c\}) \,]_{\g} \\
    &= \{\{d\}\} \cup \{\{c\}\} \cup [\, \{a,b\} \,]_{\g} \\
    &= \{\{d\}, \{c\}\} \cup \Phi_{\g}(\{a,b\}) \cup [\, \Psi(\{a,b\}) \,]_{\g} \\
    &= \{\{d\}, \{c\}\} \cup \{\{b\}\} \cup [\, \{a\} \,]_{\g} \\
    &= \{\{d\}, \{c\}, \{b\}\} \cup \Phi_{\g}(\{a\}) \cup [\, \Psi(\{a\}) \,]_{\g} \\
    &= \{\{d\}, \{c\}, \{b\}\} \cup \{\{a\}\} \cup [\, \es \,]_{\g} \\
    &= \{\{d\}, \{c\}, \{b\}, \{a\}\}.
\end{align*}

\noindent Accordingly, a probability measure satisfies the global Markov property with respect to $\g$ if and only if the following factorization holds for $\mu$-almost all $x \in \mc X$:
\[
    f_{abcd}(x) = f_{d \mid abc}(x) \, f_{c \mid b}(x) \, f_{b \mid a}(x) \, f_a(x).
\]

\begin{figure}[H]
    \centering
    \begin{minipage}{.6\textwidth}
        \centering
        \begin{tabular}{l|cccccc}
            H & $\{a\}$ & $\{b\}$ & $\{c\}$ & $\{d\}$ & $\{a,d\}$ & $\{b,c\}$ \\
            \hline
            T & $\es$ & $\es$ & $\{a\}$ & $\{b\}$  & $\{b\}$ & $\{a\}$ \\
        \end{tabular}
    \end{minipage}%
    \begin{minipage}{.35\textwidth}
        \centering
        \includegraphics[page=5]{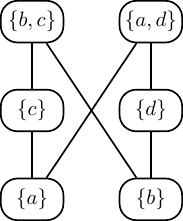}
    \end{minipage}
    \caption{The heads and tails of the ADMG in Figure \ref{fig:set_ex} and the Hasse diagram for the corresponding poset over the heads.}
    \label{fig:tr_facii}
\end{figure}

Figure \ref{fig:tr_facii} illustrates the heads and tails of the ADMG in Figure \ref{fig:set_ex} and the corresponding poset over the heads. The only ancestral sets that admit non-redundant factorizations are $\{a,b,c,d\}$ and $\{a,b\}$ which partition the variables into the following heads:
\begin{align*}
    [\, \{a,b,c,d\} \,]_{\g} &= \Phi_{\g}(\{a,b,c,d\}) \cup [\, \Psi(\{a,b,c,d\}) \,]_{\g} \\
    &= \{\{a,d\}, \{b,c\}\} \cup [\, \es \,]_{\g} \\
    &= \{\{a,d\}, \{b,c\}\} \textup{;}
\end{align*}
\begin{align*}
    [\, \{a,b\} \,]_{\g} &= \Phi_{\g}(\{a,b\}) \cup [\, \Psi(\{a,b\}) \,]_{\g} \\
    &= \{\{a\}, \{b\}\} \cup [\, \es \,]_{\g} \\
    &= \{\{a\}, \{b\}\}.
\end{align*}

\noindent Accordingly, a probability measure satisfies the global Markov property with respect to $\g$ if and only if the following factorizations hold for $\mu$-almost all $x \in \mc X$:
\begin{align*}
    f_{abcd}(x) &= f_{ad \mid b}(x) \, f_{bc \mid a}(x) \textup{;} \\
    f_{ab}(x) &= f_a(x) \, f_b(x).
\end{align*}

When the probability measure is positive, both Richardson's factorization and the \textit{m}-connecting factorization proposed in this paper are equivalent to the global Markov property and therefore equivalent to each other. The key difference is that the factorization characterized by Richardson may require multiple equations, while the \textit{m}-connecting factorization only requires a single equation, which makes it ideal for formulating a scoring criterion; see Section \ref{sec:consist_score}.

\subsection{Integer-valued Multisets}

\begin{defn}[\textit{integer-valued multiset}]
    An \textit{integer-valued multiset}, \textit{imset} for short, is an integer-valued function $u$: $\mc P(V) \ra \mbb Z$. Alternatively, $u$ is an element of $\mbb Z^{\mc P(V)}$.
\end{defn}

\noindent A simple example of an imset is the \textit{identifier}. The identifier for a set $A \sube V$ is defined:
\[
    \delta_A(S) \eq
    \begin{cases}
    \s[18] 1 & \s[36] S = A; \\ 
    \s[18] 0 & \s[36] S \neq A.
    \end{cases}
\]
More generally, the identifier for a set of sets $\ms A \sube \mc P(V)$ is defined: 
\[
    \delta_{\ms A}(S) \eq
    \begin{cases}
    \s[18] 1 & \s[36] S \in \ms A; \\ 
    \s[18] 0 & \s[36] S \not \in \ms A.
    \end{cases}
\]

\noindent The identifier imset may be used to construct \textit{elementary imsets} which are the fundamental building block of the imsetal representation of conditional independence.

\begin{defn}[\textit{elementary imset}]
If $\seq{a,b \mid C} \in \mc T(V)$, then the corresponding \textit{elementary imset} $u_{\seq{a,b \mid C}} : \mc P(V) \ra \mbb Z$ is an imset defined:
\[
    u_{\seq{a,b \mid C}} \eq \delta_{ab \s C} + \delta_{C} - \delta_{a \s C} - \delta_{b \s C}.
\]
\end{defn}

\begin{figure}[H]
    \[
        \begin{blockarray}{c @{\hspace{10mm}} rccccccccl}
            \begin{block}{c @{\hspace{10mm}} rccccccccl}
                S & & \es & \{a\} & \{b\} & \{c\} & \{a,b\} & \{a,c\} & \{b,c\} & \{a,b,c\} & \\
                \vspace{-5mm} \\
                u_{\seq{a,b \mid c}}(S) & \Big[ & 0 & 0 & 0 & 1 & 0 & \llap{$-$}1 & \llap{$-$}1 & 1 & \!\!\!\!\!\!\! \Big]^\top \\
            \end{block}
        \end{blockarray}
    \]
    \caption{An elementary imset: $u_{\seq{a,b \mid c}}$.}
    \label{fig:imset}
\end{figure}

\noindent Figure \ref{fig:imset} illustrates the elementary imset $u_{\seq{a,b \mid c}}$. The structure of an elementary imset imitates the definition of probabilistic conditional independence when the independent sets are singletons. This is not a coincidence, in fact, elementary imsets were designed to represent statements of conditional independence in this manner. Elementary imsets may be used to construct other types of imsets, such as \textit{semi-elementary imsets} and \textit{structural imsets}.

\begin{defn}[\textit{semi-elementary imset}]
    If $\seq{A,B \mid C} \in \mc T(V)$, then the corresponding \textit{semi-elementary imset} $u_{\seq{A,B \mid C}} : \mc P(V) \ra \mbb Z$ is defined:
    \[
        u_{\seq{A,B \mid C}} \eq \delta_{ABC} + \delta_C - \delta_{AC} - \delta_{BC}.
    \]
\end{defn}

\noindent Semi-elementary imsets may be constructed as linear combinations of elementary imsets with non-negative integer coefficients.

\begin{prop}[\cite{studeny2005probabilistic}]
    % Proposition 4.2 
    \label{prop:semi_elem}
    Every semi-elementary imset is a linear combination of elementary imsets with non-negative integer coefficients.
\end{prop}

\noindent Similarly, structural imsets may be constructed as linear combinations of elementary imsets with non-negative rational coefficients.

\begin{defn}[\textit{structural imset}]
    If $u : \mc P(V) \ra \mbb Z$ is an imset where:
    \[
        u = \sum_{\seq{a,b \mid C} \in \mc T(V)} k_{a,b \mid C} \, u_{\seq{a,b \mid C}} \s[18] \text{for some} \; k_{a,b \mid C} \in \mbb Q^+,
    \]
    then $u$ is a \textit{structural imset} over $V$. Denote the set of all structural imsets over $V$ by $\mc S(V)$.
\end{defn}

Conditional independence in structural imsets may be characterized by subtraction of semi-elementary imsets. If there exists $k \in \mbb Q^+$ such that $u - k \, u_{\seq{A,B \mid C}}$ is a structural imset, then we say \seq{A,B \mid C} is represented in $u$ and write \istate{A}{B}{C}{u}. The independence model induced by $u$ is denoted $\mc I(u)$. Moreover, the independence models induced by structural imsets are semi-graphoids.

\begin{prop}[\cite{studeny2005probabilistic}]
    % Lemma 4.6 
    \label{prop:struct_semigraphiod}
    Let $u$ be an imset. If $u \in \mc S(V)$, then $\mc I(u)$ is a semi-graphoid.
\end{prop}

\cite{studeny2005probabilistic} linked structural imsets and probability measures by defining a factorization that uses the density of a probability measure to algebraically imply the conditional independence statements represented in a structural imset. Studen{\'y} defined his factorization for probability measures with finite multiinformation; see \citep{studeny2005probabilistic}. Below, we restate Studen{\'y}'s factorization for positive measures.

\begin{thm}
    % Theorem 4.1 
    \label{thm:imfact}
    
    Let $u$ be an imset and $P$ be a positive measure. If $u \in \mc S(V)$, then the following are equivalent\textup{:}
    \begin{enumerate}
        \item $\prod \limits_{S \in \mc P(V)} f_{S}(x)^{u(S)} = 1 \s[18]$ for $\mu$-almost all $x \in \mc X$\textup{;}
        \item $\istate{A}{B}{C}{u} \s[18] \Ra \s[18] \istate{A}{B}{C}{P} \s[18]$ for all $\seq{A,B \mid C} \in \mc T(V)$.
    \end{enumerate}
\end{thm}

Developments to the imsetal representation of conditional independence have been largely theoretical, except for in restricted scenarios, such as the restriction to DAG independence models---indeed, several methods use imsets to learn DAGs \citep{studeny2008mathematical, studeny2017towards, studeny2014learning}. Characteristic imsets are a recent innovation in the space of imsetal independence models introduced by \citep{hemmecke2012characteristic} for representing DAG independence models. In particular, for a DAG $\g = (V,E)$, the M{\"o}bius inversion of a characteristic imset $c_{\g} : \mc P(V) \ra \{ 0, 1 \}$ yields a structural imset that represents the same conditional independence statements as $\g$:
\[
u = \mu_{\ms P} \s (1 - c_{\g}) \s[18] \Ra \s[18] u \in \mc S(V) \; \text{and} \; \mc I(u) = \mc I(\g).
\]
\noindent Furthermore, the values of a characteristic imset may be read directly off of a DAG.
\begin{prop}
    \label{prop:char_imset}
    Let $\g[] = (V,E)$ be a DAG and $S \sube V$ $(|S| \geq 2)$. If $A = \an{\g}{S}$ and $B = \ba{\g}{S}$, then\textup{:}
    \begin{enumerate}
        \item $c_{\g}(S) \in \{ 0,1 \}$\textup{;}
        \item $c_{\g}(S) = 1 \s[18] \Lra \s[18] \text{there exists} \; a \in S \; \text{such that} \; S \sm a \sube \pa{\g}{a}$\textup{;}
        \item $c_{\g}(S) = 1 \s[18] \Lra \s[18] S \sube \co{\g[A]}{B}$. 
    \end{enumerate}
\end{prop}

\noindent Properties (i) and (ii) were shown by \cite{hemmecke2012characteristic}. We add property (iii) which is equivalent to property (ii) in order to compare the characteristic imset to the \textit{m}-connecting imset; see Section \ref{sec:m-imset}. 

\section{The \textit{m}-connecting Imset}
\label{sec:m-imset}

In this section, we introduce the \textit{m-connecting imset} in order to construct a factorization for ADMGs. 

\begin{defn}[\textit{m-connecting imset}]
    If $\g = (V, E)$ is an ADMG, then the corresponding \textit{m}-connecting imset $m_{\g} : \mc P(V) \ra \{ 0,1 \}$ is defined:
    \[
        m_{\g}(S) \eq 
        \begin{cases}
            \s[18] 1 & \s[36] \{ \seq{a,b \mid C} \in \mc I(\g) \; : \; S \sm C = ab \} = \es; \\
            \s[18] 0 & \s[36] \{ \seq{a,b \mid C} \in \mc I(\g) \; : \; S \sm C = ab \} \neq \es.
        \end{cases}
    \]
    In other words, $m_{\g}(S) = 1$ if and only if there exists an \textit{m}-connecting path between $a$ and $b$ relative to $C$ for all $\seq{a,b \mid C} \in \mc T(V)$ where $S \sm C = ab$.
\end{defn}

\noindent The m-connecting imset naturally generalizes the characteristic imset in the following sense:

\begin{prop}
    \label{prop:mcimset_defn}
    Let $\g[] = (V,E)$ be an ADMG and $S \sube V$. If $A = \an{\g}{S}$ and $B = \ba{\g}{S}$, then\textup{:}
    \begin{enumerate}
        \item $m_{\g}(S) \in \{ 0,1 \}$\textup{;}
        \item $m_{\g}(S) = 1 \s[18] \Lra \s[18] S \sube \co{\g[A]}{B}$. 
    \end{enumerate}
\end{prop}

\noindent The (non-empty) sets for which the m-connecting imset is 1 are called \textit{parameterizing sets} and have been used to parameterize ADMG models \citep{evans2014markovian}, characterize margins of discrete DAG models \citep{evans2018margins}, and to test the equivalence of ADMG independence models \citep{hu2020faster}.

\begin{defn}[\textit{parameterizing sets}]
If $\g = (V,E)$ is an ADMG, then the parametrizing sets are defined:
    \[
        \mc M(\g) \eq \{ HT \in \mc P(V) \; : \; H \in \mc H(\g) \; \text{and} \; T \sube \ta{\g}{H} \}
    \]
\end{defn}

\noindent Our choice of notation stresses that the members of these sets are connected by \textit{m}-connecting paths---this notion is made rigorous by \cite{hu2020faster}.

\begin{prop}[\cite{hu2020faster}]
    % Proposition 3.3 
    \label{prop:rasm}
    If $\g = (V, E)$ is an ADMG, then $M \in \mc M(\g)$ if and only if there exists an \textit{m}-connecting path between $a$ and $b$ relative to $C$ for all $\seq{a,b \mid C} \in \mc T(V)$ where $M \sm C = ab$.
\end{prop}

\noindent Similarly, we define the \textit{non-m-connecting imset}.

\begin{defn}[\textit{non-m-connecting imset}]
If $\g = (V, E)$ is an ADMG, then the corresponding non-\textit{m}-connecting imset $n_{\g} : \mc P(V) \ra \{ 0,1 \}$ is defined:
\[
    n_{\g}(S) \eq 
    \begin{cases}
        \s[18] 1 & \s[36] \{ \seq{a,b \mid C} \in \mc I(\g) \; : \; S \sm C = ab  \} \neq \es; \\
        \s[18] 0 & \s[36] \{ \seq{a,b \mid C} \in \mc I(\g) \; : \; S \sm C = ab  \} = \es.
    \end{cases}
\]
\end{defn}

\noindent The sets for which the non-\textit{m}-connecting imset is 1 are called \textit{constrained sets} \citep{evans2018margins}. Notably, some members of these sets are \textit{not} connected by any \textit{m}-connecting paths. Accordingly, we denote the set of all constrained sets:
\[
    \mc N(\g) \eq \mc P(V) \sm [ \mc M(\g) \cup \{\es\} ].
\]

\begin{figure}[H]
\begin{center}
\includegraphics[page=7]{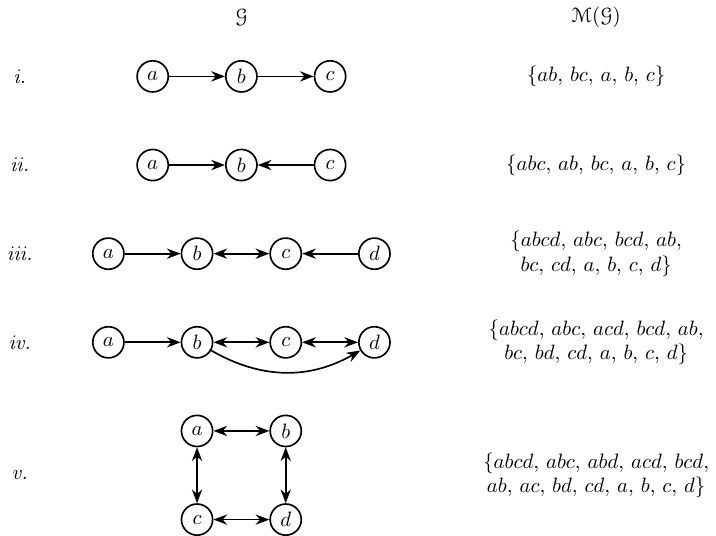}
\end{center}
\caption{Several ADMGs and their parameterizing sets.}
\label{fig:mcs}
\end{figure}

\noindent Figure \ref{fig:mcs} illustrates several ADMGs and their parameterizing sets. 

% We make several observations regarding parameterizing sets and collider-connecting sets. Lemma \ref{lem:cm_equiv} shows the equivalence of maximal parameterizing sets and maximal collider-connecting sets, and Lemma \ref{lem:unique_barren} shows that barren vertices have unique maximal parameterizing and collider-connecting sets.

% \begin{lem}
% \label{lem:cm_equiv}
% If $\g$ is an ADMG, then:
% \begin{enumerate}
%     \item $\mc C(\g) \sube \mc M(\g)$\textup{;}
%     \item $\ceil{\mc C(\g)} = \ceil{\mc M(\g)}$.
% \end{enumerate}
% \end{lem}

% \begin{lem}
% \label{lem:unique_barren}
% If $\g = (V, E)$ is an ADMG and $b \in \ba{\g}{V}$, then:
% \begin{enumerate}
%     \item $|\ceil{\{ C \in \mc C(\g) \; : \; b \in C \}}| = 1$\textup{;}
%     \item $|\ceil{\{ M \in \mc M(\g) \; : \; b \in M \}}| = 1$.
% \end{enumerate}
% \end{lem}

% \noindent Lemma \ref{lem:induced_admg} shows that the induced subgraph of an ADMG over an ancestral set induces an independence subset over the shared variables. \textbf{Furthermore, the induced subgraph has the same \textit{m}-connecting subsets of $A$.}

% \begin{lem}
% \label{lem:induced_admg}
% If $\g$ is an ADMG and $A \in \mc A(\g)$, then:
% \begin{enumerate}
%     \item $\mc I(\g[A]) = \{ \seq{A, B \mid C} \in \mc T(A) \; : \; \seq{A, B \mid C} \in \mc I(\g) \}$\textup{;}
%     \item $\mc M(\g[A]) = \{ M \in \mc P(A) \; : \; M \in \mc M(\g) \}$.
% \end{enumerate}
% \end{lem}

\subsection{Probabilistic Terms}

To facilitate the forthcoming discussion, we define several new terms. We define a function $\phi_A: \mc X_A \ra \mbb R$ as a linear combination of log density terms motivated by the M{\"o}bius inversion:
\[
    \log f_A(x) = \sum_{B \sube A} \phi_B(x) \s[72] \phi_A(x) = \sum_{B \sube A} (-1)^{|A \sm B|} \log f_B(x) \\.
\]
\noindent Defining functional vectors $h_{\ms P}(x) : \mc P(V) \ra \mbb R^{|\mc P(V)|}$ and $\phi_{\ms P}(x) : \mc P(V) \ra \mbb R^{|\mc P(V)|}$:
\[
    \underset{S \in \mc P(V)}{[h_{\ms P}(x)]_{S}} \eq \log f_{S}(x) \s[72] \underset{S \in \mc P(V)}{[\phi_{\ms P}(x)]_{S}} \eq \phi_{S}(x),
\]
\noindent we equivalently have:
\[
    h_{\ms P} = \zeta_{\ms P}^{\top} \phi_{\ms P}(x) \s[72] \phi_{\ms P}(x) = \mu_{\ms P}^{\top} h_{\ms P}(x).
\]
We accept the convention $\phi_{\es}(x) = 0$ for all $x \in \mc X$. The expectation of $\phi_A(x)$ with respect to $P$ has been previously studied in the field of information theory by several researchers, including \cite{mcgill1954multivariate}, who coined the term \textit{interaction information}. Accordingly, we call $\phi_A(x)$ the \textit{interaction information rate}.

We define the following shorthand for indicators of sets of sets and their corresponding $\phi$ terms. Let $A,B \sube V$ $(A \neq \es)$ be disjoint sets:
\[
    \delta_{A \mid B} \eq \sum_{\substack{S \sube AB \\ A \sube S}} \delta_{S} \s[72] \phi_{A \mid B}(x) \eq \sum_{\substack{S \sube AB \\ A \sube S}} \phi_{S}(x).
\]

\noindent Similar to above, we call $\phi_{A \mid B}(x)$ the \textit{conditional interaction information rate}. Another case is when the set of sets corresponds to a semi-elementary imset transformed by the M{\"o}bius inversion. Let $A,B,C \sube V$ $(AB \neq \es)$ be disjoint sets:
\[
    \delta_{A,B \mid C} \eq \sum_{\substack{S \sube ABC \\ S \not \sube AC \\ S \not \sube BC}} \delta_{S} \s[72] \phi_{A,B \mid C}(x) \eq \sum_{\substack{S \sube ABC \\ S \not \sube AC \\ S \not \sube BC}} \phi_{S}(x).
\]
The expectation of $\phi_{A,B \mid C}(x)$ with respect to $P$ is the well-known information theoretic concept of \textit{conditional mutual information}. Accordingly, we call $\phi_{A,B \mid C}(x)$ the \textit{mutual information rate}. The conditional mutual information rate and its corresponding imset have the following properties.

\vskip 5mm

\begin{prop}
\label{prop:iir}
Let $P$ be a positive measure\textup{:}
\begin{enumerate}
    \item $\displaystyle \phi_{A \mid B}(x) = \sum_{\substack{T \sube AB \\ B \sube T}} (-1)^{|AB \sm T|} \log f_{T}(x)$;
    \item $\displaystyle \phi_{A, B \mid C}(x) = \log \frac{ f_{A B C}(x) f_{C}(x) }{ f_{A C}(x) f_{B C}(x) }$;
    \item $u_{\seq{A,B \mid C}} = \mu_{\ms P} \s \delta_{A,B \mid C}$;
    \item $\istate{A}{B}{C}{P} \s[18] \Lra \s[18] \phi_{A, B \mid C}(x) = 0 \s[18] \text{for $\mu$-almost all} \; x \in \mc X$.
\end{enumerate}
\end{prop}

\subsection{The \textit{m}-connecting Factorization}

In this section, we provide two versions of the \textit{m}-connecting factorization. First we give a version of the \textit{m}-connecting factorization that holds for ADMGs with bounded parameterizing set cardinality.

\begin{thm}
\label{thm:mconn_fact_simp}
Let $\g = (V, E)$ be an ADMG. If $P$ is a positive measure and $|M| \leq 5$ for all $M \in \mc M(\g)$, then the following are equivalent:
\begin{enumerate}
    \item $\log f(x) = \sum \limits_{M \in \mc M(\g)} \phi_M(x) \s[18]$ for $\mu$-almost all $x \in \mc X$\textup{;}
    \item \istate{A}{B}{C}{\g} $\s[18] \Ra \s[18]$ \istate{A}{B}{C}{P} $\s[18]$ for every $\seq{A,B \mid C} \in \mc T(V)$.
\end{enumerate}
\end{thm}

\noindent Alternatively, the summation over the parameterizing sets may be computed with heads and tails. Let $\g = (V, E)$ be an ADMG with consistent order $\leq$. If $R = \pre{\g}{v}$ and $T = \ta{\g}{H}$ where context clarifies $v \in V$ and $H \in \mc H(\g)$, then:
\begin{align*}
    \sum_{M \in \mc M(\g)} \phi_{M}(x) &= \sum_{H \in \mc H(\g)} \phi_{H \mid T}(x) \\
    &= \sum_{H \in \mc H(\g)} \sum_{S \sube H} (-1)^{|H \sm S|} \log f_{ST}(x) \\
    &= \sum_{v \in V} \sum_{\substack{H \in \mc H(\g[R]) \\ v \in H}} \sum_{\substack{S \sube H \\ v \not \in S}} (-1)^{|H \sm S| - 1} \log f_{v \mid ST}(x).
\end{align*}

Using the alternative stated above, we derive the \textit{m}-connecting factorizations for several ADMGs with respect to total orders given by the order of the density terms. A positive measure satisfies the global Markov property with respect to an ADMG if and only if the corresponding factorization holds for $\mu$-almost all $x \in \mc X$:
\begin{align*}
    \text{Fig. \ref{fig:dag_admg} (\textit{ii}):} \s[18] \log f(x) &= \log f_{d \mid abc}(x) + \log f_{c \mid b}(x) + \log f_{b \mid a}(x) + \log f_a(x); \\
    \text{Fig. \ref{fig:set_ex}:} \s[18] \log f(x) &= \log f_{d \mid ab}(x) + \log f_{c \mid ab}(x) + \log f_{b}(x) + \log f_{a}(x); \\
    \text{Fig. \ref{fig:mcs} (\textit{i}):} \s[18] \log f(x) &= \log f_{c \mid b}(x) + \log f_{b \mid a}(x) + \log f_{a}(x); \\
    \text{Fig. \ref{fig:mcs} (\textit{ii}):} \s[18] \log f(x) &= \log f_{b \mid ac}(x) + \log f_{c}(x) + \log f_{a}(x); \\
    \text{Fig. \ref{fig:mcs} (\textit{iii}):} \s[18] \log f(x) &= \log \left[ f_{c \mid abd}(x) \frac{f_{c \mid d}(x)}{f_{c \mid a d}(x)} \right] + \log f_{b \mid a}(x) + \log f_{d}(x) + \log f_{a}(x); \\
    \text{Fig. \ref{fig:mcs} (\textit{iv}):} \s[18] \log f(x) &= \log \left[ f_{d \mid abc}(x) \frac{f_{d \mid b}(x)}{f_{d \mid ab}(x)} \right] + \log f_{b \mid ac}(x) + \log f_{c}(x) + \log f_{a}(x); \\
    \text{Fig. \ref{fig:mcs} (\textit{v}):} \s[18] \log f(x) &= \log \left[ f_{d \mid abc}(x) \frac{f_{d}(x)}{f_{d \mid a}(x)} \right] + \log \left[ f_{c \mid ab}(x) \frac{f_{c}(x)}{f_{c \mid b}(x)} \right] + \log f_{b \mid a}(x) + \log f_{a}(x).
\end{align*}
The \textit{m}-connecting factorization resembles the well-known recursive factorization for DAGs. Vertices in the density terms are conditioned on their Markov boundaries in the subgraph induced by their preceding vertices. Additionally, if a density term implies extraneous dependencies, then the term is reweighted by a ratio equaling 1 if and only if the extra dependencies are false in the probability measure.

The proof strategy for Theorem \ref{thm:mconn_fact_simp} splits the log density of a positive measure into \textit{m}-connecting and non-\textit{m}-connecting pieces using the M{\"o}bius inversion (Proposition \ref{prop:mob}):
\begin{align*}
    \log f(x) &= \phi_{\ms P}(x)^{\top} \delta_{\mc P(V)} \\
    &= \phi_{\ms P}(x)^{\top} (m_{\g} + n_{\g}) \\
    &= \sum \limits_{M \in \mc M(\g)} \phi_M(x) + \phi_{\ms P}(x)^{\top} n_{\g}.
\end{align*}

\noindent Notably, $\mu_{\ms P} \s n_{\g}$ is a structural imset if $|M| \leq 5$ for all $M \in \mc M(\g)$.

\begin{prop}
    \label{prop:struct_conds}
    Let $\g = (V, E)$ be an ADMG. If $|M| \leq 5$ for all $M \in \mc M(\g)$, then $\mu_{\ms P} \s n_{\g} \in \mc S(V)$.
\end{prop}

\noindent It follows from Proposition \ref{prop:struct_conds} and Theorem \ref{thm:imfact} that:
\[
    \mc I(\mu_{\ms P} \s n_{\g}) \sube \mc I(P) \s[18] \Lra \s[18] \phi_{\ms P}(x)^{\top} n_{\g} = 0 \s[18] \text{for} \; \mu \text{-almost all} \; x \in \mc X.
\]
Accordingly, we show:
\[
    \mc I(\g) \sube \mc I(\mu_{\ms P} \s n_{\g}).
\]

\noindent However, it is not difficult to construct and ADMG $\g$ such that $\mu_{\ms P} \s n_{\g}$ is not a structural imset. A bi-directed 6-cycle is an example; see Appendix \ref{app:extras}. When $\mu_{\ms P} \s n_{\g}$ is not a structural imset, we use the \textit{principle of inclusion-exclusion} to decompose $\mu_{\ms P} \s n_{\g}$ with respect to a consistent order $\leq$ into an inclusion structural imset $\mu_{\ms P} \s i_{\g}^{\leq}$ and an exclusion structural imset $\mu_{\ms P} \s e_{\g}^{\leq}$ such that $n_{\g} = i_{\g}^{\leq} - e_{\g}^{\leq}$. The exclusion imset is added to the factorization as an adjustment term which gives:
\[
    \mc I(\mu_{\ms P} \s i_{\g}^{\leq}) \sube \mc I(P) \s[18] \Lra \s[18] \phi_{\ms P}(x)^{\top} i_{\g}^{\leq} = 0 \s[18] \text{for} \; \mu \text{-almost all} \; x \in \mc X
\]
\noindent Accordingly, we show:
\[
    \mc I(\g) \sube \mc I(\mu_{\ms P} \s i_{\g}^{\leq}).
\]

Next we introduce Algorithms \ref{alg:pairs} and \ref{alg:nie} to facilitate the decomposition of the $\mu_{\ms P} \s n_{\g}$. Lemma \ref{lem:pairs_helper} provides two useful properties of parameterizing sets.

\begin{lem}
\label{lem:pairs_helper}
If $\g = (V, E)$ is an ADMG and $b \in \ba{\g}{V}$ and $A \in \mc A(\g)$, then\textup{:}
\begin{enumerate}
    \item $|\ceil{\{ M \in \mc M(\g) \; : \; b \in M \}}| = 1$\textup{;}
    \item $\mc M(\g[A]) = \{ M \in \mc P(A) \; : \; M \in \mc M(\g) \}$.
\end{enumerate}
\end{lem}

\noindent The intuition for Lemma \ref{lem:fact_helper} is given by the ordered local Markov property.

\begin{lem}
    \label{lem:fact_helper}
    Let $\g = (V, E)$ be an ADMG and $N \sube V$. If $b \in \ba{\g}{N}$ and $M = \ceil{\{ M \in \mc M(\g) \; : \; b \in M \sube N \}}$, then\textup{:}
    \[
        \istate{b}{N \sm M}{M \sm b}{\g}.
    \]
\end{lem}

\vskip 5mm

\begin{algorithm}[H]
    \caption{$\textsc{Pairs}(\g, b)$}
    \label{alg:pairs}
    \KwIn{ADMG: $\g[] = (V,E)$, \, barren vertex: $b \in \ba{\g}{V}$}
    \KwOut{ordered lists: $\ms M$, $\ms N$}
    Initialize ordered lists $\ms M = \seq{}$, $\ms N = \seq{}$\;
    Let $\ms R = \{N \in \mc N(\g) \; : \; b \in N$ \} \;
    \Repeat{$\ms R = \es$}{
        Pick $N \in \ceil{\ms R}$\; 
        Let $M = \ceil{\{ M \in \mc M(\g) \; : \; b \in M \sube N \}}$ ; \hfill \tcp{unique by Lemma \ref{lem:pairs_helper}}
        Append $M$ to $\ms M$ and $N$ to $\ms N$ \;
        $\ms R = \ms R \sm \{ S \in \ms R \; : \; b \in S \sube N \; \text{and} \; S \not \sube M \}$ \;
    }
\end{algorithm}

\vskip 5mm

Let $\g = (V,E)$ be an ADMG and $b \in \ba{\g}{V}$. Additionally, let $\ms M, \ms N = \textsc{Pairs}(\g, b)$ be ordered lists of parameterizing and constrained sets, respectively:
\[
    \ms M = \seq{M_1, \dots, M_n }\s[72] \ms N = \seq{N_1, \dots, N_n}
\]
where $n = |\ms M|$. Additionally, we define the universe of sets with respect to $N_i$ and $b$:
\[
    \ms U_i \eq \bigcup_{\substack{S \sube N_i \\ b \in S}} \{S\}.
\]

We simplify notation and use $\ms N_{i,i}$ to define the set of sets that corresponds to the conditional independence statement \istate{b}{N_i \sm M_i}{M_i \sm b}{\g}. We call terms of this form \textit{base conditional independence terms} because they will form the basis for the principle of inclusion-exclusion. Let $A = b$, $B = N_i \sm M_i$, and $C = M_i \sm b$, then:
\begin{align*}
    ABC &= (N_i \sm M_i) \cup (M_i \sm b) \cup b \\
    &= N_i \\
    AC &= (M_i \sm b) \cup b \\
    &= M_i \\
    BC &= (N_i \sm M_i) \cup (M_i \sm b) \\
    &= N_i \sm b
\end{align*}
therefore
\[
    \ms N_{i,i} \eq \bigcup_{\substack{S \in \ms U_{i} \\ S \not \sube M_{i}}} \{ S \} = \bigcup_{\substack{S \sube N_i \\ b \in S \\ S \not \sube M_i}} \{ S \} = \bigcup_{\substack{S \sube ABC \\ S \not \sube AC \\ S \not \sube BC}} \{ S \}
\]
and
\[
    \delta_{\ms N_{i,i}} = \delta_{A, B \mid C}.
\]

Lemma \ref{lem:ncon_ie} motivates application of the principle of inclusion-exclusion to the base conditional independence terms.

\begin{lem}
    \label{lem:ncon_ie}
    Let $\g = (V, E)$ be an ADMG and $b \in \ba{\g}{V}$. If $\ms M, \ms N = \textsc{Pairs}(\g, b)$ where $|\ms M| = n$ and $\ms R = \{ N \in \mc N(\g) \; : \; b \in N \}$, then\textup{:}
    \[
        \bigcup_{i=1}^n \ms N_{i,i} = \ms R.
    \]
\end{lem}

\noindent Accordingly, we cover all constrained sets using sets of sets that represent conditional independence relations. 

Next, we generalize our notation to intersections of the sets in $\ms M$ and $\ms N$:
\[
    M_K \eq \bigcap_{k \in K} M_k \s[36] N_J \eq \bigcap_{j \in J} N_j \s[36] \ms U_J \eq \bigcup_{\substack{T \sube N_J \\ b \in T}} \{T\} \s[36] M_{J,K} \eq M_K \cap N_J.
\]

Similarly, we use $\ms N_{J,K}$ to define the set of sets that corresponds to the conditional independence statement \istate{b}{N_J \sm M_{J,K}}{M_{J,K} \sm b}{\g}. We call terms that fit this more general from \textit{conditional independence terms}. If $A = b$, $B = N_{J} \sm M_{J,K}$, and $C = M_{J,K} \sm b$, then:
\begin{align*}
    ABC &= b \cup (N_J \sm M_{J,K}) \cup (M_{J,K} \sm b) \\
    &= N_J \\
    AC &= b \cup (M_{J,K} \sm b) \\
    &= M_{J,K} \\
    BC &= (N_J \sm M_{J,K}) \cup (M_{J,K} \sm b) \\
    &= N_J \sm b
\end{align*}
therefore
\[
    \ms N_{J,K} \eq \bigcup_{\substack{S \in \ms U_{J} \\ S \not \sube M_{J,K}}} \{S\} = \bigcup_{\substack{S \sube N_J \\ S \not \sube M_{J,K} \\ S \not \sube N_J \sm b}} \{S\} = \bigcup_{\substack{S \sube ABC \\ S \not \sube AC \\ S \not \sube BC}} \{S\}
\]
and
\[
    \delta_{\ms N_{J,K}} = \delta_{A,B \mid C}.
\]

Next we define Algorithm \ref{alg:nie} and show that it decomposes the non-\textit{m}-connecting imsets into two structural imsets using the principle of inclusion exclusion---the inclusion and exclusion imsets. Notably, there may be some redundancy in the imsets constructed by Algorithm \ref{alg:nie}. In Appendix \ref{app:alt_nie} we present a non-redundant version of Algorithm \ref{alg:nie}, however, we have not verified that removing redundant terms does not change the induced independence model.

\vskip 5mm

\begin{algorithm}[H]
\caption{${\textsc{Non-m-connecting imset via In/Ex-clusion NIE}(\g, \leq)}$}
\label{alg:nie}
\KwIn{ADMG: $\g[] = (V,E)$, \, total order: $\leq$}
\KwOut{imsets: $i_{\g}^{\leq}$, $e_{\g}^{\leq}$}
Initialize imsets $i_{\g}^{\leq} : \mc P(V) \ra 0$, $e_{\g}^{\leq} : \mc P(V) \ra 0$ \;
Let $A = V$ \;
\Repeat{$A = \es$}{
    Let $b = \ceo{A}$ and $\ms M, \ms N = \textsc{Pairs}(\g[A],b)$ \;
    \ForEach{$J \sube \{1, \dots, |\ms M| \}$}{
        \ForEach{$K \sube J$ \textup{where} $K \neq \es$}{
            \uIf{$|J \sm K|$ \textup{is even}}{    
                $i_{\g}^{\leq} = i_{\g}^{\leq} + \delta_{\ms N_{J,K}}$ \;
            }\Else{
                $e_{\g}^{\leq} = e_{\g}^{\leq} + \delta_{\ms N_{J,K}}$ \;
            }
        }
    }
    Remove $b$ from $A$ \;   
}
\end{algorithm}

\vskip 5mm

\begin{prop}
    \label{prop:nconn_decomp}
    Let $\g = (V,E)$ be an ADMG with consistent order $\leq$. If $i_{\g}^{\leq}, e_{\g}^{\leq} = \textsc{NIE}(\g, \leq)$ then\textup{:}
    \begin{enumerate}
        \item $n_{\g} = i_{\g}^{\leq} - e_{\g}^{\leq}$\textup{;}
        \item $\mu_{\ms P} \s i_{\g}^{\leq}, \mu_{\ms P} \s e_{\g}^{\leq} \in \mc S(V)$.
\end{enumerate}
\end{prop}

\vskip 5mm

In what follows, we provide an intuition for Proposition \ref{prop:nconn_decomp} with an illustrative example of Algorithm \ref{alg:nie}. Figure \ref{fig:workedoutex} depicts an ADMG $\g = (V,E)$, its parameterizing sets $\mc M(\g)$, and its constrained sets $\mc N(\g)$. Consider the total order $\leq$ over $V$ where $e \leq a \leq d \leq b \leq c$.

\begin{figure}[H]
\begin{center}
\includegraphics[page=12]{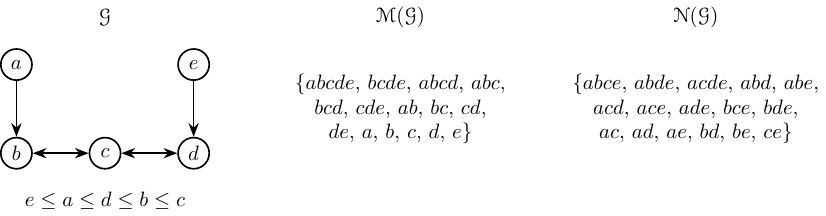}
\end{center}
\caption{An ADMG with vertices $\{a,b,c,d,e\}$ and the corresponding parameterizing and constrained sets for the ADMG.}
\label{fig:workedoutex}
\end{figure}

Run $\textsc{Pairs}(\g[abcde], c)$ to construct ordered lists $\ms N = \seq{\{a,b,c,e\}, \{a,c,d,e\}}$ and $\ms M = \seq{\{a,b,c\}, \{c,d,e\}}$.

\begin{figure}[H]
\begin{center}
\includegraphics[page=13]{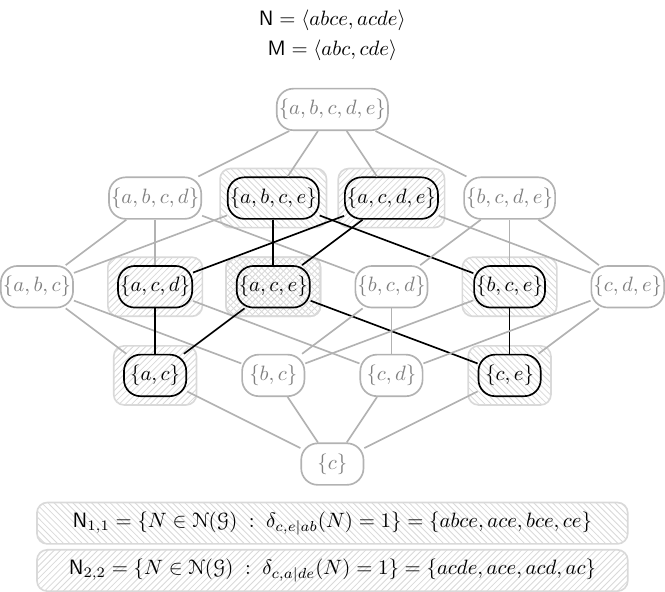}
\end{center}
\caption{A visualization of \textsc{Pairs}$(\g[abcde],c)$ applied to the ADMG in Figure \ref{fig:workedoutex} and the base elementary conditional independence terms. The constrained sets are bold and shaded to denote membership of a base conditional independence term.}
\label{fig:workedout_i}
\end{figure}

\noindent The intersection terms are as follows---these terms correspond to intersections over members of $\ms N$ and $\ms M$ indexed by the loops on lines 5 and 6 of Algorithm \ref{alg:nie}.

\begin{center}
$\ms N$ and $\ms M$ Intersection Terms:
\begin{align*}
N_{1} = \{a,b,c,e\} && M_{1} = \{a,b,c\} \\
N_{2} = \{a,c,d,e\} && M_{2} = \{c,d,e\} \\
N_{12} = \{a,c,e\} && M_{12} = \{c\} \\
\end{align*}
\end{center}

\noindent The conditional independence terms are as follows---these terms correspond to those added on lines 8 and 10.

\begin{center}
Conditional Independence Terms:
\begin{align*}
\ms N_{1,1} = \{ N \in \mc N(\g) \; ; \quad \delta_{c,e \mid ab}(N) = 1 \} \\
\ms N_{2,2} = \{ N \in \mc N(\g) \; ; \quad \delta_{c,a \mid cd}(N) = 1 \} \\
\ms N_{12,1} = \{ N \in \mc N(\g) \; ; \quad \delta_{c,e \mid a}(N) = 1 \} \\
\ms N_{12,2} = \{ N \in \mc N(\g) \; ; \quad \delta_{c,a \mid e}(N) = 1 \} \\
\ms N_{12,12} = \{ N \in \mc N(\g) \; ; \quad \delta_{c,ae}(N) = 1 \} \\
\end{align*}
\end{center}

\noindent The conditional independence terms are partitioned into inclusion and exclusion terms as follows---the inclusion terms are on the left and correspond to those added on line 8 and the exclusion terms are on the right and correspond to those added on line 10. \\

\begin{center}
\begin{minipage}{.45\textwidth}
\centering
Inclusion Terms:
\begin{align*}
\ms N_{1,1} = \{ N \in \mc N(\g) \; ; \quad \delta_{c,e \mid ab}(N) = 1 \} \\
\ms N_{2,2} = \{ N \in \mc N(\g) \; ; \quad \delta_{c,a \mid de}(N) = 1 \} \\
\ms N_{12,12} = \{ N \in \mc N(\g) \; ; \quad \delta_{c,ae}(N) = 1 \} \\
\end{align*}
\end{minipage}%
\begin{minipage}{.05\textwidth}
\hfill
\end{minipage}%
\begin{minipage}{.45\textwidth}
\centering
Exclusion Terms:
\begin{align*}
\ms N_{12,1} = \{ N \in \mc N(\g) \; ; \quad \delta_{c,e \mid a}(N) = 1 \} \\
\ms N_{12,2} = \{ N \in \mc N(\g) \; ; \quad \delta_{c,a \mid e}(N) = 1 \} \\
\vphantom{N_1} \\
\end{align*}
\end{minipage}%
\end{center}

\noindent The inclusion and exclusion imsets are updated accordingly:
\begin{align*}
    i_{\g}^{\leq} &= i_{\g}^{\leq} + \delta_{c,e \mid ab} \! + \, \delta_{c,a \mid de} \! + \, \delta_{c,ae} \\
    &= i_{\g}^{\leq} + \left [ \delta_{abce} + \delta_{ace} + \delta_{bce} + \delta_{ce} \right ] + \left [ \delta_{acde} + \delta_{acd} + \delta_{ace} + \delta_{ac} \right ] + \left [ \delta_{ace} + \delta_{ac} + \delta_{ce} \right ] \\
    e_{\g}^{\leq} &= e_{\g}^{\leq} + \delta_{c,e \mid a} \! + \, \delta_{c,a \mid e} \\
    &= e_{\g}^{\leq} + \left [ \delta_{ace} + \delta_{ce} \right ] + \left [ \delta_{ace} + \delta_{ac} \right ].
\end{align*}

\noindent The difference of $i_{\g}^{\leq}$ and $e_{\g}^{\leq}$ equals 1 for all constrained subsets of $\{a,b,c,d,e\}$ that contain $c$:
\[
    i_{\g}^{\leq}(S) - e_{\g}^{\leq}(S) \eq 
    \begin{cases}
        \s[18] 1 & \s[36] c \in S \in \mc N(\g); \\
        \s[18] 0 & \s[36] c \in S \not \in \mc N(\g).
    \end{cases}
\]

\noindent We encourage the reader to reference Figure \ref{fig:workedout_i} in order to verify this fact. \\

Run $\textsc{Pairs}(\g[abde], b)$ to construct ordered lists $\ms N = \seq{\{a,b,d,e\}}$ and $\ms M = \seq{\{a,b\}}$.

\begin{figure}[H]
\begin{center}
\includegraphics[page=14]{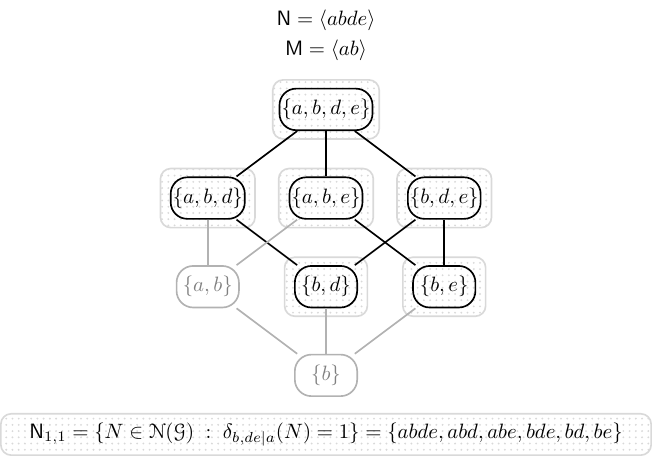}
\end{center}
\caption{A visualization of \textsc{Pairs}$(\g[abde],b)$ applied to the ADMG in Figure \ref{fig:workedoutex} and the corresponding conditional independence term. The constrained sets are bold and shaded to denote membership of the conditional independence term.}
\label{fig:workedout_ii}
\end{figure}

\noindent The intersection and inclusion terms are as follows---these terms correspond to intersections over members of $\ms N$ and $\ms M$ indexed by the loops on lines 5 and 6 of Algorithm \ref{alg:nie} and the terms added on line 8. \\

\begin{minipage}{\textwidth}
\begin{minipage}{.46\textwidth}
\centering
$\ms N$ and $\ms M$ Intersection Terms:
\vspace{-5mm}
\end{minipage}%
\begin{minipage}{.05\textwidth}
\hfill
\end{minipage}%
\begin{minipage}{.46\textwidth}
\centering
Inclusion Terms:
\vspace{-5mm}
\end{minipage}
\begin{minipage}{.46\textwidth}
\begin{align*}
N_1 = \{a,b,d,e\} && M_1 = \{a,b\} \\
\end{align*}
\end{minipage}%
\begin{minipage}{.05\textwidth}
\hfill
\end{minipage}%
\begin{minipage}{.46\textwidth}
\begin{align*}
\ms N_{1,1} = \{ N \in \mc N(\g) \; ; \quad \delta_{b,de \mid a}(N) = 1 \} \\
\end{align*}
\end{minipage}
\end{minipage}

\noindent The inclusion imset is updated accordingly (no update is made to the exclusion imset):
\begin{align*}
    i_{\g}^{\leq} &= i_{\g}^{\leq} + \delta_{b,de \mid a} \\
    &= i_{\g}^{\leq} + \left[ \delta_{abde} + \delta_{abd} + \delta_{abe} + \delta_{bde} + \delta_{bd} + \delta_{be} \right]
\end{align*}

\noindent The difference of $i_{\g}^{\leq}$ and $e_{\g}^{\leq}$ equals 1 for all constrained subsets of $\{a,b,d,e\}$ that contain $b$:
\[
    i_{\g}^{\leq}(S) - e_{\g}^{\leq}(S) \eq 
    \begin{cases}
        \s[18] 1 & \s[36] b \in S \in \mc N(\g); \\
        \s[18] 0 & \s[36] b \in S \not \in \mc N(\g).
    \end{cases}
\]

\noindent We encourage the reader to reference Figure \ref{fig:workedout_ii} in order to verify this fact. \\

Run $\textsc{Pairs}(\g[ade], d)$ to construct ordered lists $\ms N = \seq{\{a,d,e\}}$ and $\ms M = \seq{\{d,e\}}$.

\begin{figure}[H]
\begin{center}
\includegraphics[page=15]{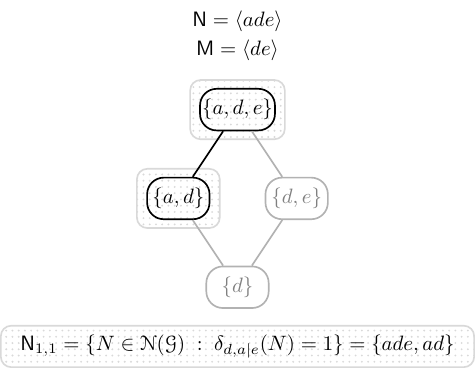}
\end{center}
\caption{A visualization of \textsc{Pairs}$(\g[ade],d)$ applied to the ADMG in Figure \ref{fig:workedoutex} and the corresponding conditional independence term. The constrained sets are bold and shaded to denote membership of the conditional independence term.}
\label{fig:workedout_iii}
\end{figure}

\noindent The intersection and inclusion terms are as follows---these terms correspond to intersections over members of $\ms N$ and $\ms M$ indexed by the loops on lines 5 and 6 of Algorithm \ref{alg:nie} and the terms added on line 8. \\

\begin{center}
\begin{minipage}{.46\textwidth}
\centering
Intersection Terms:
\vspace{-5mm}
\end{minipage}%
\begin{minipage}{.05\textwidth}
\hfill
\end{minipage}%
\begin{minipage}{.46\textwidth}
\centering
Positive Conditional Terms:
\vspace{-5mm}
\end{minipage}
\begin{minipage}{.46\textwidth}
\begin{align*}
N_1 = \{a,d,e\} && M_1 = \{d,e\} \\
\end{align*}
\end{minipage}%
\begin{minipage}{.05\textwidth}
\hfill
\end{minipage}%
\begin{minipage}{.46\textwidth}
\begin{align*}
\ms N_{1,1} = \{ N \in \mc N(\g) \; ; \quad \delta_{d,a \mid e}(N) = 1 \} \\
\end{align*}
\end{minipage}
\end{center}

\noindent The inclusion imset is updated accordingly (no update is made to the exclusion imset):
\begin{align*}
    i_{\g}^{\leq} &= i_{\g}^{\leq} + \delta_{d,a \mid e} \\
    &= i_{\g}^{\leq} + \left[ \delta_{ade} + \delta_{ad} \right]
\end{align*}

\noindent The difference of $i_{\g}^{\leq}$ and $e_{\g}^{\leq}$ equals 1 for all constrained subsets of $\{a,d,e\}$ that contain $d$:
\[
    i_{\g}^{\leq}(S) - e_{\g}^{\leq}(S) \eq 
    \begin{cases}
        \s[18] 1 & \s[36] d \in S \in \mc N(\g); \\
        \s[18] 0 & \s[36] d \in S \not \in \mc N(\g).
    \end{cases}
\]

\noindent We encourage the reader to reference Figure \ref{fig:workedout_iii} in order to verify this fact. \\

Run $\textsc{Pairs}(\g[ae], a)$ to construct ordered lists $\ms N = \seq{\{a,e\}}$ and $\ms M = \seq{\{a\}}$.

\begin{figure}[H]
\begin{center}
\includegraphics[page=16]{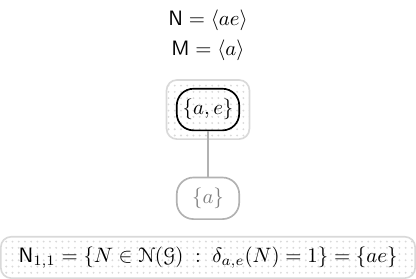}
\end{center}
\caption{A visualization of \textsc{Pairs}$(\g[ae],a)$ applied to the ADMG in Figure \ref{fig:workedoutex} and the corresponding conditional independence term. The constrained sets are bold and shaded to denote membership of the conditional independence term.}
\label{fig:workedout_iv}
\end{figure}

\noindent The intersection and inclusion terms are as follows---these terms correspond to intersections over members of $\ms N$ and $\ms M$ indexed by the loops on lines 5 and 6 of Algorithm \ref{alg:nie} and the terms added on line 8. \\

\begin{center}
\begin{minipage}{.46\textwidth}
\centering
Intersection Terms:
\vspace{-5mm}
\end{minipage}%
\begin{minipage}{.05\textwidth}
\hfill
\end{minipage}%
\begin{minipage}{.46\textwidth}
\centering
Positive Conditional Terms:
\vspace{-5mm}
\end{minipage}
\begin{minipage}{.46\textwidth}
\begin{align*}
N_1 = \{a,e\} && M_1 = \{a\} \\
\end{align*}
\end{minipage}%
\begin{minipage}{.05\textwidth}
\hfill
\end{minipage}%
\begin{minipage}{.46\textwidth}
\begin{align*}
\ms N_{1,1} = \{ N \in \mc N(\g) \; ; \quad \delta_{a,e}(N) = 1 \} \\
\end{align*}
\end{minipage}
\end{center}

\noindent The inclusion imset is updated accordingly (no update is made to the exclusion imset):
\begin{align*}
    i_{\g}^{\leq} &= i_{\g}^{\leq} + \delta_{ae}
\end{align*}

\noindent The difference of $i_{\g}^{\leq}$ and $e_{\g}^{\leq}$ equals 1 for all constrained subsets of $\{a,e\}$ that contain $a$:
\[
    i_{\g}^{\leq}(S) - e_{\g}^{\leq}(S) \eq 
    \begin{cases}
        \s[18] 1 & \s[36] a \in S \in \mc N(\g); \\
        \s[18] 0 & \s[36] a \in S \not \in \mc N(\g).
    \end{cases}
\]

\noindent We encourage the reader to reference Figure \ref{fig:workedout_iv} in order to verify this fact. \\

Run $\textsc{Pairs}(\g[e], e)$ to construct ordered lists $\ms N = \seq{}$ and $\ms M = \seq{}$.

\begin{figure}[H]
\begin{center}
\includegraphics[page=17]{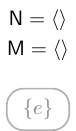}
\end{center}
\caption{A visualization of \textsc{Pairs}(\g[e],e) applied to the ADMG in Figure \ref{fig:workedoutex}.}
\label{fig:workedout_v}
\end{figure}

\noindent There are no intersection terms or updates to be made. Trivially, the difference of $i_{\g}^{\leq}$ and $e_{\g}^{\leq}$ equals 1 for all constrained subsets of $\{e\}$ that contain $e$:
\[
    i_{\g}^{\leq}(S) - e_{\g}^{\leq}(S) \eq 
    \begin{cases}
        \s[18] 1 & \s[36] e \in S \in \mc N(\g); \\
        \s[18] 0 & \s[36] e \in S \not \in \mc N(\g).
    \end{cases}
\]

\vskip 5mm

After all iterations of Algorithm \ref{alg:nie}:
\begin{align*}
i_{\g}^{\leq} &= \delta_{c,e \mid ab} \! + \, \delta_{c,a \mid de} \! + \, \delta_{c,ae} \! + \, \delta_{b,de \mid a} \! + \, \delta_{d,a \mid e} \! + \, \delta_{a,e} \\
&= \left [ \delta_{abce} + \delta_{ace} + \delta_{bce} + \delta_{ce} \right ] + \left [ \delta_{acde} + \delta_{acd} + \delta_{ace} + \delta_{ac} \right ] + \left [ \delta_{ace} + \delta_{ac} + \delta_{ce} \right ] \\
e_{\g}^{\leq} &= \delta_{c,e \mid a} \! + \, \delta_{c,a \mid e} \\
&= \left [ \delta_{ace} + \delta_{ce} \right ] + \left [ \delta_{ace} + \delta_{ac} \right ].
\end{align*}
Lastly, we apply the the M{\"o}bius inversion:
\begin{align*}
\mu_{\ms P} \s i_{\g}^{\leq} &= u_{\seq{c,e \mid ab}} + u_{\seq{c,a \mid de}} + u_{\seq{c,ae}} + u_{\seq{b,de \mid a}} + u_{\seq{d,a \mid e}} + u_{\seq{a,e}} \\
&= \left [ \delta_{abce} + \delta_{ab} - \delta_{abc} - \delta_{abe} \right ] + \left [ \delta_{acde} + \delta_{de} - \delta_{ade} - \delta_{cde} \right ] + \left [ \delta_{ace} - \delta_{ae} - \delta_{c} \right ] \\
&\quad + \left[ \delta_{abde} + \delta_{a} - \delta_{ade} - \delta_{ab} \right ] + \left[ \delta_{ade} + \delta_{e} - \delta_{ae} + \delta_{de} \right ] + \left [ \delta_{ae} - \delta_{a} - \delta_{e} \right ] \\
\mu_{\ms P} \s e_{\g}^{\leq} &= u_{\seq{c,e \mid a}} + u_{\seq{c,a \mid e}} \\
&= \left [ \delta_{ace} + \delta_{a} - \delta_{ac} - \delta_{ae} \right ] + \left [ \delta_{ace} + \delta_{e} - \delta_{ae} - \delta_{ce} \right ].
\end{align*}
Clearly $\mu_{\ms P} \s i_{\g}^{\leq}, \mu_{\ms P} \s e_{\g}^{\leq} \in \mc S(V)$ and $n_{\g} = i_{\g}^{\leq} - e_{\g}^{\leq}$.

\vskip 5mm

Lemma \ref{lem:olmp} shows that for an ADMG $\g$ with consistent order $\leq$, the independence statements represented in the structural imset $\mu_{\ms P} \s i_{\g}^{\leq}$ imply the independence statements required by the ordered local Markov property.

\begin{lem}
\label{lem:olmp}
Let $\g = (V, E)$ be an ADMG and $\leq$ be a total order consistent with $\g[]$. If $i_{\g}^{\leq}, e_{\g}^{\leq} = \textsc{NIE}(\g,\leq)$, $A \in \mc A(\g)$, and $b = \ceo{A}$, then\textup{:}
\[
    \istate{b}{A \sm \cl{\g[A]}{b}}{\mb{\g[A]}{b}}{\mu_{\ms P} \s i_{\g}^{\leq}}.
\]
\end{lem}

\noindent Accordingly, we get the first of the following two results.

\begin{prop}
\label{prop:in_ex_indmodels}
Let $\g = (V, E)$ be an ADMG, $\leq$ be a total order consistent with $\g[]$. If $i_{\g}^{\leq}, e_{\g}^{\leq} = \textsc{NIE}(\g,\leq)$, then\textup{:}
\begin{enumerate}
    \item $\mc I(\mu_{\ms P} \s i_{\g}^{\leq}) = \mc I(\g)$\textup{;}
    \item $\mc I(\mu_{\ms P} \s e_{\g}^{\leq}) \sube \mc I(\g)$.
\end{enumerate}
\end{prop}

\begin{cor}
    \label{cor:fact}
    If $\g$ be an ADMG with consistent order $\leq$, then\textup{:}
    \begin{enumerate}
        \item $\mc I(\g) \sube \mc I(P) \s[18] \Lra \s[18] \phi_{\ms P}(x)^{\top} i_{\g}^{\leq} = 0 \s[18]$ for $\mu$-almost all $x \in \mc X$\textup{;}
        \item $\mc I(\g) \sube \mc I(P) \s[18] \Ra \s[18] \phi_{\ms P}(x)^{\top} e_{\g}^{\leq} = 0 \s[18]$ for $\mu$-almost all $x \in \mc X$.
    \end{enumerate}
\end{cor}

\noindent Combining all the results above, we now state the main theorem of this paper.

\begin{thm}
\label{thm:mconn_fact}
Let $\g = (V, E)$ be an ADMG with consistent order $\leq$. If $P$ is a positive measure, then the following are equivalent\textup{:}
\begin{enumerate}
    \item $\log f(x) = \sum \limits_{M \in \mc M(\g)} \phi_{M}(x) - \phi_{\ms P}(x)^{\top} e_{\g}^{\leq} \s[18]$ for $\mu$-almost all $x \in \mc X$\textup{;}
    \item \istate{A}{B}{C}{\g} $\s[18] \Ra \s[18]$ \istate{A}{B}{C}{P} $\s[18]$ for all $\seq{A,B \mid C} \in \mc T(V)$.
\end{enumerate}
\end{thm}

\noindent Alternatively, the adjusted sum over the parameterizing sets may be computed as an adjusted sum over conditional densities. Let $\g = (V, E)$ be an ADMG with consistent order $\leq$. If $R = \pre{\g}{b}$ where context clarifies $b \in V$, then:
\begin{align*}
    \sum \limits_{M \in \mc M(\g)} \phi_{M}(x) - \phi_{\ms P}(x)^{\top} e_{\g}^{\leq} &= \sum_{S \in \mc P(V)} [m_{\g}(S) - e_{\g}^{\leq}(S)] \, \phi_S(x) \\
    &= \sum_{b \in V} \sum_{\substack{S \sube \cl{\g[R]}{b} \\ b \in S}} [1 - i_{\g}^{\leq}(S)] \, \phi_S(x) \\
    &= \sum_{b \in V} \bigg [ \log f_{b \mid \mb{\g[R]}{b}}(x) - \sum_{\substack{S \sube \cl{\g[R]}{b} \\ b \in S}} i_{\g}^{\leq}(S) \, \phi_{S}(x) \bigg ].
\end{align*}

\noindent This alternative formulation of the \textit{m}-connecting factorization in conjunction with the notion of a \textit{dominating DAG} gives a nice intuition.

\begin{defn}[\textit{dominating DAG}]
Let $\g = (V, E)$ be an ADMG with consistent order $\leq$. If $R = \pre{\g}{b}$ where context clarifies $b \in V$, then the corresponding \textit{dominating DAG} is defined:
\[
    \dom{\g} \eq (V, \{ (a,b) \in \mc E(V) \; : \; a \in \mb{\g[R]}{b} \}).
\]
\end{defn}

\noindent If $\g = (V, E)$ be an ADMG with consistent order $\leq$ and dominating DAG $\mc D = \dom{\g}$, then:
\[
    \log f(x) = \sum_{b \in V} \bigg [ \log f_{b \mid \pa{\mc D}{b}}(x) - \sum_{S \sube \pa{\mc D}{b}} i_{\g}^{\leq}(bS) \, \phi_{bS}(x) \bigg ].
\]

\noindent Using this formulation, the \textit{m}-connecting factorization becomes the well-known recursive factorization for the dominating DAG with a few additional conditional mutual information rate terms to adjust for any superfluous dependencies induced by $\mc D$. 

\section{Consistent Scoring Criterion}
\label{sec:consist_score}

In this section, we discuss an application of the \textit{m}-connecting factorization. In particular, we formulate a consistent scoring criterion with a closed-form solution for many curved exponential families, such as Gaussian and multinomial ADMG models. We investigate asymptotic properties of the score and compare its ability to recover correct independence models against the well-known FCI algorithm and two of its variants

\begin{defn}[\textit{exponential family}]
An \textit{exponential family} is a family of probability measures that admit densities with respect to $\sigma$-finite measure $\mu$:
\[
    f(x \mid \theta) = \exp \left[ \theta^\top t(x) - \psi(\theta) \right] \s[18] \textup{for all} \; \theta \in \Theta \; \textup{where}
\]
\begin{itemize}
    \item $\Theta \eq \{ \theta \in \mbb R^{|t(x)|} \, : \, \int_{x \in \mc X} \exp \left[ \theta^\top t(x) \right] \textup{d}\mu(x) < \infty \}$ is the \textit{natural parameter space};
    \item $t(x)$ is the \textit{sufficient statistic};
    \item $\psi(\theta) \eq \int_{x \in \mc X} \exp \left[ \theta^\top t(x) \right] \textup{d}\mu(x)$ is the \textit{cumulant function}.
\end{itemize}
Let $P_{\theta}$ denote a member of some curved exponential family parameterized by $\theta$. 
\end{defn}

Let $\g = (V, E)$ be an ADMG. We are interested in exponential families restricted to positive measures whose conditional independence relations are represented in $\g$. We denote the general parameter space constrained by $\mc I(\g)$:
\[
    \Theta_{\g} \eq \{ \theta \in \Theta \; : \; \mc I(\g) \sube \mc I(P_{\theta}) \}.
\]
When $\Theta_{\g}$ is a smooth manifold, the family of probability measures $\{ P_\theta \; : \; \theta \in \Theta_{\g} \}$ is a \textit{curved exponential family}; see \cite{kass1997geometrical}.

\begin{prop}[\cite{richardson2002ancestral}]
\label{prop:gadmg}
If $\g = (V,E)$ is an ADMG and $\mc F_\textup{G}(\g) \eq \{ P_\theta \; : \; \theta \in \Theta_{\g} \}$ is a family of Gaussian ADMG models, then $\mc F_\textup{G}(\g)$ is a curved exponential family with dimension\textup{:}
\[
    |\Theta_{\g}| = |V| + |\{ M \in \mc M(\g) \; : \; |M| \in \{ 1, 2\} \}|.
\]
\end{prop}

\begin{prop}[\cite{evans2014markovian}]
\label{prop:madmg}
If $\g = (V,E)$ is an ADMG and $\mc F_\textup{M}(\g) \eq \{ P_\theta \; : \; \theta \in \Theta_{\g} \}$ is a family of multinomial ADMG models, then $\mc F_\textup{M}(\g)$ is a curved exponential family with dimension\textup{:}
\[
    |\Theta_{\g}| = \sum_{H \in \mc H(\g)} |\mc X_{\ta{\g}{H}}| \prod_{h \in H} (|\mc X_{h}| - 1)
\]
where the convention $|\mc X_{\es}| = 1$ is accepted.
\end{prop}

Let $X$ be a collection of random variables indexed by $V$ with positive measure $P_\theta$ that admits density $f(x \mid \theta)$ with respect to dominating $\sigma$-finite product measure $\mu$. Throughout this section, let $P_{\theta}$ belong to a curved exponential family with parameter space $\Theta$ and $x^1, \dots, x^n \overset{iid}{\sim} f(x \mid \theta)$. The log-likelihood for parameter estimate $\hat{\theta} \in \Theta$ is defined:
\[
\ell(\hat{\theta} \mid x^1, \dots, x^n) \eq \sum_{i=1}^n \log f(x^i \mid \hat{\theta}). 
\]
Moreover, the maximum log-likelihood estimate of the parameters is defined:
\[
    \hat{\theta}_{n}^\text{mle} \eq \argmax_{\hat{\theta} \in \Theta} \ell(\hat{\theta} \mid x^1, \dots, x^n)
\]
\noindent and the maximum log-likelihood estimate of the parameters constrained by $\g$ is defined:
\[
    \hat{\theta}_{\g,n}^\text{mle} \eq \argmax\limits_{\hat{\theta} \in \Theta_{\g}} \ell(\hat{\theta} \mid x^1, \dots, x^n).
\]
\noindent Let $\hat{\theta} \in \Theta$ be an estimate of the parameters. If $T = \ta{\g}{H}$ and $R = \pre{\g}{b}$ where context clarifies $H \in \mc H(\g)$ and $b \in V$, then we approximate the log-likelihood constrained by $\g$ using $\hat{\theta}$ and the \textit{m}-connecting factorization as:

\begin{align*}
    \hat{\ell}_{\g}(\hat{\theta} \mid x^1, \dots, x^n) &\eq \sum_{M \in \mc M(\g)} \sum_{i=1}^n \phi_{M}(x^i \mid \hat{\theta}) \\
    &= \sum_{H \in \mc H(\g)} \sum_{i=1}^n \phi_{H \mid T}(x^i \mid \hat{\theta}) \\
    &= \sum_{H \in \mc H(\g)} \sum_{S \sube H} (-1)^{|H \sm S|} \sum_{i=1}^n \log f_{ST}(x^i \mid \hat{\theta}) \\
    &= \sum_{b \in V} \sum_{\substack{H \in \mc H(\g[R]) \\ b \in H}} \sum_{\substack{S \sube H \\ b \not \in S}} (-1)^{|H \sm S| - 1} \sum_{i=1}^n \log f_{b \mid ST}(x^i \mid \hat{\theta})
\end{align*}
\noindent Similarly, if $\mc D = \dom{\g}$, then the log-likelihood constrained by $\g$ and $\leq$ is approximated using $\hat{\theta}$ and the \textit{m}-connecting factorization as:
\[
\hat{\ell}_{\g}^{\leq}(\hat{\theta} \mid x^1, \dots, x^n) \eq \sum_{b \in V} \bigg [ \sum_{i=1}^n \log f_{b \mid \pa{\mc D}{b}}(x^i \mid \hat{\theta}) \, - \sum_{S \sube \pa{\mc D}{b} } i_{\g}^{\leq}(bS) \sum_{i=1}^n \phi_{bS}(x^i \mid \hat{\theta}) \bigg ].
\]

\subsection{Bayesian Information Criterion}

The Bayesian information criterion (BIC) score provides a computationally efficient $O_p(1)$ approximation for the log marginal likelihood of an exponential family. \cite{schwarz1978estimating} originally formulated the BIC score for models defined by affine transformations of their natural parameter space which was later extended to curved exponential families by
\cite{haughton1988choice}. The BIC score of a model has a simple formulation using the model's maximum likelihood estimate and dimension:
\[
\text{BIC}(\g, x^1, \dots, x^n) \eq \ell(\hat{\theta}_{\g,n}^\text{mle} \mid x^1, \dots, x^n) - \frac{|\Theta_{\g}|}{2} \log(n).
\]
Notably, a positive measure $P_{\theta}$ satisfies the global Markov property with respect to an ADMG \g[] if and only if $\theta \in \Theta_{\g}$. Moreover, the BIC score is a consistent score for model selection.

\begin{prop}[Proposition 1.2 \cite{haughton1988choice}]
\label{prop:bic_consist}
Let $\g = (V, E)$ and $\g' = (V, E')$ be ADMGs. Furthermore, let $X$ be a collection of random variables indexed by $V$ with positive measure $P_\theta$ that admits density $f(x \mid \theta)$ with respect to dominating $\sigma$-finite product measure $\mu$. Let $x^1, \dots, x^n \overset{iid}{\sim} f(x \mid \theta)$. If either $\theta \in (\Theta_{\g'} \sm \Theta_{\g})$ or $\theta \in (\Theta_{\g} \cap \, \Theta_{\g'})$ and $|\Theta_{\g'}| < |\Theta_{\g}|$, then\textup{:}
\[
\lim_{n \ra \infty} \textup{Pr}(\textup{BIC}(\g, x^1, \dots, x^n) < \textup{BIC}(\g', x^1, \dots, x^n)) = 1.
\]
\end{prop}

\vskip 5mm

The BIC score has a closed-form solution for many curved exponential families, such as Gaussian and multinomial DAG models. Unfortunately, this is not always the case when the parameter space is constrained by the independence model induced by an ADMG. We develop an approximation for the BIC score, called the $\textup{BIC}_\textup{MF}$ score, that has a closed-form solution in both cases. Let $\mc D = \dom{\g}$, we define the $\textup{BIC}_\textup{MF}$ score using the approximate maximum log-likelihood characterized by the \textit{m}-connecting factorization:

\[
    \textup{BIC}_\textup{MF}(\g, \leq, x^1, \dots, x^n) \eq \begin{dcases*}
        \s[18] \hat{\ell}_{\g}(\hat{\theta}_{\mc D,n}^\text{mle} \mid x^1, \dots, x^n) - \frac{|\Theta_{\g}|}{2} \log(n) \s[36] & $|M| \leq 5$ for all $M \in \mc M(\g)$ \\
        \s[18] \hat{\ell}_{\g}^{\leq}(\hat{\theta}_{\mc D,n}^\text{mle} \mid x^1, \dots, x^n) - \frac{|\Theta_{\g}|}{2} \log(n) \s[36] & otherwise
    \end{dcases*}
\]

\noindent which has nice asymptotic properties and simplifies to BIC if $\g$ is a DAG.

\begin{lem}
\label{lem:consistant_score}
Let $\g = (V, E)$ and $\g' = (V, E')$ be ADMGs and $\leq$ and $\leq'$ be total orders consistent with $\g$ and $\g'$, respectively. Furthermore, let $X$ be a collection of random variables indexed by $V$ with positive measure $P_\theta$ that admits density $f(x \mid \theta)$ with respect to dominating $\sigma$-finite product measure $\mu$. Let $x^1, \dots, x^n \overset{iid}{\sim} f(x \mid \theta)$. \\

If and $\theta \in \Theta_{\g} \sm \Theta_{\g'}$, then:
\[
\lim_{n \ra \infty} \, \frac{1}{n} \left | \textup{BIC}_\textup{MF}(\g, \leq, x^1, \dots, x^n) - \textup{BIC}_\textup{MF}(\g', \leq', x^1, \dots, x^n) \right | = \mbb E_{P_\theta} \left[ \phi_{\ms P}(x)^{\top} i_{\g'}^{\leq} \right].
\]

If $x^1, \dots, x^n \overset{iid}{\sim} f(x \mid \theta)$ and $\theta \in \Theta_{\g} \cap \, \Theta_{\g'}$ with $|\Theta_{\g}| < |\Theta_{\g'}|$, then:
\[
\lim_{n \ra \infty} \frac{1}{\log n} \left | \textup{BIC}_\textup{MF}(\g, \leq, x^1, \dots, x^n) - \textup{BIC}_\textup{MF}(\g', \leq', x^1, \dots, x^n) \right | = \frac{|\Theta_{\g'}| - |\Theta_{\g}|}{2}.
\]
\end{lem}

\noindent The $\textup{BIC}_\textup{MF}$ is a consistent score for model selection; see Appendix \ref{app:bic} for a discussion of $\textup{BIC}_\textup{MF}$ use as an approximation of log marginal likelihoods for exponential families. Furthermore, this score is readily applied to Gaussian and multinomial ADMG models by Propositions \ref{prop:gadmg} and \ref{prop:madmg}.

\begin{prop}
\label{prop:consistant_score}
Let $\g = (V, E)$ and $\g' = (V, E')$ be ADMGs and $\leq$ and $\leq'$ be total orders consistent with $\g$ and $\g'$, respectively. Furthermore, let $X$ be a collection of random variables indexed by $V$ with positive measure $P_{\theta}$ that admits density $f(x \mid \theta)$ with respect to dominating $\sigma$-finite product measure $\mu$. If $x^1, \dots, x^n \overset{iid}{\sim} f(x \mid \theta)$ and $\theta \in \Theta_{\g} \sm \Theta_{\g'}$ or $\theta \in \Theta_{\g} \cap \, \Theta_{\g'}$ where $|\Theta_{\g}| < |\Theta_{\g'}|$, then\textup{:}
\[
\lim_{n \ra \infty} \textup{Pr}(\textup{BIC}_\textup{MF}(\g', \leq', x^1, \dots, x^n) < \textup{BIC}_\textup{MF}(\g, \leq, x^1, \dots, x^n)) = 1.
\]
\end{prop}

\subsection{Empirical Evaluation}

While the BIC score is a consistent scoring criterion, it does not always have a closed-form solution for ADMG models. In particular, the calculation of the maximum log-likelihood requires the development of model specific optimizers. Such optimizers have been developed for Gaussian and multinomial ADMG models by \cite{drton2004iterative, drton2009computing} and \cite{evans2014markovian}, respectively. In contrast, the approximate maximum log-likelihood calculation developed in this paper uses the maximum likelihood estimate of the parameters for a dominating DAG model. Accordingly, the $\textup{BIC}_\textup{MF}$ score has a closed form solution for many curved exponential families, such as Gaussian and multinomial ADMG models.

In this section, we compare the ability of the $\textup{BIC}_\textup{MF}$ score to recover the correct MEC for Gaussian ADMG models against the $\textup{BIC}_\textup{ICF}$---the BIC score where the maximum likelihood estimate of the parameters is approximated using an R implementation of the iterative conditional fitting (ICF) procedure\footnote{\texttt{https://CRAN.R-project.org/package=ggm}} \citep{drton2004iterative}. ICF optimizes the log-likelihood for curved exponential families constrained by an ADMG independence model, however, this space is not necessarily convex. Accordingly, ICF is not guaranteed to converge to the maximum likelihood estimate of the parameters. 

In addition to the $\textup{BIC}_\textup{ICF}$ score, we compare the $\textup{BIC}_\textup{MF}$ score's ability to recover the correct MEC for Gaussian ADMG models against three causal discovery algorithms:
\begin{itemize}
\item FCI \citep{spirtes1999algorithm, zhang2008completeness} using the 
zero partial correlation test of conditional independence.
\item $\text{FCI}_\text{max}$  \citep{raghu2018comparison} using Fisher's Z test of conditional independence.
\item $\text{GFCI}_{i}$ \citep{ogarrio2016hybrid} using the BIC score with penalty discount $i \in \{1,2\}$ and Fisher's Z test of conditional independence.
\end{itemize}
We used Java implementations of these algorithms from the TETRAD project\footnote{\texttt{https://github.com/cmu-phil/tetrad}}. In our simulations, we generated data from a subfamily of ADMGs called directed maximal ancestral graphs (MAGs); directed MAGs induce the same independence models as ADMGs and have a straightforward parameterization \citep{richardson2002ancestral}. Algorithm \ref{alg:sim} outlines the procedure used to simulate $n$ data instances with respect to directed MAG $\g$.

\vskip 5mm

\begin{algorithm}[H]
\caption{$\textsc{Simulate}(\g, n)$}
\label{alg:sim}
\KwIn{directed MAG: $\g = (V,E)$, number of instances: $n$}
\KwOut{data: $x^1, \dots x^n$}

\Repeat{$\Omega$ is symmetric and positive-definite}{
% $\Omega = \underset{a,b \in V}{(\omega_{ab})} \; \text{where} \; \omega_{ab} \sim \Bigg\{ \!\!
$\Omega = (\omega_{ab})_{a,b \in V} \; \text{where} \; \omega_{ab} \sim \Bigg\{ \!\!
\begin{array}{ll}
\text{Uniform} \, [-0.7,-0.3]\cup[0.3,0.7] & \quad \text{if} \; \{a, b\} \in E \\
\text{Uniform} \, [1.0,3.0] & \quad \text{if} \; a = b \\
0 & \quad \text{otherwise}
\end{array}$ \;
}
% $B^\top = \underset{a,b \in V}{(\beta_{ab})} \; \text{where} \; \beta_{ab} \sim \Big\{ \!\!
$B^\top = (\beta_{ab})_{a,b \in V} \; \text{where} \; \beta_{ab} \sim \Big\{ \!\!
\begin{array}{ll}
\text{Uniform} \, [-0.7,-0.3]\cup[0.3,0.7] & \quad \text{if} \; (a, b) \in E \\
0 & \quad \text{otherwise}
\end{array}$ \;
$\Sigma = (I-B)^{-1} \, \Omega \, (I-B)^{-\top}$ \;
$x^1, \dots, x^n \sim \text{Gaussian}(0, \Sigma)$ \;
\end{algorithm}

\vskip 5mm 

We use a Python implementation of $\textup{BIC}_\textup{MF}$ to exhaustively rank all ADMG MECs\footnote{\texttt{https://github.com/bja43/agMm}}. Histograms show the distribution of the rank of the data generating MEC. That is, the $i^\textup{th}$ bin counts the number of times the data generating MEC is ranked $i^\textup{th}$ among all models. Notably, there are 24,259 possible positions in the ranking for graphs with 5 vertices. 

For all experiments, we simulate datasets with 500, 5,000 and 50,000 instances. We run experiments for two prespecified graphs and two sets of random graphs with 5 vertices; The first set has between 0 and 5 edges and the second set has between 6 and 10 edges. For each scenario, we run 1,000 repetitions (100 for comparison to $\textup{BIC}_\textup{ICF}$). When not prespecified, graphs are resampled with equal probability assigned to each directed MAG satisfying the experimental parameters. The tables below report the proportion of times that algorithms recovered the correct MEC. These numbers may be compared to the counts reported in the first bin of each histogram divided by 1,000. All experiments were run on a system with the following hardware:
\begin{itemize}
\item Memory: 7.7 GiB
\item Processor: Intel\textsuperscript{\textregistered} Core\textsuperscript{TM} i5-5200U CPU @ 2.20GHz $\times$ 4
\end{itemize}

\begin{figure}[H]
\begin{center}
Directed MAG 
\vskip -5mm
\includegraphics{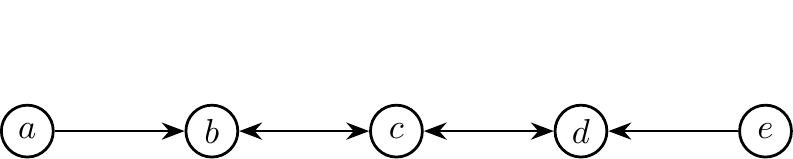}
\vskip 5mm
MEC Recovery
\vskip 2mm
\begin{minipage}{0.3\textwidth}
\centering
\includegraphics[width=0.9\textwidth]{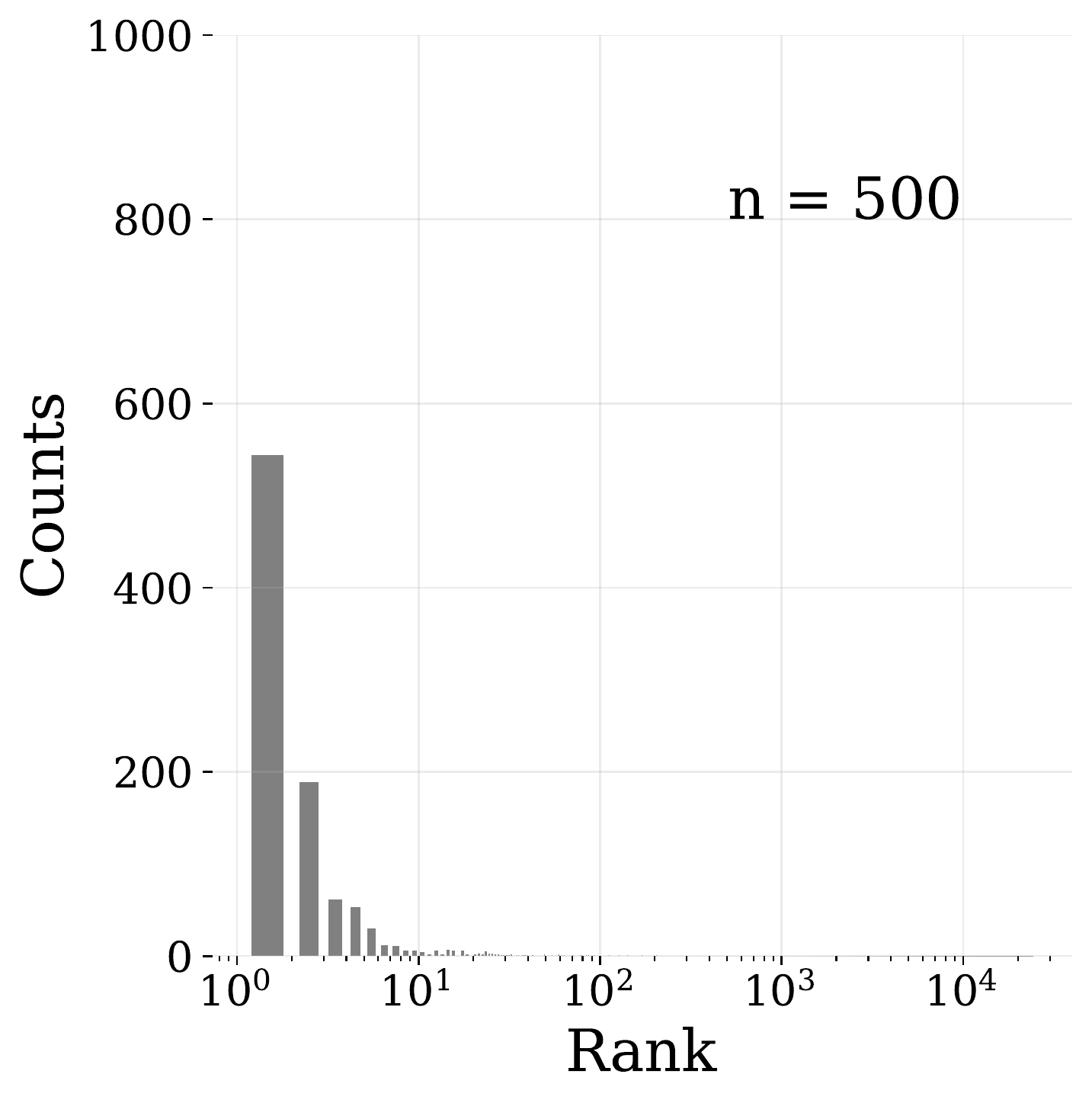}
\end{minipage}%
\begin{minipage}{0.3\textwidth}
\centering
\includegraphics[width=0.9\textwidth]{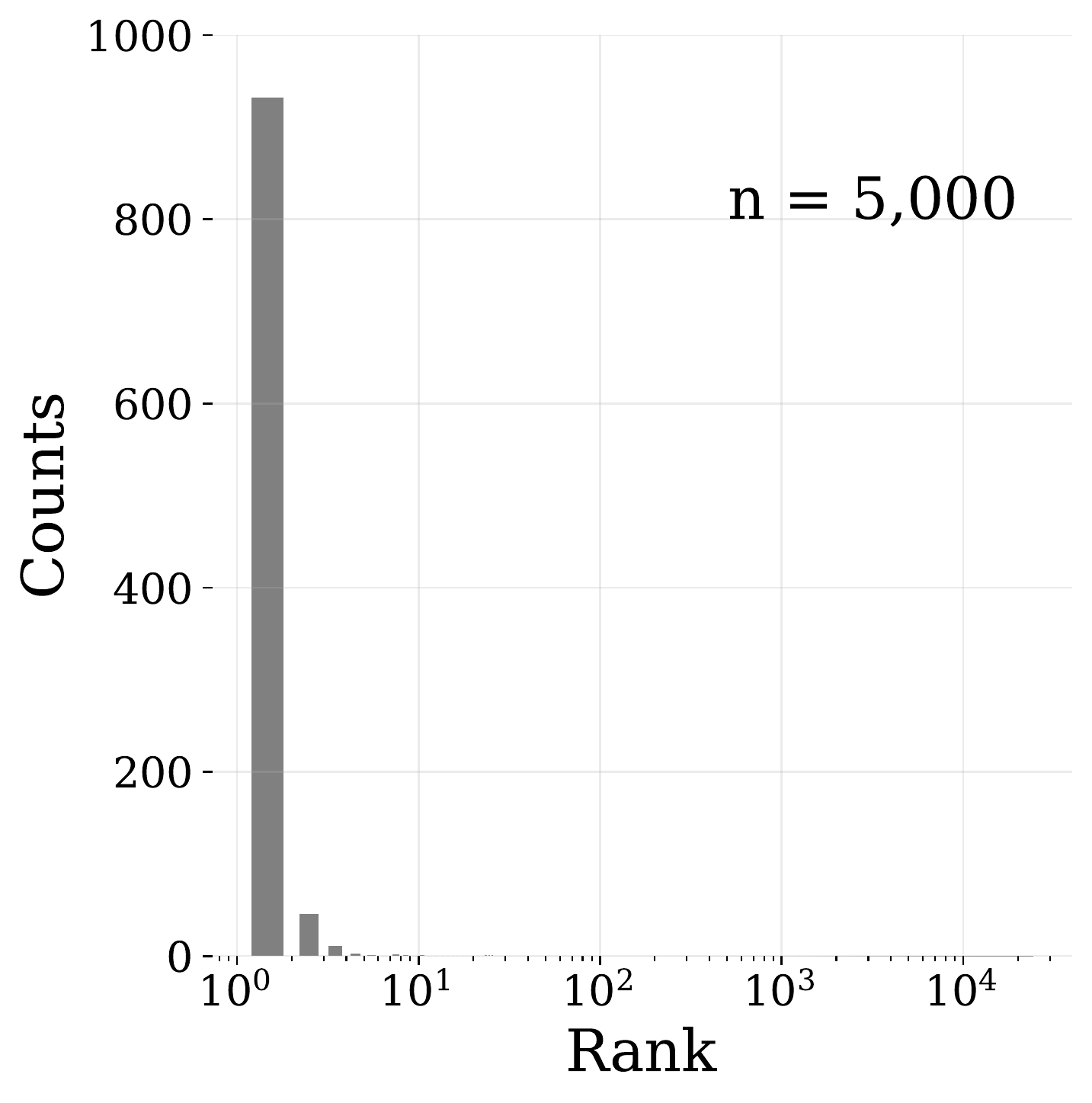}
\end{minipage}%
\begin{minipage}{0.3\textwidth}
\centering
\includegraphics[width=0.9\textwidth]{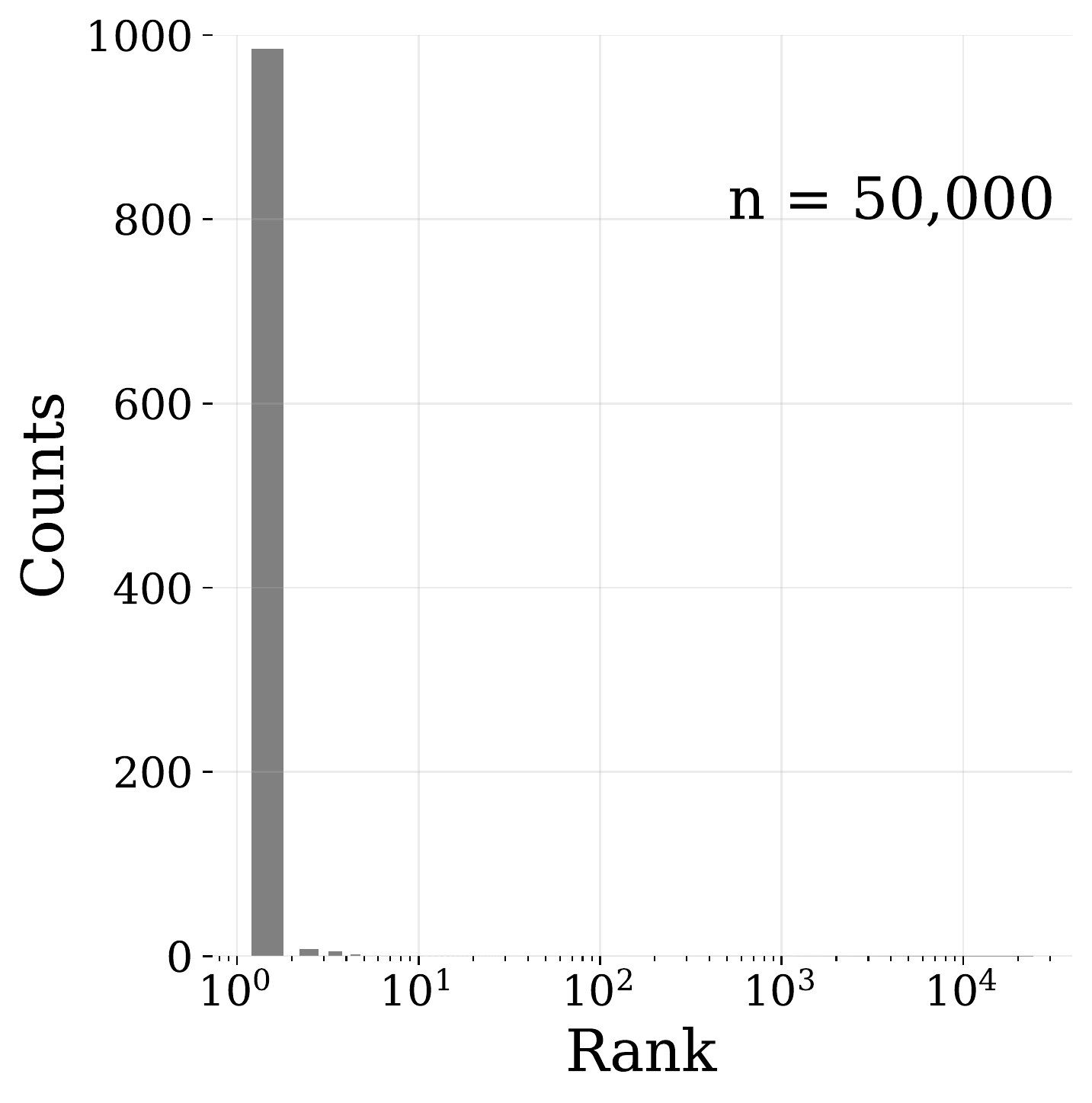}
\end{minipage}
\vskip 2mm
\begin{tabular}{c|c|c|c|c|c|c|c|c|c|c|c}
& $\textup{BIC}_\textup{ICF}$ & \multicolumn{2}{c|}{$\textup{BIC}_\textup{MF}$} & \multicolumn{2}{c|}{FCI} & \multicolumn{2}{c|}{FCI$_\text{max}$} & \multicolumn{2}{c|}{GFCI$_{1}$} & \multicolumn{2}{c}{GFCI$_{2}$} \\
\hline
repetitions & 100 & 100 & 1,000 & \multicolumn{2}{c|}{1,000} & \multicolumn{2}{c|}{1,000} & \multicolumn{2}{c|}{1,000} & \multicolumn{2}{c}{1,000} \\
\hline
$\alpha$-level & - & - & - & 0.01 & 0.001 & 0.01 & 0.001 & 0.01 & 0.001 & 0.01 & 0.001 \\
\hline
n = 500 & 0.49 & 0.47 & 0.544 & 0.836 & 0.702 & 0.499 & 0.444 & 0.223 & 0.213 & 0.08 & 0.08 \\
n = 5,000 & 0.91 & 0.93 & 0.932 & 0.978 & 0.999 & 0.933 & 0.948 & 0.779 & 0.786 & 0.575 & 0.576 \\
n = 50,000 & 1.0 & 1.0 & 0.985 & 0.975 & 0.998 & 0.975 & 0.998 & 0.99 & 0.991 & 0.977 & 0.977 
\end{tabular}
\end{center}
\caption{An evaluation of the $\textup{BIC}_\textup{MF}$ relative to the $\textup{BIC}_\textup{ICF}$ and state-of-the-art algorithms on data generated from the depicted directed MAG.}
\label{fig:case_5}
\end{figure}

In Figure \ref{fig:case_5}, the prespecified graph is a bi-directed chain of length five. The $\textup{BIC}_\textup{MF}$ score performs nearly identically to the $\textup{BIC}_\textup{ICF}$ score and consistently ranks the correct MEC in the top 10 with the ranking converging to a point-mass in the first bin as $n \ra \infty$. The top ranked model according to the $\textup{BIC}_\textup{MF}$ score was correct less often than the model returned by FCI, and more often than the model returned by GFCI for small sample sizes. 

\begin{figure}[H]
\begin{center}
Directed MAG 
\vskip -5mm
\includegraphics{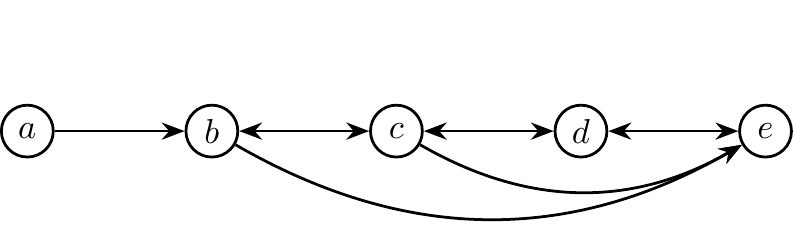}
\vskip 5mm
MEC Recovery
\vskip 2mm
\begin{minipage}{0.3\textwidth}
\centering
\includegraphics[width=0.9\textwidth]{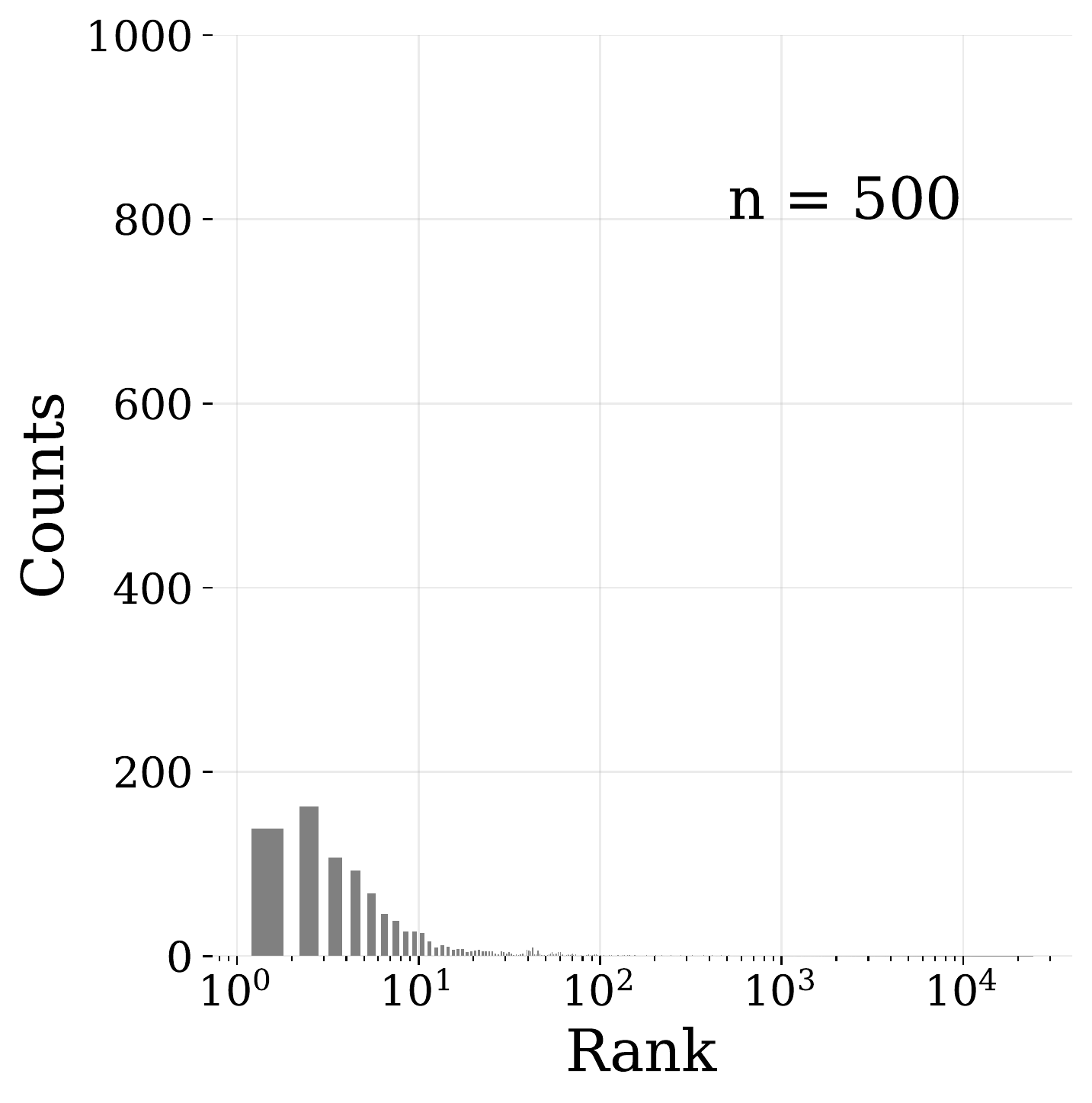}
\end{minipage}%
\begin{minipage}{0.3\textwidth}
\centering
\includegraphics[width=0.9\textwidth]{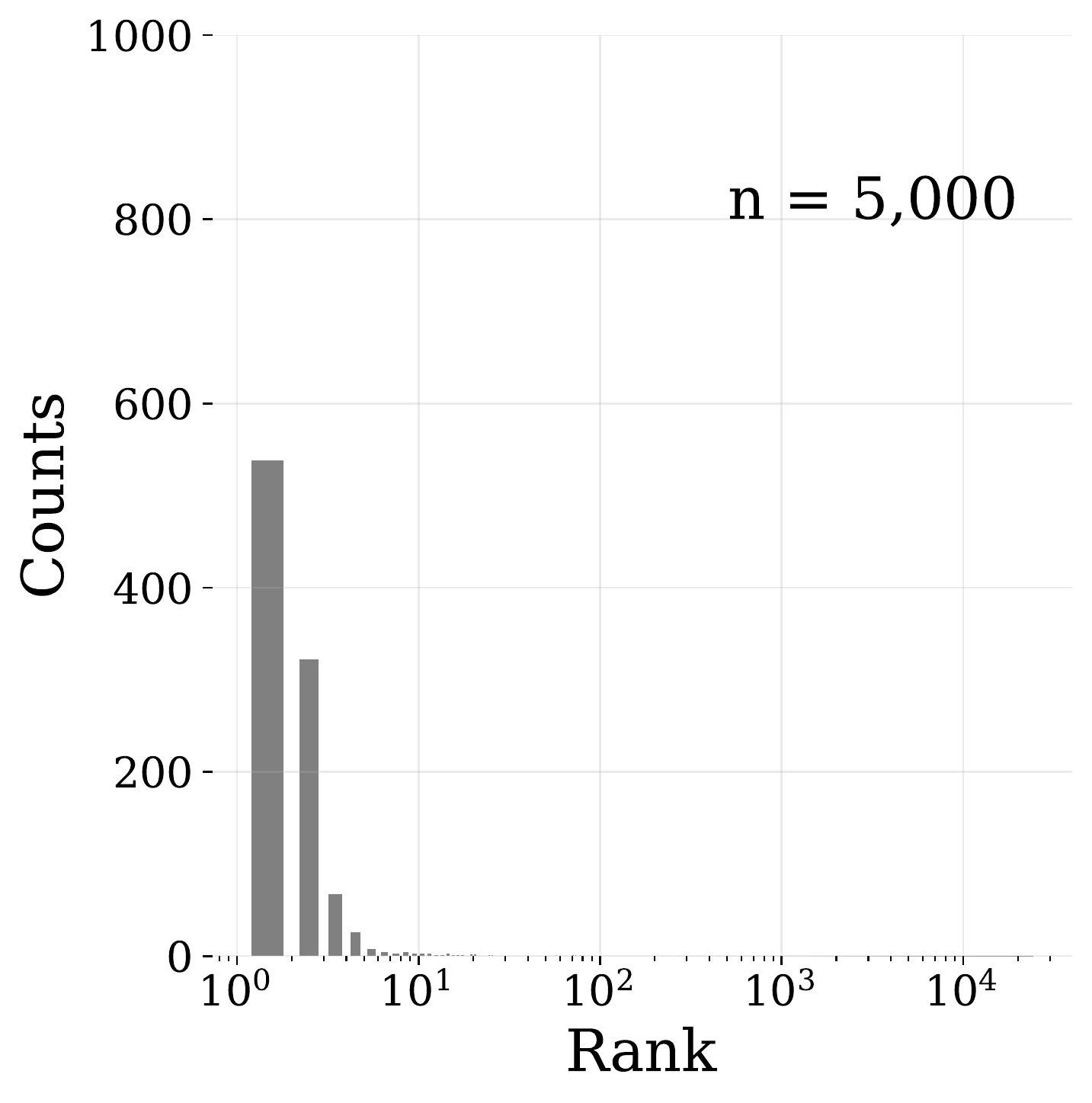}
\end{minipage}%
\begin{minipage}{0.3\textwidth}
\centering
\includegraphics[width=0.9\textwidth]{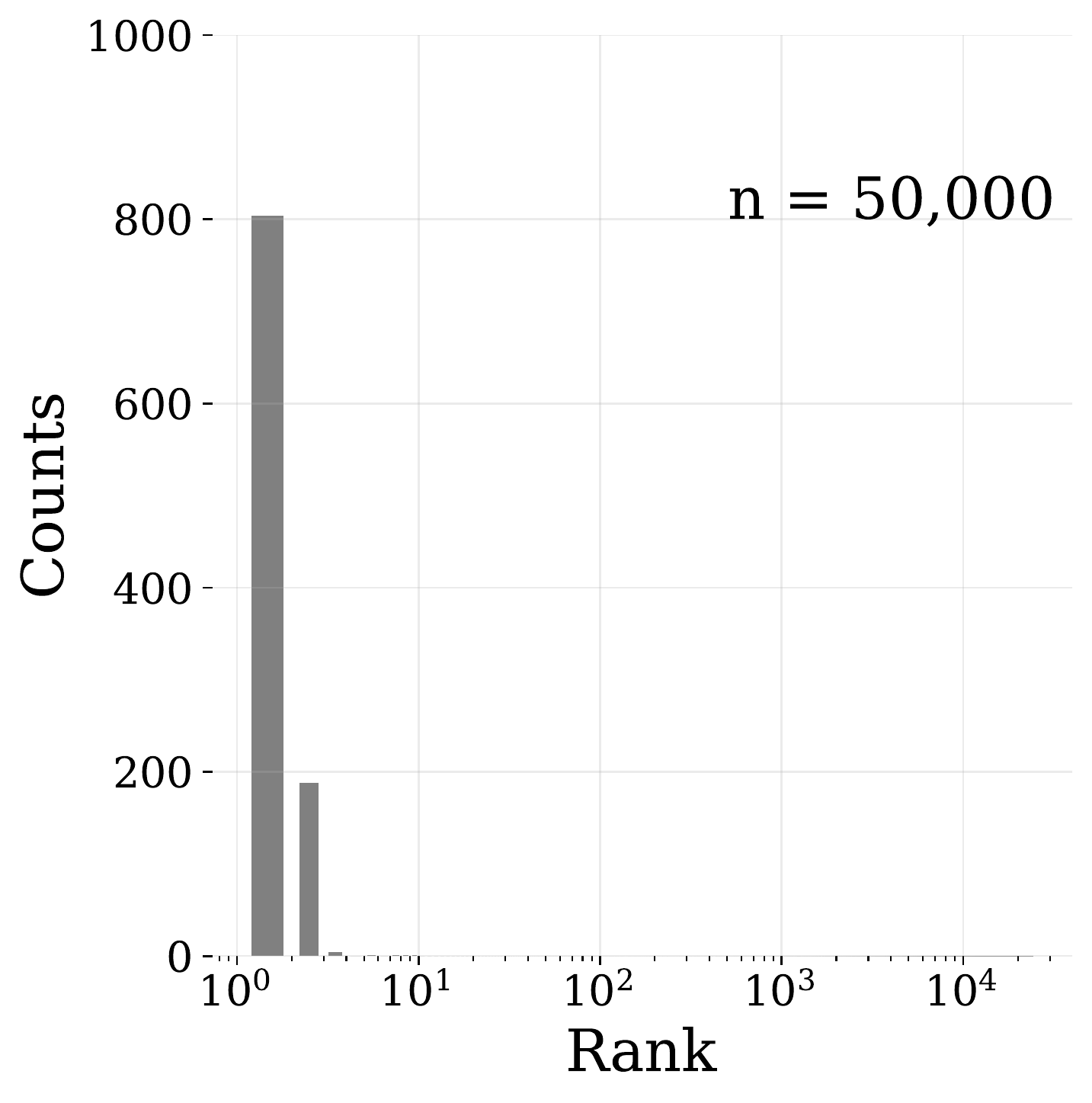}
\end{minipage}
\vskip 2mm
\begin{tabular}{c|c|c|c|c|c|c|c|c|c|c|c}
& $\textup{BIC}_\textup{ICF}$ & \multicolumn{2}{c|}{$\textup{BIC}_\textup{MF}$} & \multicolumn{2}{c|}{FCI} & \multicolumn{2}{c|}{FCI$_\text{max}$} & \multicolumn{2}{c|}{GFCI$_{1}$} & \multicolumn{2}{c}{GFCI$_{2}$} \\
\hline
repetitions & 100 & 100 & 1,000 & \multicolumn{2}{c|}{1,000} & \multicolumn{2}{c|}{1,000} & \multicolumn{2}{c|}{1,000} & \multicolumn{2}{c}{1,000} \\
\hline
$\alpha$-level & - & - & - & 0.01 & 0.001 & 0.01 & 0.001 & 0.01 & 0.001 & 0.01 & 0.001 \\
\hline
n = 500 & 0.12 & 0.14 & 0.138 & 0.006 & 0.0 & 0.005 & 0.0 & 0.0 & 0.0 & 0.0 & 0.0 \\
n = 5,000 & 0.53 & 0.54 & 0.538 & 0.052 & 0.028 & 0.048 & 0.028 & 0.041 & 0.029 & 0.032 & 0.02 \\
n = 50,000 & 0.77 & 0.79 & 0.804 & 0.307 & 0.23 & 0.296 & 0.221 & 0.304 & 0.247 & 0.28 & 0.221
\end{tabular}
\end{center}
\caption{An evaluation of the $\textup{BIC}_\textup{MF}$ relative to the $\textup{BIC}_\textup{ICF}$ and state-of-the-art algorithms on data generated from the depicted directed MAG.}
\label{fig:case_6}
\end{figure}

In Figure \ref{fig:case_6}, the prespecified graph contains a discriminating path of length five; see \citep{richardson2002ancestral} for details. The $\textup{BIC}_\textup{MF}$ score performs nearly identically to the $\textup{BIC}_\textup{ICF}$ score and consistently ranks the correct MEC in the top 10 with the ranking converging to a point-mass in the first bin as $n \ra \infty$. The top ranked model according to the $\textup{BIC}_\textup{MF}$ score was correct more often than the models returned by other methods.

\begin{figure}[H]
\begin{center}    
Random Directed MAGs with $|V| = 5$ and $|E| \in [0,5]$
\vskip 5mm
% Case (0, 5)
MEC Recovery
\begin{minipage}{0.3\textwidth}
\centering
\includegraphics[width=0.9\textwidth]{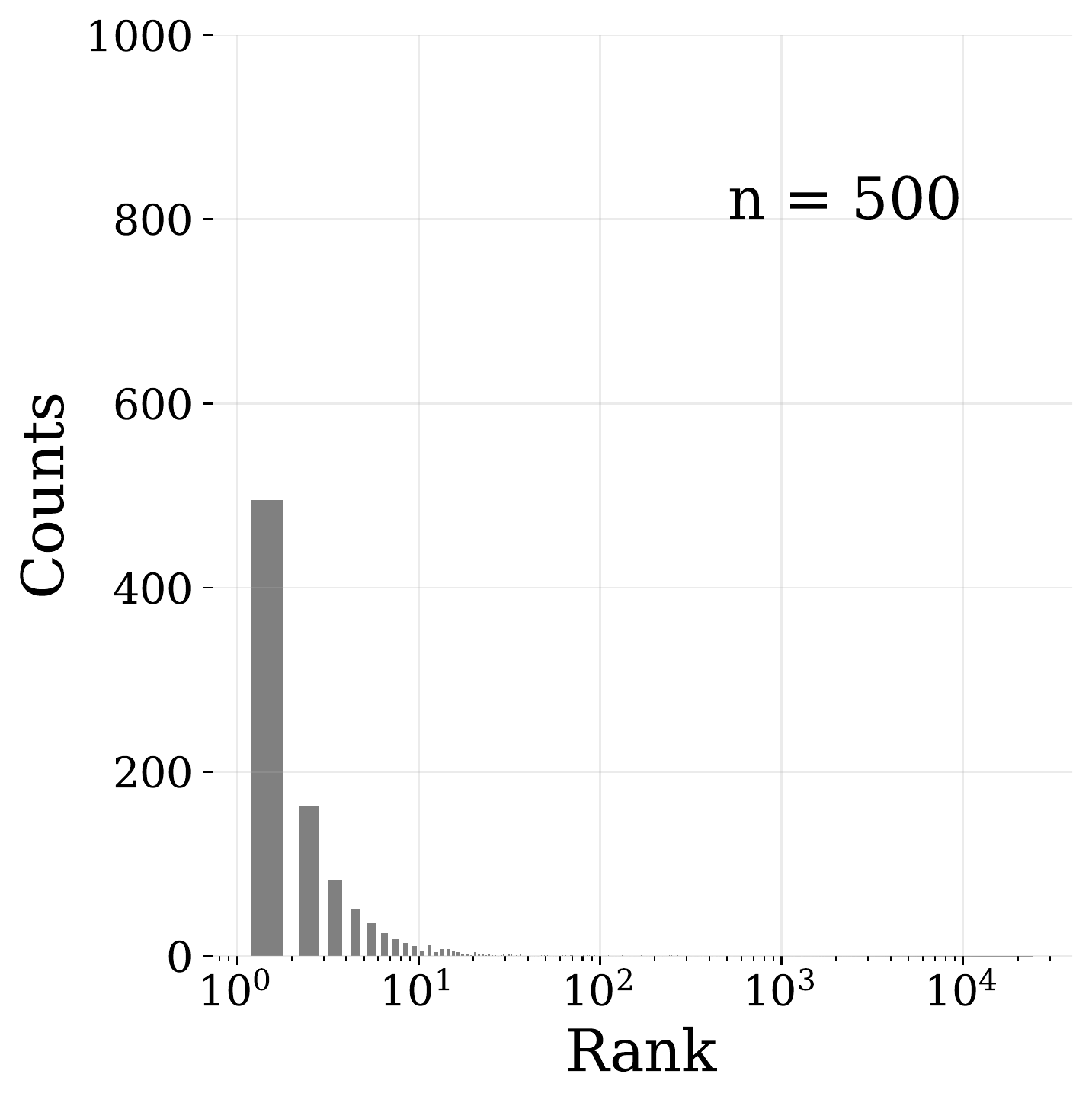}
\end{minipage}%
\begin{minipage}{0.3\textwidth}
\centering
\includegraphics[width=0.9\textwidth]{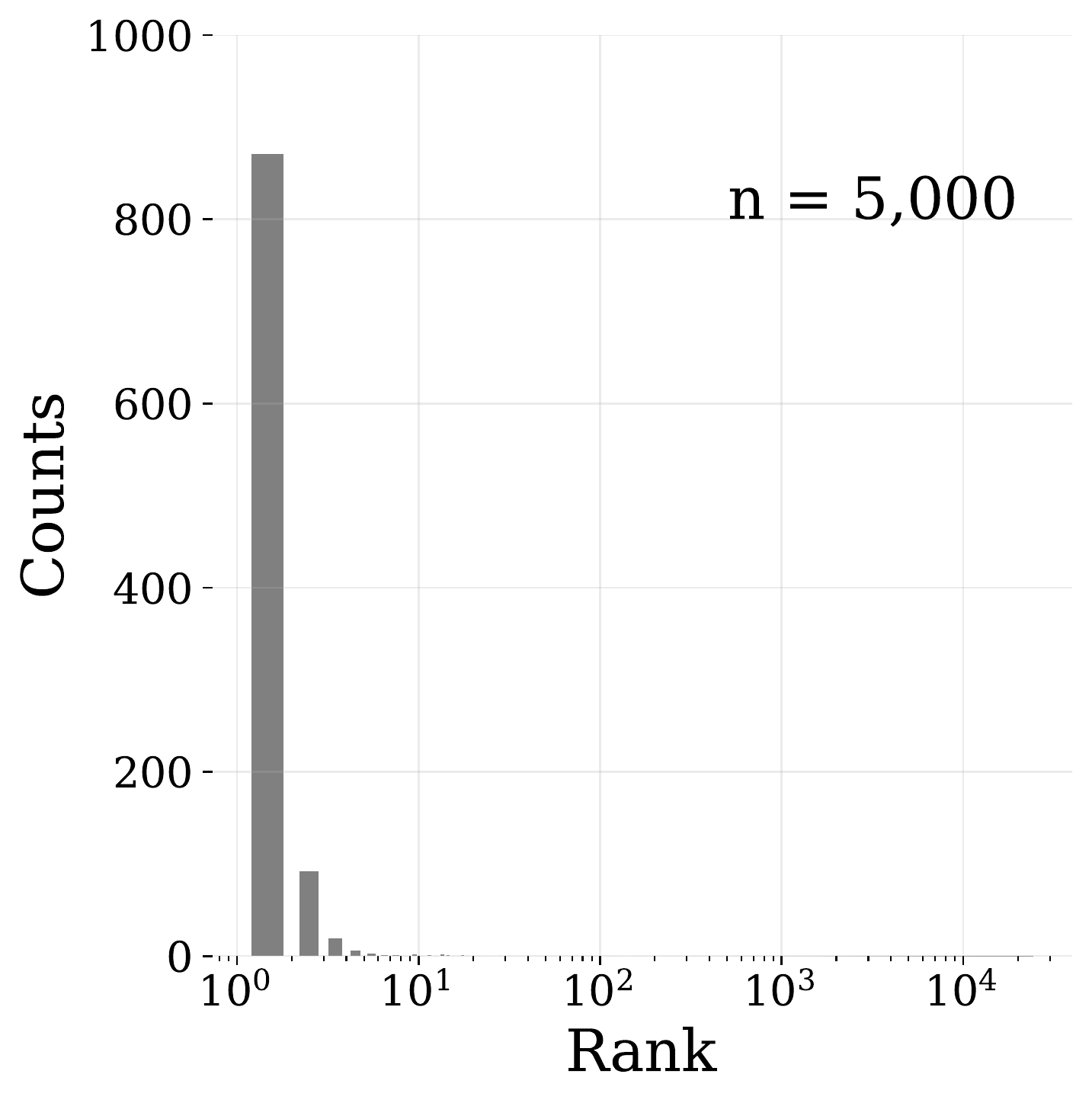}
\end{minipage}%
\begin{minipage}{0.3\textwidth}
\centering
\includegraphics[width=0.9\textwidth]{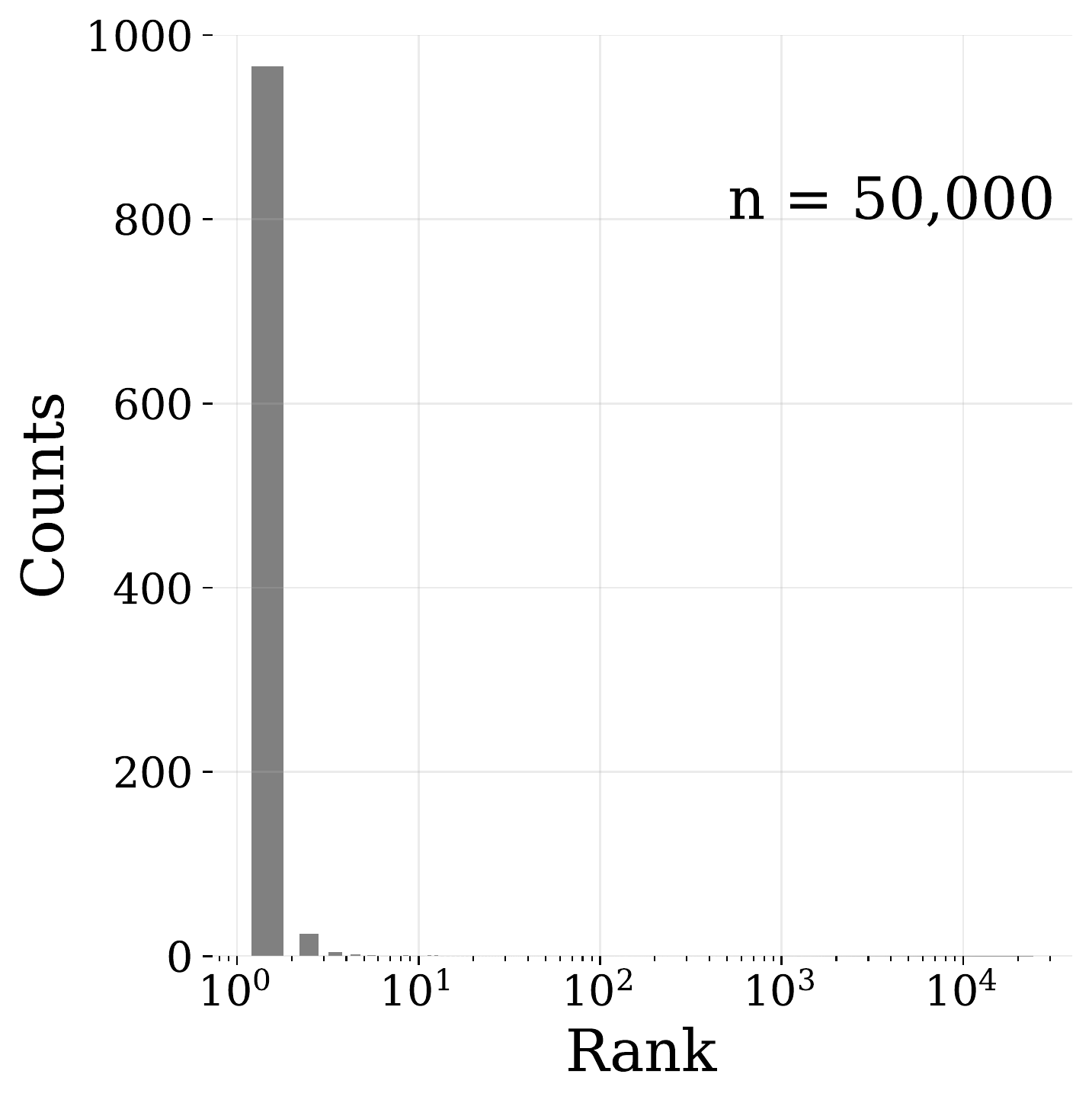}
\end{minipage}
\vskip 2mm
\begin{tabular}{c|c|c|c|c|c|c|c|c|c|c|c}
& $\textup{BIC}_\textup{ICF}$ & \multicolumn{2}{c|}{$\textup{BIC}_\textup{MF}$} & \multicolumn{2}{c|}{FCI} & \multicolumn{2}{c|}{FCI$_\text{max}$} & \multicolumn{2}{c|}{GFCI$_{1}$} & \multicolumn{2}{c}{GFCI$_{2}$} \\
\hline
repetitions & 100 & 100 & 1,000 & \multicolumn{2}{c|}{1,000} & \multicolumn{2}{c|}{1,000} & \multicolumn{2}{c|}{1,000} & \multicolumn{2}{c}{1,000} \\
\hline
$\alpha$-level & - & - & - & 0.01 & 0.001 & 0.01 & 0.001 & 0.01 & 0.001 & 0.01 & 0.001 \\
\hline
n = 500 & 0.57 & 0.58 & 0.495 & 0.273 & 0.16 & 0.398 & 0.338 & 0.349 & 0.324 & 0.258 & 0.252 \\
n = 5,000 & 0.9 & 0.88 & 0.871 & 0.776 & 0.741 & 0.788 & 0.78 & 0.741 & 0.732 & 0.676 & 0.671 \\
n = 50,000 & 0.99 & 0.99 & 0.966 & 0.919 & 0.915 & 0.92 & 0.919 & 0.904 & 0.9 & 0.893 & 0.888
\end{tabular}
\end{center}
\caption{An evaluation of the $\textup{BIC}_\textup{MF}$ relative to the $\textup{BIC}_\textup{ICF}$ and state-of-the-art algorithms on data generated using random directed MAGs with specified vertex and edge ranges.}
\label{fig:case_0_5}
\end{figure}

In Figure \ref{fig:case_0_5}, the $\textup{BIC}_\textup{MF}$ score performs nearly identically to the $\textup{BIC}_\textup{ICF}$ score and consistently ranks the correct MEC in the top 10 with the ranking converging to a point-mass in the first bin as $n \ra \infty$. The top ranked model according to the $\textup{BIC}_\textup{MF}$ score was correct more often than the models returned by other methods.

\begin{figure}[H]
\begin{center}    
Random Directed MAGs with $|V| = 5$ and $|E| \in [6,10]$
\vskip 5mm
% Case (6, 10)
MEC Recovery
\vskip 2mm
\begin{minipage}{0.3\textwidth}
\centering
\includegraphics[width=0.9\textwidth]{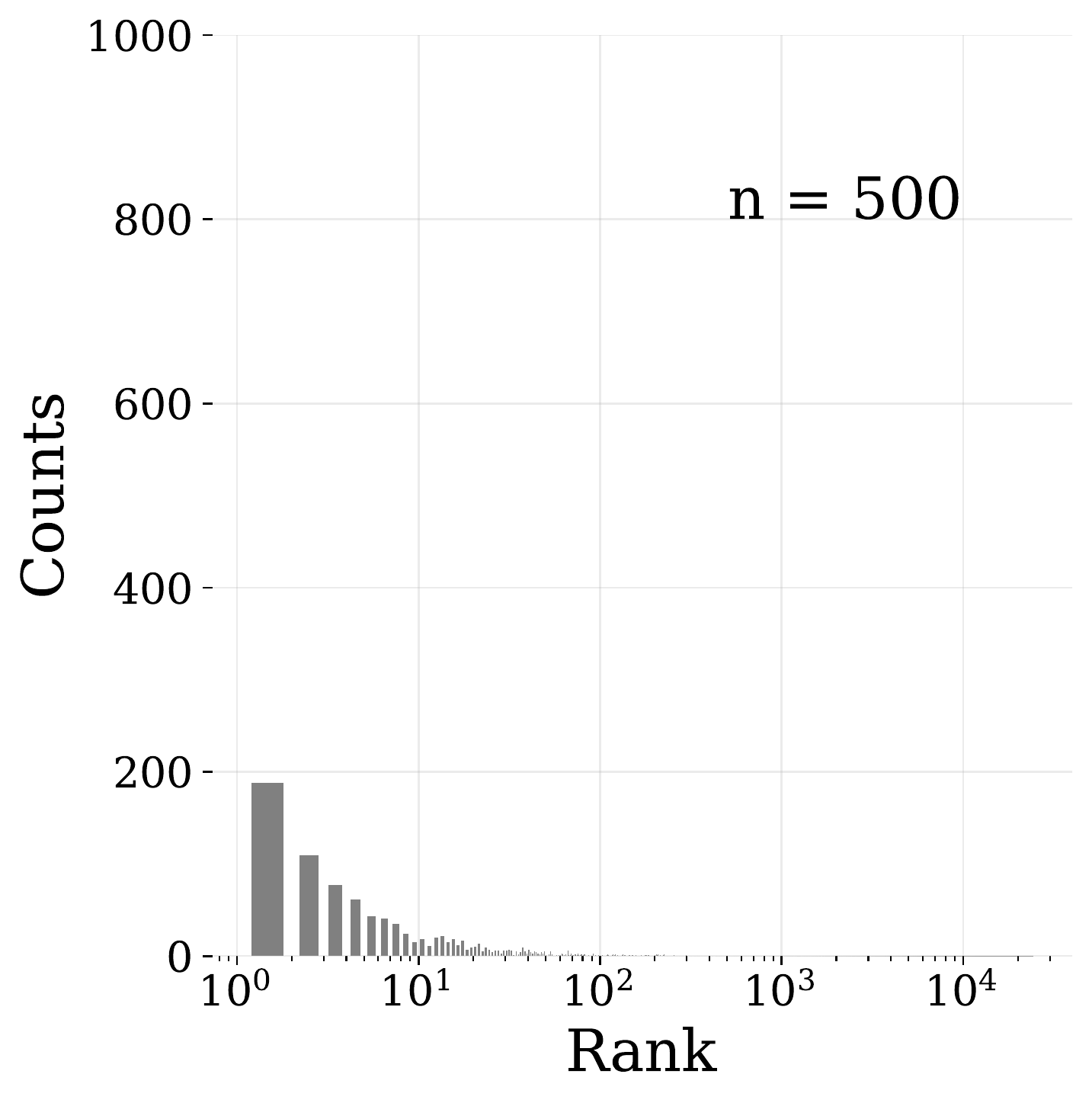}
\end{minipage}%
\begin{minipage}{0.3\textwidth}
\centering
\includegraphics[width=0.9\textwidth]{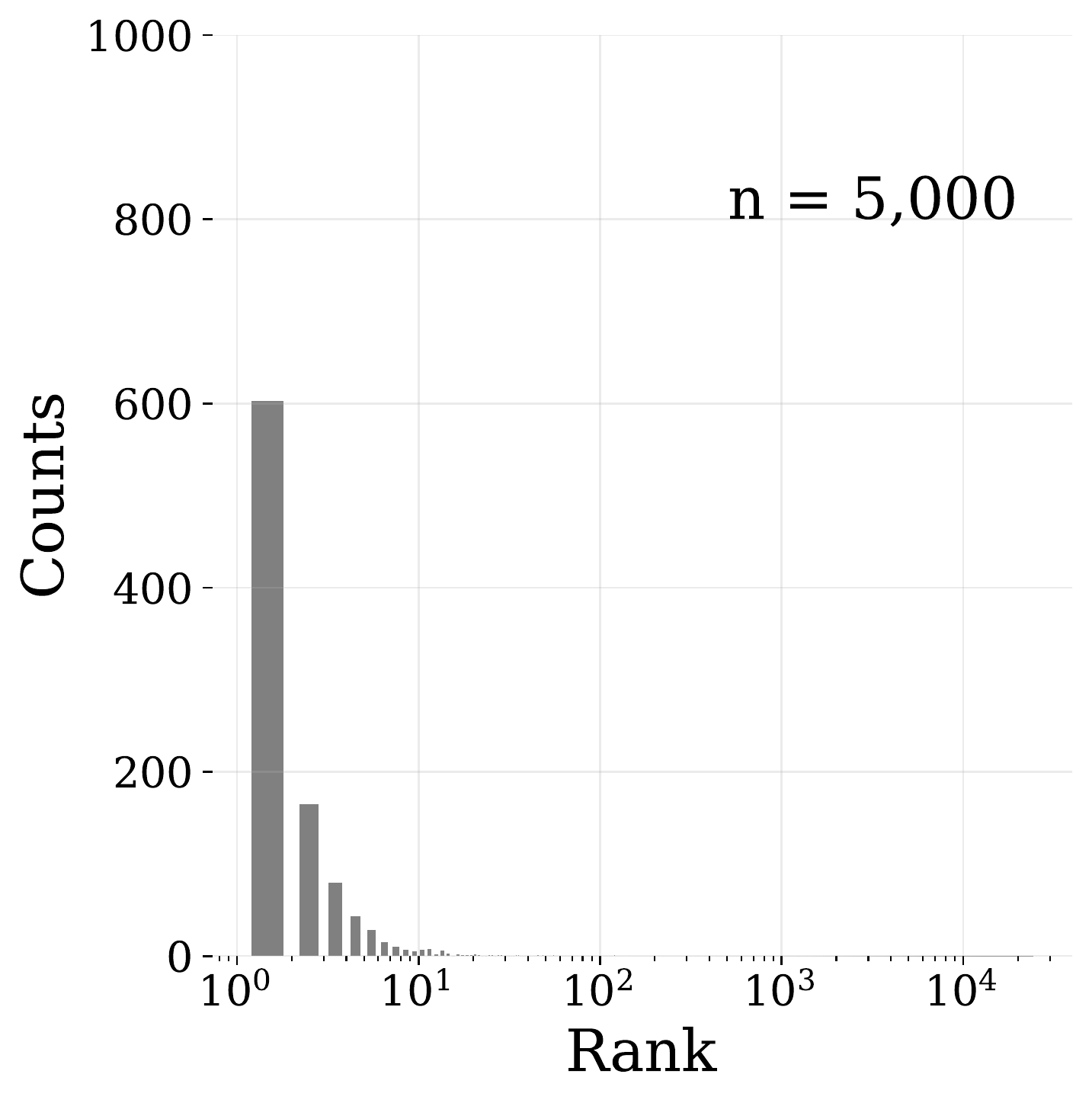}
\end{minipage}%
\begin{minipage}{0.3\textwidth}
\centering
\includegraphics[width=0.9\textwidth]{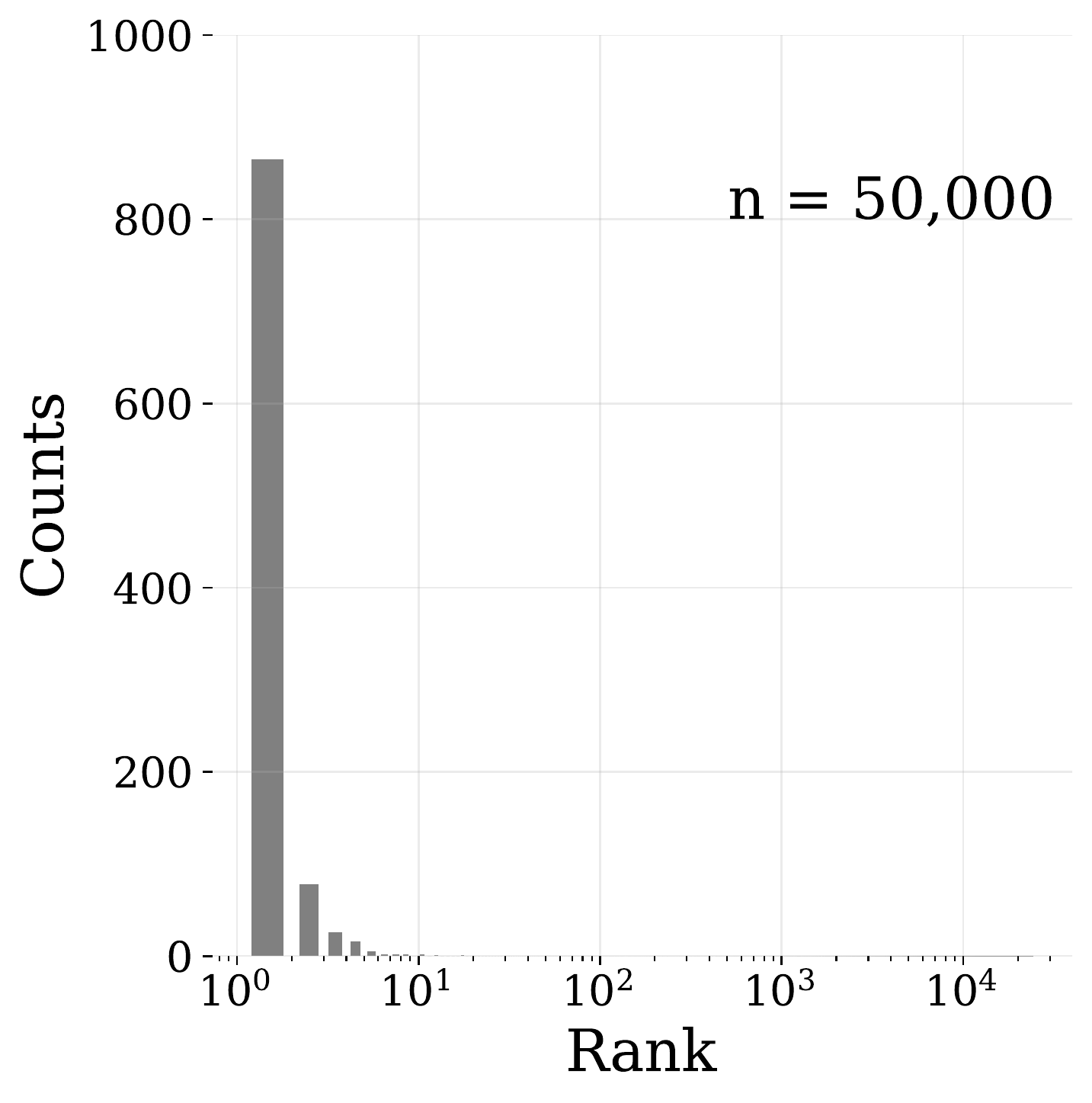}
\end{minipage}
\vskip 2mm
\begin{tabular}{c|c|c|c|c|c|c|c|c|c|c|c}
& $\textup{BIC}_\textup{ICF}$ & \multicolumn{2}{c|}{$\textup{BIC}_\textup{MF}$} & \multicolumn{2}{c|}{FCI} & \multicolumn{2}{c|}{FCI$_\text{max}$} & \multicolumn{2}{c|}{GFCI$_{1}$} & \multicolumn{2}{c}{GFCI$_{2}$} \\
\hline
repetitions & 100 & 100 & 1,000 & \multicolumn{2}{c|}{1,000} & \multicolumn{2}{c|}{1,000} & \multicolumn{2}{c|}{1,000} & \multicolumn{2}{c}{1,000} \\
\hline
$\alpha$-level & - & - & - & 0.01 & 0.001 & 0.01 & 0.001 & 0.01 & 0.001 & 0.01 & 0.001 \\
\hline
n = 500 & 0.24 & 0.24 & 0.188 & 0.026 & 0.009 & 0.072 & 0.031 & 0.035 & 0.021 & 0.022 & 0.019 \\
n = 5,000 & 0.61 & 0.59 & 0.603 & 0.311 & 0.243 & 0.354 & 0.319 & 0.297 & 0.261 & 0.238 & 0.216 \\
n = 50,000 & 0.9 & 0.91 & 0.865 & 0.665 & 0.617 & 0.687 & 0.671 & 0.641 & 0.622 & 0.595 & 0.579
\end{tabular}
\end{center}
\caption{An evaluation of the $\textup{BIC}_\textup{MF}$ score relative to the $\textup{BIC}_\textup{ICF}$ score and state-of-the-art algorithms on data generated using random directed MAGs with specified vertex and edge ranges.}
\label{fig:case_6_10}
\end{figure}

In Figure \ref{fig:case_6_10}, the $\textup{BIC}_\textup{MF}$ score performs nearly identically to the $\textup{BIC}_\textup{ICF}$ score and consistently ranks the correct MEC in the top 100 with the ranking converging to a point-mass in the first bin as $n \ra \infty$. The top ranked model according to the $\textup{BIC}_\textup{MF}$ score was correct more often than the models returned by other methods.

\begin{table}[H]
\centering
\begin{tabular}{c|c|c|c|c|c|c}
& \multicolumn{3}{c|}{$\textup{BIC}_\textup{ICF}$} & \multicolumn{3}{c}{$\textup{BIC}_\textup{MF}$} \\
\hline
sample size & n = 500 & n = 5,000 & n = 50,000 & n = 500 & n = 5,000 & n = 50,000 \\
\hline
Figure \ref{fig:case_5} & 74.6 (5.8) & 72.8 (5.7) & 71.9 (5.8) & 0.64 (0.02) & 0.63 (0.01) & 0.63 (0.01) \\
Figure \ref{fig:case_6} & 81.3 (8.5) & 80.9 (8.1) & 80.8 (8.4) & 0.64 (0.06) & 0.63 (0.06) & 0.63 (0.07) \\
Figure \ref{fig:case_0_5} & 77.7 (7.5) & 76.5 (7.8) & 75.5 (7.7) & 0.32 (0.1) & 0.29 (0.07) & 0.28 (0.03) \\
Figure \ref{fig:case_6_10} & 80.9 (10.5) & 80.3 (10.1) & 80.3 (10.4) & 0.26 (0.01) & 0.26 (0.01) & 0.26 (0.01)
\end{tabular}
\caption{The mean run times (stds in parentheses) for the simulations in seconds (100 reps).}
\label{tab:runtime_5}
\end{table}

Table \ref{tab:runtime_5} tabulates the time it took the $\textup{BIC}_\textup{MF}$ and $\textup{BIC}_\textup{ICF}$ scores to rank all 24,259 models. Both scores were given a sample covariance matrix as input. The time to calculate the sample covariance matrix generally took between 2 and 5 milliseconds. In general, we find that the $\textup{BIC}_\textup{MF}$ and $\textup{BIC}_\textup{ICF}$ scores have little to no difference in performance, while the $\textup{BIC}_\textup{MF}$ score takes a fraction of the time. Both scores perform well with small sample sizes and compare favorably to FCI and two of its variants on recovering correct independence models.

\section{Discussion} 
\label{sec:discussion}

In this paper, we introduce the \textit{m}-connecting imset which provides an alternative representation for ADMG independence models. Furthermore, we define the \textit{m}-connecting factorization criterion for ADMG models and prove its equivalence to the global Markov property. The \textit{m}-connecting imset and factorization criterion provide two new statistical tools for learning and inference with ADMG models. We demonstrate the usefulness of these tools by formulating and evaluating a consistent scoring criterion for Gaussian and multinomial ADMG models with a closed form. Unlike existing methods for scoring ADMG models, our proposed approach does not require the development of a distribution specific optimizer. Consequently, we do not solve for parameters in a non-convex space. 

In the past five years, there has been an influx of methods capable of learning ADMGs by optimizing a score, such as the methods of \cite{chen2021integer, nowzohour2017distributional, triantafillou2016score, tsirlis2018scoring}---these methods rely on the $\textup{BIC}_\textup{ICF}$ score. One could both improve the efficiency of these algorithms and allow them to be extended to more general distributions by replacing the $\textup{BIC}_\textup{ICF}$ score with the $\textup{BIC}_\textup{MF}$ score. Additionally, since our score is formulated using the \textit{m}-connecting imset, one could also explore a geometric approach for learning ADMGs, similar to existing DAG learning methods \cite{studeny2014learning, studeny2017towards, linusson2021greedy}.

\section*{Acknowledgements}

We thank Joe Ramsey, Kun Zhang, Clark Glymour, Robin Evans, and Zhongyi Hu for insightful discussion and comments on earlier versions of the draft. The research reported in this paper was supported by grant U54HG008540 awarded by the National Human Genome Research Institute through funds provided by the trans-NIH Big Data to Knowledge (BD2K) initiative, grant \#4100070287 from the Pennsylvania Department of Health (DOH), and grant IIS-1636786 from the National Science Foundation. Author BA also received support from training grant T15LM007059 from the National Library of Medicine. The content of this paper is solely the responsibility of the authors and does not necessarily represent the official views of these funding agencies.

\bibliographystyle{apalike}
\bibliography{main.bib}

\newpage

\appendix

\section{Proofs}

\label{app:proofs}

\noindent \textbf{Proposition \ref{prop:ordered_markov_helper}}
\textit{Let $\g = (V, E)$ be an ADMG with consistent order $\leq$. If $b \in V$ and $R = \pre{\g}{b}$ then\textup{:}}
\[
    \istate{b}{V \sm \cl{\g}{b}}{\mb{\g}{b}}{\g}.
\]

\begin{proof}
    By Proposition \ref{prop:admg_graphoid} $\mc I(\g)$ is a compositional graphoid, so we may apply the compositional graphoid axioms. Let $a \in V \sm \cl{\g}{b}$ and $\pi_{ab}$ be a path between $a$ and $b$. Traverse $\pi_{ab}$ from $a$ to $b$ until reaching a member of $\cl{\g}{b}$ denoting it $c$ and the vertex immediately preceding it on the transversal $d$. Notably, if $b = c$ or $c$ is a collider on $\pi_{ab}$, then $d \in \cl{\g}{b}$ which is false by construction. Accordingly, $c \in \mb{\g}{b}$ and is a non-collider on $\pi_{ab}$. It follows that $\pi_{ab}$ is not an \textit{m}-connecting path between $a$ and $b$ relative to $\mb{\g}{b}$. Since we picked $a$ and $\pi_{ab}$ arbitrarily, \istate{a}{b}{\mb{\g}{b}}{\g} for all $a \in V \sm \cl{\g}{b}$. By the composition graphoid axiom that \istate{b}{V \sm \cl{\g}{b}}{\mb{\g}{b}}{\g}.
\end{proof}

\vskip 5mm

\noindent \textbf{Theorem \ref{thm:imfact}}
\textit{Let $u$ be an imset and $P$ be a positive measure. If $u \in \mc S(V)$, then the following are equivalent\textup{:}}
\begin{enumerate}
    \item $\sum \limits_{S \in \mc P(V)} u(S) \log f_{S}(x) = 0 \s[18]$ for $\mu$-almost all $x \in \mc X$\textup{;}
    \item $\istate{A}{B}{C}{u} \s[18] \Ra \s[18] \istate{A}{B}{C}{P} \s[18]$ for all $\seq{A,B \mid C} \in \mc T(V)$.
\end{enumerate}

\begin{proof}
    Notably:
    \begin{align*} 
        \sum \limits_{S \in \mc P(V)} u(S) \log f_{S}(x) &= \sum \limits_{S \in \mc P(V)} \sum_{\seq{a,b \mid C} \in \mc T(V)} k_{a,b \mid C} u_{\seq{a,b \mid C}}(S) \log f_{S}(x) \\
        &= \sum_{\seq{a,b \mid C} \in \mc T(V)} k_{a,b \mid C} \sum \limits_{S \in \mc P(V)} u_{\seq{a,b \mid C}}(S) \log f_{S}(x) \\
        &= \sum_{\seq{a,b \mid C} \in \mc T(V)} k_{a,b \mid C} \log \frac{f_{abC}(x) f_{C}(x)}{f_{aC}(x) f_{bC}(x)}
    \end{align*}
    for $k_{\seq{a,b \mid C}} \in \mbb Q^+$.
    
    \vskip 5mm
    
    \noindent ($i \Ra ii$) If:
    \[
        \sum_{\seq{a,b \mid C} \in \mc T(V)} k_{a,b \mid C} \log \frac{f_{abC}(x) f_{C}(x)}{f_{aC}(x) f_{bC}(x)} = 0 \s[18] \text{for $\mu$-almost all} \; x \in \mc X,
    \]
    then:
    \[
        \mbb E_P \log \frac{f_{abC}(x) f_{C}(x)}{f_{aC}(x) f_{bC}(x)} = 0 \s[18] \text{for all} \; k_{a,b \mid C} \neq 0.
    \]
    Accordingly, $\istate{a}{b}{C}{u} \; \Ra \; \istate{a}{b}{C}{P}$ for all $\seq{a,b \mid C} \in \mc T(V)$ and by repeated application of the graphoid axoims $\istate{A}{B}{C}{u} \; \Ra \;  \istate{A}{B}{C}{P}$ for all $\seq{A,B \mid C} \in \mc T(V)$.
    
    \vskip 5mm
    
    \noindent ($i \La ii$) If $\istate{A}{B}{C}{u} \; \Ra \; \istate{A}{B}{C}{P}$ for all $\seq{A,B \mid C} \in \mc T(V)$, then:
    \[
        \log \frac{f_{abC}(x) f_{C}(x)}{f_{aC}(x) f_{bC}(x)} = 0 \s[18] \text{for all} \; k_{\seq{a,b \mid C}} \neq 0.
    \]
    Accordingly: 
    \[
        \sum_{\seq{a,b \mid C} \in \mc T(V)} k_{a,b \mid C} \log \frac{f_{abC}(x) f_{C}(x)}{f_{aC}(x) f_{bC}(x)} = 0.
    \]
\end{proof}

\vskip 5mm

\noindent \textbf{Proposition \ref{prop:char_imset}}
\textit{Let $\g[] = (V,E)$ be a DAG and $S \sube V$ $(|S| \geq 2)$. If $A = \an{\g}{S}$ and $B = \ba{\g}{S}$, then\textup{:}}
\begin{enumerate}
    \item $c_{\g}(S) \in \{ 0,1 \}$\textup{;}
    \item $c_{\g}(S) = 1 \s[18] \Lra \s[18] \text{there exists} \; b \in S \; \text{such that} \; S \sm b \sube \pa{\g}{b}$\textup{;}
    \item $c_{\g}(S) = 1 \s[18] \Lra \s[18] S \sube \co{\g[A]}{B}$. 
\end{enumerate}

\begin{proof}

    Properties (\textit{i}) and (\textit{ii}) were shown by \cite{hemmecke2012characteristic}. Accordingly, we show the following are equivalent:
    \begin{enumerate}
        \item $S \sube \co{\g[A]}{B}$\textup{;}
        \item $\text{there exists} \; b \in S \; \text{such that} \; S \sm b \sube \pa{\g}{b}$.
    \end{enumerate}

    \noindent ($i \Ra ii$) If $S \sube \co{\g[A]}{B}$, then let $\leq$ be a total order consistent with $\g$ and $b = \ceo{S}$. It follows that $S \sm b = \pa{\g}{a}$ since $\g$ is a DAG.
    
    \vskip 5mm
    
    \noindent ($i \La ii$) If $b \in S \; \text{such that} \; S \sm b \sube \pa{\g}{b}$, then by definition, $S \sube \co{\g[A]}{B}$
\end{proof}

\vskip 5mm

\noindent \textbf{Proposition \ref{prop:iir}}
\textit{If $P$ is a positive measure, then\textup{:}}
\begin{enumerate}
    \item $\displaystyle \phi_{A \mid B}(x) = \sum_{\substack{T \sube AB \\ B \sube T}} (-1)^{|AB \sm T|} \log f_{T}(x)$
    \item $\displaystyle \phi_{A, B \mid C}(x) = \log \frac{ f_{A B C}(x) f_{C}(x) }{ f_{A C}(x) f_{B C}(x) }$
    \item $u_{\seq{A,B \mid C}} = \mu_{\ms P} \s \delta_{A,B \mid C}$
    \item $\istate{A}{B}{C}{P} \s[18] \Lra \s[18] \phi_{A, B \mid C}(x) = 0 \s[18] \text{for $\mu$-almost all} \; x \in \mc X$
\end{enumerate}

\begin{proof}
\noindent (\textit{i}):

\begin{align*}
    \phi_{A \mid B}(x) &= \sum_{\substack{S \sube AB \\ A \sube S}} \phi_{S}(x) \\
    &= \sum_{\substack{S \sube AB \\ A \sube S}} \sum_{T \sube S} (-1)^{|S \sm T|} \log f_{T}(x) \\
    &= \sum_{T \sube AB} \log f_{T}(x) \sum_{\substack{S \sube AB \\ A,T \sube S}} (-1)^{|S \sm T|} \\
    &= \sum_{T \sube AB} \log f_{T}(x) \sum_{U \sube B \sm T} (-1)^{|A \sm T|+|U|} \\
    &= \sum_{\substack{T \sube AB \\ B \sube T}} (-1)^{|AB \sm T|} \log f_{T}(x)
\end{align*}
where
\[
    \sum_{U \sube B \sm T } (-1)^{|A \sm T|+|U|} =
    \begin{cases}
        \s[18] (-1)^{|AB \sm T|} \s[36] & B \sube T; \\
        \s[18] 0 \s[36] & B \not \sube T
    \end{cases}
\]
since whenever $B \not \sube T$ note that $|A \sm T|$ is constant and $B \sm T \neq \es$:
\[
    \sum_{S \sube B \sm T } (-1)^{|A \sm T|+|S|} = 0
\]
and whenever $B \sube T$ note that $B \sm T = \es$ and $A \sm T = AB \sm T$:
\[
    \sum_{S \sube B \sm T } (-1)^{|A \sm T|+|S|} = (-1)^{|AB \sm T|}.
\]

\vskip 5mm

\noindent (\textit{ii}):
\begin{align*}
\phi_{A, B \mid C}(x) &= \sum_{\substack{T \sube ABC \\ T \not \sube AC \\ T \not \sube BC}} \phi_{T}(x) \\
&= \sum_{T \sube ABC} \phi_{T}(x) + \sum_{T \sube C} \phi_{T}(x) - \sum_{T \sube AC} \phi_{T}(x) - \sum_{T \sube BC} \phi_{T}(x) \\
&= \log f_{A B C}(x) + \log f_{C}(x) - \log f_{A C}(x) - \log f_{B C}(x).
\end{align*}

\noindent (\textit{iii}): Note that $\sum_{T \sube A} \delta_T = \zeta_{\ms P} \s \delta_A$:
\begin{align*}
\delta_{A, B \mid C} &= \sum_{\substack{T \sube ABC \\ T \not \sube AC \\ T \not \sube BC}} \delta_{T} \\
&= \sum_{T \sube ABC} \delta_{T} + \sum_{T \sube C} \delta_{T} - \sum_{T \sube AC} \delta_{T} - \sum_{T \sube BC} \delta_{T} \\
&= \zeta_{\ms P} \s \delta_{A B C} + \zeta_{\ms P} \s \delta_{C} - \zeta_{\ms P} \s \delta_{A C} - \zeta_{\ms P} \s \delta_{B C} \\
&= \zeta_{\ms P} \s \left[ \delta_{ABC} + \delta_{C} - \delta_{AC} - \delta_{BC} \right] \\
&= \zeta_{\ms P} \s u_{\seq{A,B \mid C}}.
\end{align*}

\vskip 5mm

\noindent (\textit{iv}): This immediately follows from (\textit{ii}) and the definition of probabilistic conditional independence.

\end{proof}

\noindent \textbf{Theorem \ref{thm:mconn_fact_simp}}
\textit{Let $\g = (V, E)$ be an ADMG with consistent order $\leq$. If $P$ is a positive measure and $|M| \leq 5$ for all $M \in \mc M(\g)$, then the following are equivalent\textup{:}}
\begin{enumerate}
    \item $\log f(x) = \sum \limits_{M \in \mc M(\g)} \phi_M(x) \s[18]$ for $\mu$-almost all $x \in \mc X$\textup{;}
    \item \istate{A}{B}{C}{\g} $\s[18] \Ra \s[18]$ \istate{A}{B}{C}{P} $\s[18]$ for every $\seq{A,B \mid C} \in \mc T(V)$.
\end{enumerate}

\begin{proof}
    This result immediately follows from the enumeration of m-connecting factorizations for ADMG MECs in Appendix \ref{app:fac}.
\end{proof}

\vskip 5mm

\noindent \textbf{Proposition \ref{prop:struct_conds}}
\textit{Let $\g = (V, E)$ be an ADMG. If $|M| \leq 5$ for all $M \in \mc M(\g)$, then $\mu_{\ms P} \s n_{\g} \in \mc S(V)$\textup{.}}

\begin{proof}
    This result immediately follows from the enumeration of m-connecting factorizations for ADMG MECs in Appendix \ref{app:fac}.
\end{proof}

\vskip 5mm

Lemma \ref{lem:cm_equiv} shows that every collider-connecting set is a parameterizing set, and that every maximal collider-connecting set is a maximal parameterizing set.

\begin{lem}
    \label{lem:cm_equiv}
    If $\g$ is an ADMG, then\textup{:}
    \begin{enumerate}
        \item $\mc C(\g) \sube \mc M(\g)$\textup{;}
        \item $\ceil{\mc C(\g)} = \ceil{\mc M(\g)}$.
\end{enumerate}
\end{lem}

\begin{proof}

\noindent $(i)$: If $S \in \mc C(\g)$, then by the definitions of head and tail, there exist $H \in \mc H(\g)$ and $T = \ta{\g}{H}$ such that $S = HT$. Therefore, by the definition of parameterizing sets, $S \in \mc M(\g)$ and $\mc C(\g) \sube \mc M(\g)$.

\vskip 5mm

\noindent $(ii)$: Notably, if $\mc C(\g) \sube \mc M(\g)$, then $\ceil{\mc C(\g)} \sube \mc M(\g)$. If $S \in \ceil{\mc M(\g)}$, then by Proposition \ref{prop:char_imset}, $S \sube \co{\g[A]}{B}$ where $A = \an{\g}{S}$ and $B = \ba{\g}{S}$. Moreover, $S = \co{\g[A]}{B}$ by the maximality of $S$. Therefore, since $\mc C(\g[A]) \sube \mc C(\g)$, $S \in \mc C(\g)$ and $\ceil{\mc M(\g)} \sube \mc C(\g)$. If $\ceil{\mc C(\g)} \sube \mc M(\g)$ and $\ceil{\mc M(\g)} \sube \mc C(\g)$, then $\ceil{\mc M(\g)} = \ceil{\mc C(\g)}$.
\end{proof}

\vskip 5mm

\noindent Lemma \ref{lem:unique_barren} shows that vertices in the barren subset of all vertices have unique maximal parameterizing and collider-connecting sets.

\begin{lem}
    \label{lem:unique_barren}
    If $\g = (V, E)$ is an ADMG and $b \in \ba{\g}{V}$, then\textup{:}
    \begin{enumerate}
        \item $|\ceil{\{ C \in \mc C(\g) \; : \; b \in C \}}| = 1$\textup{;}
        \item $|\ceil{\{ M \in \mc M(\g) \; : \; b \in M \}}| = 1$.
    \end{enumerate}
\end{lem}

\begin{proof}
By Lemma \ref{lem:cm_equiv}, it is sufficient to prove (\textit{i}). Notably, by the definitions of collider-connecting set and collider-connecting vertices:
\[
    S = \co{\g}{S} \sube \co{\g}{b} \s[7] \text{for all} \s[7] S \in \{ C \in \mc C(\g) \; : \; b \in C \}
\]
and
\[
    \co{\g}{b} = \dis{\g}{b} \cup \pa{\g}{\dis{\g}{b}}.
\]
Accordingly,
\[
    \co{\g}{b} \in \{ C \in \mc C(\g) \; : \; b \in C \}.
\]
\end{proof}

\vskip 5mm

Corollary \ref{cor:induced_admg} shows that the induced subgraph of an ADMG over an ancestral set induces an independence subset over the shared variables. Furthermore, the induced subgraph has the same \textit{m}-connecting subsets of $A$.

\begin{cor}
\label{cor:induced_admg}
    If $\g$ is an ADMG and $A \in \mc A(\g)$, then\textup{:}
    \begin{enumerate}
        \item $\mc I(\g[A]) = \{ \seq{A, B \mid C} \in \mc T(A) \; : \; \seq{A, B \mid C} \in \mc I(\g) \}$\textup{;}
        \item $\mc M(\g[A]) = \{ M \in \mc P(A) \; : \; M \in \mc M(\g) \}$.
    \end{enumerate}
\end{cor}

\begin{proof}
    This immediately follows from Proposition \ref{prop:induce_margin}.
\end{proof}

\vskip 5mm

\noindent \textbf{Lemma \ref{lem:pairs_helper}}
\textit{If $\g = (V, E)$ is an ADMG and $b \in \ba{\g}{V}$ and $A \in \mc A(\g)$, then\textup{:}}
\begin{enumerate}
    \item $|\ceil{\{ M \in \mc M(\g) \; : \; b \in M \}}| = 1$\textup{;}
    \item $\mc M(\g[A]) = \{ M \in \mc P(A) \; : \; M \in \mc M(\g) \}$.
\end{enumerate}

\begin{proof}
    \noindent (\textit{i}) This immediately follows from Lemmas \ref{lem:cm_equiv} and \ref{lem:unique_barren}.
    
    \vskip 5mm
    
    \noindent (\textit{ii}) This immediately follows from Corollary \ref{cor:induced_admg}.
\end{proof}

\vskip 5mm

Lemma \ref{lem:routing} shows that \textit{m}-connecting sets may be characterized by the existence of inducing paths between a barren vertex and the other vertices in the set.

\begin{lem}
    \label{lem:routing}
    Let $\g = (V, E)$ be an ADMG, $b \in \ba{\g}{N}$, and $N \sube V$. If $b \in N$, then the following are equivalent\textup{:}
    \begin{enumerate}
    \item there exists $\seq{a,b \mid C} \in \mc T(V)$ such that $N \sm C = ab$ and \istate{a}{b}{C}{\g}\textup{;}
    \item $N \in \mc N(\g)$.
    \end{enumerate} 
\end{lem}

\begin{proof}
    \noindent ($i \Ra ii$) is a special case of Proposition \ref{prop:rasm}.
    
    \vskip 5mm
    
    \noindent ($i \La ii$) When $|N| = 2$, the proof is trivial. Let $|N| > 2$ and consider the contrapositive. Pick arbitrary $a,c \in N \sm b$ ($a \neq c$) and $N \sube S$, and let $C = S \sm abc$. By the antecedent of the contrapositive, \dstate{a}{b}{c \s[2] C}{\g} and \dstate{b}{c}{a \s[2] C}{\g} so there exist \textit{m}-connecting paths $\pi_{ab}$ and $\pi_{bc}$ between $a$ and $b$ relative to $c \s[2] C$ and between $b$ and $c$ relative to $a \s[2] C$, respectively. Construct a path $\pi_{ac}$ by traversing $\pi_{ab}$ from $a$ to $b$ until reaching some $d \in \pi_{bc}$ then traversing $\pi_{bc}$ from $d$ to $c$. Note the status of every non-endpoint vertex on $\pi_{ac}$. In particular, check if every non-collider on $\pi_{ac}$ is not member of $C$ and if every collider on $\pi_{ac}$ is in the ancestors of $S$. By construction, every non-endpoint vertex on $\pi_{ac}$ has the same status as on $\pi_{ab}$ or $\pi_{bc}$ except for $d$. Accordingly, we consider all possible scenarios for $d$. 
    
    \begin{itemize}
        \item If $d = b$, then $d$ is a collider on $\pi_{ac}$ because $b \in \ba{\g}{V}$ and $d \in \an{\g}{S}$ because $b \in S$.
        \item If $d \neq b$ is a non-collider on $\pi_{ac}$, then $d \not \in C$ because $d$ is a non-collider on $\pi_{ab}$ or $\pi_{bc}$.
        \item If $d \neq b$ is a collider on $\pi_{ac}$ and $d$ is a collider on $\pi_{ab}$ or $\pi_{bc}$, then $d \in \an{\g}{S}$. 
        \item If $d \neq b$ is a collider on $\pi_{ac}$ and a non-collider on $\pi_{ab}$ and $\pi_{bc}$, then $d$ is an ancestor of $a$, $c$, or a collider on $\pi_{ac}$, accordingly $d \in \an{\g}{S}$. 
    \end{itemize}
    Therefore, \dstate{a}{c}{S \sm ac}{\g} and the contrapositive is satisfied.
\end{proof}

\vskip 5mm

\noindent \textbf{Lemma \ref{lem:fact_helper}}
\textit{Let $\g = (V, E)$ be an ADMG and $N \sube V$. If $b \in \ba{\g}{N}$ and $M = \ceil{\{ M \in \mc M(\g) \; : \; b \in M \sube N \}}$, then\textup{:}}
\[
    \istate{b}{N \sm M}{M \sm b}{\g}.
\]
    
\begin{proof}
    By Proposition \ref{prop:admg_graphoid} the induced independence model $\mc I(\g)$ is a compositional graphoid. Accordingly, graphoid axioms (\textit{i - vi}) may be applied. If $N$ is a parameterizing set, then by maximally $M = N$. By the triviality graphoid axiom:
    \[
    \istate{b}{N \sm M}{M \sm b}{\g}.
    \]
    If $N$ is a constrained set, then $M \sub N$. Let $A = \an{\g}{N}$ and pick $a \in N \sm M$ and let $N_a = M \, \cup \, a$. By maximally $N_a$ is a constrained set. By Lemma \ref{lem:routing} and Corollary \ref{cor:induced_admg}, since $b \in \ba{\g[A]}{N}$:
    \[
    \istate{b}{a}{M \sm b}{\g} \quad \text{for all} \quad a \in N \sm M.
    \]
    By the composition graphoid axiom:
    \[
    \istate{b}{N \sm M}{M \sm b}{\g}.
    \]
\end{proof}

\vskip 5mm

\noindent \textbf{Lemma \ref{lem:ncon_ie}}
\textit{Let $\g = (V, E)$ be an ADMG and $b \in \ba{\g}{V}$. If $\ms M, \ms N = \textsc{Pairs}(\g, b)$ where $|\ms M| = n$ and $\ms R = \{ N \in \mc N(\g) \; : \; b \in N \}$, then\textup{:}}
\[
    \bigcup_{i=1}^n \ms N_{i,i} = \ms R.
\]

\begin{proof}
    Let $N \in \ms R$ and suppose by way of contradiction that $N \not \in \ms N_{i,i}$ for any $1 \leq i \leq n$. Note that $N \sube \ms N_i$ for some $1 \leq i \leq n$ since $\ceil{R} \sube \ms N$. Pick $i$ such that $N \sube \ms N_i$. If $N \not \sube \ms M_i$, then $N \in \ms N_{i,i}$. Otherwise, there exists $N \sube \ms N_j \sub \ms M_i$ for some $\ms N_j \in \ms N$. Repeat this logic until $N \in \ms N_{j,j}$ or $\ms M_j$ has no maximal non-\textit{m}-connecting subsets; the latter is a contradiction because $N$ is a subset and non-\textit{m}-connecting. Thus, there exists $1 \leq i \leq n$ such that $N \in \ms N_{i,i}$ and $\bigcup_{i=1}^n \ms N_{i,i} = \ms R$.
\end{proof}

\vskip 5mm

\begin{defn}[\textit{inclusion/exclusion for imsets} \cite{wildberger2003multisets}]
Let $V$ be a non-empty set of variables and $\ms S_1, \dots, \ms S_n \sube \mc P(V)$ be $n$ sets of sets. The concept of inclusion/exclusion is extended to imsets as follows:
\[
    \delta_{\, \bigcup_{i = 1}^{n} \ms S_i} = \sum_{\substack{J \sube \{1,\dots,n\} \\ J \neq \es}} (-1)^{|J|-1} \; \delta_{\, \bigcap_{j \in J} \ms S_j}.
\]
Several applications of De Morgan's laws gives an alternative form as follows:
\[
    \delta_{\, \bigcap_{i = 1}^{n} \ms S_i} = \sum_{\substack{J \sube \{1,\dots,n\} \\ J \neq \es}} (-1)^{|J|-1} \; \delta_{\, \bigcup_{j \in J} \ms S_j}.
\]
\end{defn}

\vskip 5mm

\noindent \textbf{Proposition \ref{prop:nconn_decomp}}
\textit{Let $\g = (V,E)$ be an ADMG with consistent order $\leq$. If $i_{\g}^{\leq}, e_{\g}^{\leq} = \textsc{NIE}(\g, \leq)$ then\textup{:}}
\begin{enumerate}
    \item $n_{\g} = i_{\g}^{\leq} - e_{\g}^{\leq}$\textup{;}
    \item $\mu_{\ms P} \s i_{\g}^{\leq}, \mu_{\ms P} \s e_{\g}^{\leq} \in \mc S(V)$.
\end{enumerate}

\begin{proof}

Let $b \in V$ and $R = \pre{\g}{b}$ and define $n_{\g,b}^{\leq} : \mc P(v) \ra \{ 0, 1 \}$:
\[
    n_{\g,b}^{\leq}(S) \eq 
    \begin{cases}
        \s[18] 1 & \s[36] \{ \seq{a,b \mid C} \in \mc I(\g[R]) \; : \; S \sm C = ab  \} \neq \es; \\
        \s[18] 0 & \s[36] \{ \seq{a,b \mid C} \in \mc I(\g[R]) \; : \; S \sm C = ab  \} = \es.
    \end{cases}
\]
For all $S \in \mc P(V)$, $n_{\g,b}^{\leq}(S) = 1$ if and only if $S \in \mc N(\g)$ and $b \in S \sube R$ by Lemma \ref{lem:unique_barren}.

\vskip 5mm

\noindent If $\ms M, \ms N = \textsc{Pairs}(\g, b)$ and $n = |\ms M|$, then:
\begin{allowdisplaybreaks}
\begin{align*}
n_{\g,b}^{\leq} &= \delta_{\, \bigcup_{i = 1}^n \ms N_{i,i}} && \text{(Lemma \ref{lem:ncon_ie})} \\
&= \delta_{\, \bigcup_{i = 1}^{n} \left[ \bigcup_{\substack{T \in \ms U_i \\ T \not \sube M_i}} \{T\} \right]} \\
&= \sum_{\substack{J \sube \{1,\dots,n\} \\ J \neq \es}} (-1)^{|J|-1} \; \delta_{\, \bigcap_{j \in J} \left[ \bigcup_{\substack{T \in \ms U_j \\ T \not \sube M_j}} \{T\} \right]}  && \text{(inclusion/exclusion)} \\
&= \sum_{\substack{J \sube \{1,\dots,n\} \\ J \neq \es}} (-1)^{|J|-1} \; \delta_{\, \bigcap_{j \in J} \left[ \bigcup_{\substack{T \in \ms U_J \\ T \not \sube M_j}} \{T\} \right]} && (\ms U_j \ra \ms U_J) \\
&= \sum_{\substack{J \sube \{1,\dots,n\} \\ J \neq \es}} (-1)^{|J|-1} \sum_{\substack{K \sube J \\ K \neq \es}} (-1)^{|K|-1} \; \delta_{\, \bigcup_{k \in K} \left[ \bigcup_{\substack{T \in \ms U_J \\ T \not \sube M_k}} \{T\} \right]} && \text{(inclusion/exclusion)} \\
&= \sum_{\substack{J \sube \{1,\dots,n\} \\ J \neq \es}} \sum_{\substack{K \sube J \\ K \neq \es}} (-1)^{|J \sm K|} \; \delta_{\, \bigcup_{k \in K} \left[ \bigcup_{\substack{T \in \ms U_J \\ T \not \sube M_k}} \{T\} \right]} \\
&= \sum_{\substack{J \sube \{1,\dots,n\} \\ J \neq \es}} \sum_{\substack{K \sube J \\ K \neq \es}} (-1)^{|J \sm K|} \; \delta_{\, \bigcup_{k \in K} \left[ \ms U_J \, \sm \, \bigcup_{\substack{T \in \ms U_J \\ T \sube M_k}} \{T\} \right]} && \text{(complement)} \\
&= \sum_{\substack{J \sube \{1,\dots,n\} \\ J \neq \es}} \sum_{\substack{K \sube J \\ K \neq \es}} (-1)^{|J \sm K|} \; \delta_{\, \ms U_J \, \sm \, \left[ \bigcap_{k \in K} \bigcup_{\substack{T \in \ms U_J \\ T \sube M_k}} \{T\} \right]} && \text{(De Morgan's law)} \\
&= \sum_{\substack{J \sube \{1,\dots,n\} \\ J \neq \es}} \sum_{\substack{K \sube J \\ K \neq \es}} (-1)^{|J \sm K|} \; \delta_{\, \ms U_J \, \sm \, \left[ \bigcap_{k \in K} \bigcup_{\substack{T \in \ms U_J \\ T \sube M_{J,K}}} \{T\} \right]} && (M_k \ra M_{J,K}) \\
&= \sum_{\substack{J \sube \{1,\dots,n\} \\ J \neq \es}} \sum_{\substack{K \sube J \\ K \neq \es}} (-1)^{|J \sm K|} \; \delta_{\, \ms U_J \sm \left[ \bigcup_{\substack{T \in \ms U_J \\ T \sube M_{J,K}}} \{T\} \right]} \\
&= \sum_{\substack{J \sube \{1,\dots,n\} \\ J \neq \es}} \sum_{\substack{K \sube J \\ K \neq \es}} (-1)^{|J \sm K|} \; \delta_{\, \bigcup_{\substack{T \in \ms U_J \\ T \not \sube M_{J,K}}} \{T\}} && \text{(complement)} \\
&= \sum_{\substack{J \sube \{1,\dots,n\} \\ J \neq \es}} \sum_{\substack{K \sube J \\ K \neq \es}} (-1)^{|J \sm K|} \; \delta_{\ms N_{J,K}}
\end{align*}
\end{allowdisplaybreaks}

\noindent Furthermore, define $i_{\g,b}^{\leq} : \mc P(V) \ra \mbb Z$ and $e_{\g,b}^{\leq} : \mc P(V) \ra \mbb Z$:

\[
    i_{\g,b}^{\leq} \eq \sum_{\substack{J \sube \{1,\dots,n\} \\ J \neq \es}} \sum_{\substack{K \sube J \\ K \neq \es}} \mathbbm{1}[(-1)^{|J \sm K|} > 0] \; \delta_{\ms N_{J,K}} \s[72] e_{\g,b}^{\leq} \eq \sum_{\substack{J \sube \{1,\dots,n\} \\ J \neq \es}} \sum_{\substack{K \sube J \\ K \neq \es}} \mathbbm{1}[(-1)^{|J \sm K|} < 0] \; \delta_{\ms N_{J,K}}.
\]

\noindent Accordingly, $n_{\g}$ decomposes as follows:
\begin{align*}
    n_{\g} &= \sum_{b \in V} n_{\g,b}^{\leq} \\
    &= \sum_{b \in V} \left[ i_{\g,b}^{\leq} - e_{\g,b}^{\leq} \right] \\
    &= \sum_{b \in V} i_{\g,b}^{\leq} - \sum_{b \in V} e_{\g,b}^{\leq} \\
    &= i_{\g}^{\leq} - e_{\g}^{\leq}
\end{align*}
Moreover, by Proposition \ref{prop:iir}, $\mu_{\ms P} \s i_{\g}^{\leq}, \mu_{\ms P} \s e_{\g}^{\leq} \in \mc S(V)$.
\end{proof}

\newpage

In order to prove our main results, we introduce the concept of a minimal latent set. 

\begin{defn}[\textit{minimal latent set}]
    Let $\g$ be an ADMG with consistent order $\leq$ and $A \in \mc A(\g)$. If $b = \ceo{A}$ and $R = \pre{\g}{A}$, then the \textit{minimal latent set} with respect to $\leq$ for $A$ is defined:
    \[
        \ml{\g}{A} \eq \sib{\g[R]}{\dis{\g[A]}{b}} \sm \dis{\g[A]}{b}.
    \]
\end{defn}

\noindent The set is minimal in the sense that making any additional members of $R$ latent will not change the Markov boundary:
\[
    \cl{\g[A]}{b} = \cl{\g[\,l \cup A]}{b} \s[18] \text{for all} \s[18] l \in R \sm \ml{\g}{A}.
\]
This concept was originally introduced by \cite{richardson2003markov} to construct ancestral sets that are maximal in the same way.

\begin{figure}[H]
    \begin{center}
        \begin{minipage}{.32\textwidth}
            \centering
            \vspace{12mm}
            \includegraphics[page=9]{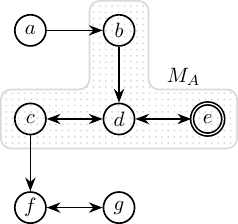}
            \[
                c,d \in A
            \]
        \end{minipage}%
        \begin{minipage}{.32\textwidth}
            \centering
            \vspace{12mm}
            \includegraphics[page=10]{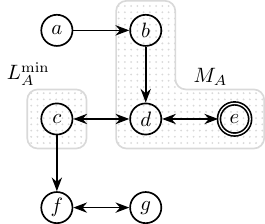}
            \[
                 d \in A, \; c \in L
            \]
        \end{minipage}
        \begin{minipage}{.32\textwidth}
            \centering
            \vspace{12mm}
            \includegraphics[page=11]{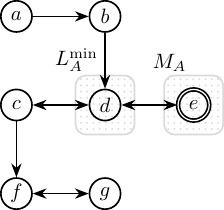}
            \[
            d \in L
            \]
        \end{minipage}
    \end{center}
    \caption{An illustration of the minimal latent sets.}
    \label{fig:lat_ex}
\end{figure}

Figure \ref{fig:lat_ex} illustrates the closure $M_A = \cl{\g[A]}{\ceo{A}}$ and corresponding minimal latent set $L_A^\text{min} = \ml{\g}{A}$ for an abstract $A \in \mc A(\g)$ in an ADMG $\g$ with abstract consistent order $\leq$ where $\ceo{A}$ is denoted with a doubled walled vertex. Lemma \ref{lem:pairs_ci} uses the concept of a minimal latent set to extract conditional independence statements from an ADMG.

\begin{lem}
\label{lem:pairs_ci}
Let $\g$ be an ADMG, $\leq$ be a total order consistent with $\g$, and $i_{\g}^{\leq}, e_{\g}^{\leq} = \textsc{NIE}(\g,\leq)$. If $A \in \mc A(\g)$, $b_A = \ceo{A}$, and $R_A = \pre{\g}{A}$, then\textup{:}
\begin{enumerate}
    \item \istate{b_A}{R \sm M_{R_A}}{M_{R_A} \sm b_A}{\mu_{\ms P} \s i_{\g}^{\leq}} where $M_{R_A} = \co{\g[R_A]}{b_A}$\textup{;}
    \item \istate{b_A}{N_A \sm M_A}{M_A \sm b_A}{\mu_{\ms P} \s i_{\g}^{\leq}} where $N_A = M_{R_A} \sm \ml{\g}{A}$ and $M_A = \co{\g[A]}{b}$.
\end{enumerate}
\end{lem}

\begin{proof}

By Proposition \ref{prop:struct_semigraphiod} the independence model induced by a structural imset is a semi-graphoid. Accordingly, we may apply the semi-graphoid axioms. Let $\ms M$ and $\ms N$ be the ordered lists constructed by Algorithm \ref{alg:pairs} and let $n$ be their cardinality. The structural imset $\mu_{\ms P} \s  \s i_{\g}^{\leq}$ is constructed as the sum over a set of semi-elementary imsets including the semi-elementary imsets defined as $\mu_{\ms P} \s \delta_{\ms N_{i,i}}$ for $1 \leq i \leq n$. Accordingly:
\[
\istate{b_A}{N_i \sm M_i}{M_i \sm b_A}{\mu_{\ms P} \s i_{\g}^{\leq}}.
\]

If $R_A \sm M_{R_A} = \es$, then by the by semi-graphoid axiom of triviality:
\[
\istate{b_A}{R_A \sm M_{R_A}}{M_{R_A} \sm b_A}{\mu_{\ms P} \s i_{\g}^{\leq}}.
\]

\noindent Accordingly, let $M_{R_A} \sube R_A$. By construction $N_1 = R_A$ and by Lemma \ref{lem:cm_equiv} $M_1 = M_{R_A}$, therefore:
\[
\istate{b_A}{R_A \sm M_{R_A}}{M_{R_A} \sm b_A}{\mu_{\ms P} \s i_{\g}^{\leq}}.
\]

If $N_A \sm M_A = \es$, then by semi-graphoid axiom of triviality:
\[
\istate{b}{N_A \sm M_A}{M_A \sm b}{\mu_{\ms P} \s i_{\g}^{\leq}}.
\]

\noindent Accordingly, let $N_A \sm M_A \neq \es$. Let $\ms N_A = \{ N \in \ms N \; : \; M_{R_A} = N \cup \ml{\g}{A} \}$ and note that $N_A = \bigcap_{N \in \ms N_A} N$. By construction $N_J = N_A$ and by Lemma \ref{lem:cm_equiv} $M_{J,J} = M_A$ for some $J \sube \{ 1, \dots, n\}$, therefore:
\[
\istate{b}{N_A \sm M_A}{M_A \sm b}{\mu_{\ms P} \s i_{\g}^{\leq}}.
\]
\end{proof}

\vskip 5mm

We now extend the ideas of Lemma \ref{lem:pairs_ci} to incorporate the conditional independence statements required by the ordered local Markov property. Let $\g = (V,E)$ be an ADMG with consistent order $\leq$. The following sets will be instrumental:
\begin{itemize}
\setlength{\itemindent}{45mm}
\item[{\makebox[45mm]{$A \in \mc A(\g)$\hfill}}] an ancestral set;
\item[{\makebox[45mm]{$R_A = \pre{\g}{A}$\hfill}}] the preceding vertices of $A$ with respect to $\leq$;
\item[{\makebox[45mm]{$b_A = \ceo{A}$\hfill}}] the last vertex in $A$ with respect to $\leq$;
\item[{\makebox[45mm]{$M_A = \co{\g[A]}{b_A}$\hfill}}] the maximal parameterizing set with respect to $A$ and $b_A$;
\item[{\makebox[45mm]{$M_{R_A} = \co{\g[R_A]}{b_A}$\hfill}}] the maximal parameterizing set with respect to $R_A$ and $b_A$;
\item[{\makebox[45mm]{$L_A = R_A \sm A$\hfill}}] the latent set with respect to $A$;
\item[{\makebox[45mm]{$L_A^\text{min} = \ml{\g}{A}$\hfill}}] the minimal latent set with respect to $A$ and $\leq$;
\item[{\makebox[45mm]{$N_A = M_{R_A} \sm L_A^\text{min}$\hfill}}] the maximal constrained subset of $M_{R_A}$;
\item[{\makebox[45mm]{$B_A = b_A \, \cup \, L_A^\text{min}$\hfill}}] the independent set containing $b_A$;
\item[{\makebox[45mm]{$C_A = M_{R_A} \sm B_A$\hfill}}] the conditioning set;
\item[{\makebox[45mm]{$D_A = \de{\g[R_A]}{L_A^\text{min}} \sm L_A^\text{min}$\hfill}}] the set to be dropped;
\item[{\makebox[45mm]{$F_A = R_A \sm M_{R_A} D_A$\hfill}}] the independent set not containing $b_A$.
\end{itemize}

\noindent In order to help the reader grasp these sets and how they relate, Figure \ref{fig:proofposet} illustrates their Hasse diagram ordered by inclusion. Notably, $B_A$, $C_A$, $D_A$, and $F_A$ partition $R_A$ and our ultimate goal is to show that \istate{B_A}{F_A}{C_A}{u_{\ms P} \s i_{\g}^{\leq}}.

\begin{figure}[H]
\begin{center}
\includegraphics[page=42]{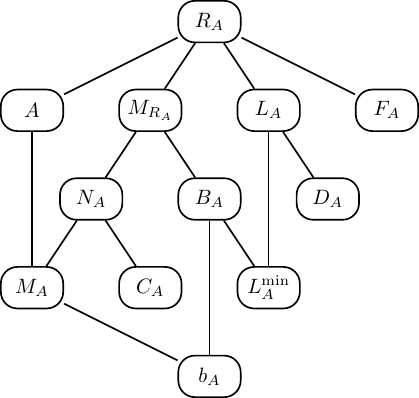}
\caption{The Hasse diagram for the poset over sets ordered by inclusion.}
\label{fig:proofposet}
\end{center}
\end{figure}

\noindent Figures \ref{fig:set_i} and \ref{fig:sets_ii} provide additional visual decomposition and partitioning of the sets used in the extension of Lemma \ref{lem:pairs_ci}.

\begin{figure}[H]
\begin{center}
\begin{minipage}{.32\textwidth}
\centering
\includegraphics[page=43]{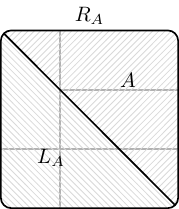} \\
\includegraphics[page=44]{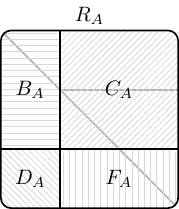}
\end{minipage}%
\begin{minipage}{.32\textwidth}
\centering
\includegraphics[page=46]{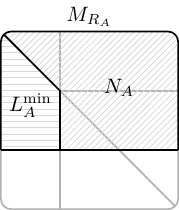} \\
\includegraphics[page=45]{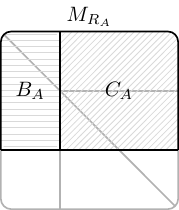}
\end{minipage}%
\begin{minipage}{.32\textwidth}
\centering
\includegraphics[page=47]{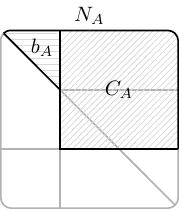} \\
\includegraphics[page=52]{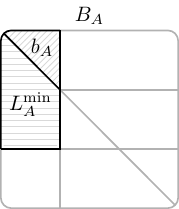}
\end{minipage}
\end{center}
\caption{An illustration of how various sets interact and partition each other.}
\label{fig:set_i}
\end{figure}

\begin{figure}[H]
\begin{center}
\begin{minipage}{.32\textwidth}
\centering
\includegraphics[page=53]{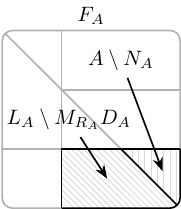}
\end{minipage}%
\begin{minipage}{.32\textwidth}
\centering
\includegraphics[page=54]{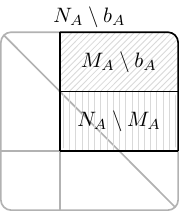}
\end{minipage}%
\begin{minipage}{.32\textwidth}
\centering
\includegraphics[page=55]{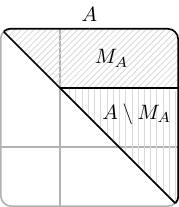}
\end{minipage}
\end{center}
\caption{An illustration of how various sets interact and partition each other.}
\label{fig:sets_ii}
\end{figure}

\noindent $C_A = N_A \sm b_A$ because $N_A \sube M_{R_A}$:
\begin{align*}
C_A &= M_{R_A} \sm B_A \\
&= M_{R_A} \sm (b_A \, \cup \, L_A^\text{min}) && (B = b_A \, \cup \, L_A^\text{min}) \\
&= (M_{R_A} \sm b_A) \, \cap \, (M_{R_A} \sm L_A^\text{min}) && (\text{distributive property}) \\
&= M_{R_A} \sm b_A \, \cap \, N_A && (N_A = M_{R_A} \sm L_A^\text{min}) \\
&= N_A \sm b_A && (M_{R_A} \, \cap \, N_A = N_A)
\end{align*}
$R_A = A L_A$ because $A \sube R_A$:
\begin{align*}
R_A &= (R_A \, \cap \, A) \, \cup \, L_A && \text{$(L_A = R_A \sm A)$}\\
&= A L_A && \text{$(R_A \, \cap \, A = A)$}
\end{align*}
$M_{R_A} = N_A L_A^\text{min}$ because $L_A^\text{min} = \ml{\g}{A} \sube \co{\g[R_A]}{b_A} = M_{R_A}$:
\begin{align*}
M_{R_A} &= N_A \, \cup \, (M_{R_A} \, \cap \, L_A^\text{min}) && \text{$(N_A = M_{R_A} \sm L_A^\text{min})$}\\
&= N_A L_A^\text{min} && \text{$(M_{R_A} \, \cap \, L_A^\text{min} = L_A^\text{min})$}
\end{align*}

In order to facilitate the forthcoming proof we define a few alternative relations. 
$F_A = (A \sm N_A) \, \cup \, (L_A \sm M_{R_A} D_A)$ because $L_A \, \cap \, A = \es$ and $L_A^\text{min} D_A \sube L_A$. Note that $D_A \sube L_A$ because $D_A = \de{\g}{L_{A}^\text{min}}$ and $A$ is an ancestral set. 
\begin{align*}
F_A &= R_A \sm M_{R_A} D_A \\
&= R_A \sm (M_{R_A} D_A \, \cup \, (L_A \, \cap \, A)) && \text{$(L_A \, \cap \, A = \es)$} \\
&= R_A \sm (M_{R_A} D_A L_A \, \cap \, M_{R_A} D_A A) && (\text{distributive property}) \\
&= A L_A \sm (N_A L_A^\text{min} D_A L_A \, \cap \, M_{R_A} D_A A) && (\text{change notation}) \\
&= A L_A \sm (N_A L_A \, \cap \, M_{R_A} D_A A) && \text{$(L_A^\text{min} D_A \sube L_A)$} \\
&= (A L_A \sm N_A L_A) \, \cup \, (AL_A \sm M_{R_A} D_A A) && (\text{distributive property}) \\
&= (A \sm N_A) \, \cup \, (L_A \sm M_{R_A} D_A) && (\text{simplify differences})
\end{align*}
$N_A \sm b_A = (N_A \sm M_A) \, \cup \, (M_A \sm b_A)$ because $b_A \in M_A \sube N_A$. Note that $M_A \sube N_A$ because $M_A \sube M_{R_A}$ and $M_A \, \cap \, L_A = \es$.
\begin{align*}
N_A \sm b_A &= N_A \sm (M_A \, \cap \, b_A) && (b_A = M_A \, \cap \, b_A) \\
&= N_A \sm ((M_A \, \cap \, b_A) \, \cup \, (M_A \, \cap \, (N_A \sm M_A))) && (M_A \, \cap \, (N_A \sm M_A) = \es)\\
&= N_A \sm (M_A \, \cap \, (b_A \, \cup \, (N_A \sm M_A))) && (\text{distributive property}) \\
&= (N_A \sm M_A) \, \cup \, (N_A \sm (b_A \, \cup \, (N_A \sm M_A))) && (\text{distributive property}) \\
&= (N_A \sm M_A) \, \cup \, ((N_A \sm b_A) \, \cap \, (N_A \sm (N_A \sm M_A))) && (\text{distributive property}) \\
&= (N_A \sm M_A) \, \cup \, ((N_A \sm b_A) \, \cap \, (N_A \, \cap \, M_A)) && (N_A \sm (N_A \sm M_A) = N_A \, \cap \, M_A) \\
&= (N_A \sm M_A) \, \cup \, ((N_A \sm b_A) \, \cap \, M_A) && (N_A \, \cap \, M_A = M_A) \\
&= (N_A \sm M_A) \, \cup \, (M_A \sm b_A) && ((N_A \sm b_A) \, \cap \, M_A = M_A \sm b_A)
\end{align*}
$A \sm M_A \sube (A \sm N_A) \, \cup \, (N_A \sm M_A)$ because $M_A \sube N_A$:
\begin{align*}
A \sm M_A &= A \sm (N_A \, \cap \, M_A) && (M_A = N_A \, \cap \, M_A) \\
&= A \sm ((N_A \, \cap \, M_A) \, \cup \, (N_A \, \cap \, (A \sm N_A))) && (N_A \, \cap \, (A \sm N_A) = \es)\\
&= A \sm (N_A \, \cap \, (M_A \, \cup \, (A \sm N_A))) && (\text{distributive property}) \\
&= (A \sm N_A) \, \cup \, (A \sm (M_A \, \cup \, (A \sm N_A))) && (\text{distributive property}) \\
&= (A \sm N_A) \, \cup \, ((A \sm M_A) \, \cap \, (A \sm (A \sm N_A))) && (\text{distributive property}) \\
&= (A \sm N_A) \, \cup \, ((A \sm M_A) \, \cap \, (A \, \cap \, N_A)) && (A \sm (A \sm N_A) = A \, \cap \, N_A) \\
&\sube (A \sm N_A) \, \cup \, ((A \sm M_A) \, \cap \, N_A) && (A \, \cap \, N_A \sube N_A) \\
&= (A \sm N_A) \, \cup \, (N_A \sm M_A) && ((A \sm M_A) \, \cap \, N_A = N_A \sm M_A)
\end{align*}

Algorithm \ref{alg:olmp} outlines a generalized process to extract conditional independence statements from an ADMG. The conditional independence statements are used to construct a structural imset whose induced independence model is a subset of the induced independence model of the graph and a subset of the independence model induced by the output of Algorithm \ref{alg:nie}. Furthermore, the conditional independence statements required by the ordered local Markov property are represented in the constructed imset. Every set Algorithm \ref{alg:olmp} is defined with respect to a given ancestral set. Accordingly, we drop the $A$ subset in the pseudocode to simplify notation. Applications of the symmetry semi-graphoid axiom are not noted in the algorithm.

\vskip 5mm

\begin{algorithm}[H]
    \caption{$\textsc{Ordered Local Markov Property OLMP}(\g, \leq, A)$}
    \label{alg:olmp}
    \KwIn{ADMG: $\g[] = (V,E)$, \, total order: $\leq$, \, ancestral set: $A$}
    \KwOut{imset: $u$}
    Let $b = \ceo{A}$, \, $R = \pre{\g}{b}$, \, $M_R = \co{\g[R]}{b}$, \, $L = \ml{\g}{A}$, \, $N = M_R \sm L$ \;
    Let $B = b \cup L$, \, $C = M_R \sm B$, \, $D = \de{\g[R]}{L} \sm L$, \, $F = R \sm M_R D$ \;
    Initialize imset $u : \mc P(V) \ra 0$ \;
    Let $i = 1$, \, $r_i = \flo{B}$, \, $R_i = \pre{\g}{r_i}$ \;
    \Repeat{$r_i = b$}{
        Let $B_i =  R_i \cap B$, \, $C_i = R_i \cap C$, \, $D_i = R_i \cap D$, \, $F_i = R_i \cap F$ \;
        \uIf{$r_i \in \dis{\g[R_i]}{b}$}{
            Let $M_i = \co{\g[R_i]}{r_i}$ \;
            $u = u + u_{\seq{r_i, R_i \sm M_i \mid M_i \sm r_i}}$ ; \hfill \tcp{Lemma \ref{lem:pairs_ci}}
            $u = u + u_{\seq{r_i, F_i \mid B_i C_i \sm r_i}}$ ; \hfill \tcp{decomposition and weak union}
            \If{$r_i \in C$}{
                $u = u + u_{\seq{r_i \cup B_i, F_i \mid C_{i-1}}}$ ; \hfill \tcp{contraction}
            }
        }
        \ElseIf{$r_i \in C F$}{
            Let $A' = R_i \sm D$\;
            $u = u + \textsc{OLMP}(\g, \leq, A')$ ; \hfill \tcp{recursive call}
            $u = u + u_{\seq{B_i, r_i \mid C_i F_i \sm r_i}}$ ; \hfill \tcp{decomposition and weak union}
            \If{$r_i \in C$}{
                $u = u + u_{\seq{B_i, r_i \cup F_i \mid C_{i-1}}}$ ; \hfill \tcp{contraction}
            }
        }
        $u = u + u_{\seq{B_i, F_i \mid C_i}}$ ; \hfill \tcp{weak union or contraction}
        \If{$r_i \neq b$}{
        Let $i = i + 1$, \, $r_{i} = \flo{R \sm R_{i-1}}$, \, $R_i = \pre{\g}{r_i}$ \;
        }
    }
    Let $M = \co{\g[A]}{b}$ \;
    $u = u + u_{\seq{b, A \sm N \mid N \sm b}}$ ; \hfill \tcp{decomposition}
    $u = u + u_{\seq{b, N \sm M \mid M \sm b}}$ ; \hfill \tcp{Lemma \ref{lem:pairs_ci}}
    $u = u + u_{\seq{b, A \sm M \mid M \sm b}}$ ; \hfill \tcp{contraction}
\end{algorithm}

\vskip 5mm

\begin{lem}
\label{lem:olmp_helper}
Let $\g = (V,E)$ be an ADMG with consistent order $\leq$. Let $A \in \mc A(\g)$, $b_A = \ceo{A}$, and $R_A = \pre{\g}{A}$. Furthermore, let $r \in R_A$ and $R = \pre{\g}{r}$. If $r \in \dis{\g[R_A]}{b_A}$, then\textup{:}
\[
\co{\g[R]}{r} \sube \co{\g[R_A]}{b_A}.
\]
If $r \not \in \dis{\g[R_A]}{b_A}$, then let $B_A = b_A \, \cup \, \ml{\g}{A}$\textup{:}
\[
\co{\g[R]}{r} \, \cap \, B_A = \es.
\]
\end{lem}

\begin{proof}
Note that $\g[R]$ is a subgraph of $\g[R_A]$ so any vertices and paths in $\g[R]$ are in $\g[R_A]$. Pick vertex $a \in \co{\g[R]}{r}$ and path $\pi_{ar}$ in \g[R] between $a$ and $r$ such that $\pi_{ar}$ is collider-connecting. Notably, $r = \ceo{R}$ so $\pi_{ar}$ must have an arrowhead directed into $r$.

For the first statement, we show that $a \in \co{\g[R_A]}{b_A}$. Since $r \in \dis{\g[R_A]}{b_A}$, there is a path $\pi_{br}$ in $\g[R_A]$ between $b_A$ and $r$ consisting entirely of bi-directed edges. Accordingly, the composition of $\pi_{ar}$ from $a$ to $r$ with $\pi_{br}$ from $r$ to $b_A$ is a collider-connecting path between $a$ and $b_A$ in $\g[R_A]$. It follows that $\co{\g[R]}{r} \sube \co{\g[R_A]}{b}$.

For the second statement, we show that if $a \in B_A$, then $r \in \dis{\g[R_A]}{b_A}$; this is the contrapositive statement. Since $a \in B_A$, there is a path $\pi_{ab}$ in $\g[R_A]$ between $a$ and $b_A$ consisting entirely of bi-directed edges. Accordingly, the composition of $\pi_{ab}$ from $b_A$ to $a$ with $\pi_{ar}$ from $a$ to $r$ is a collider-connecting path between $b_A$ and $r$ in $\g[R_A]$. It follows that $\co{\g[R]}{r} \, \cap \, B_A = \es$
\end{proof}

\noindent \textbf{Lemma \ref{lem:olmp}}
\textit{Let $\g = (V, E)$ be an ADMG and $\leq$ be a total order consistent with $\g[]$. If $i_{\g}^{\leq}, e_{\g}^{\leq} = \textsc{NIE}(\g,\leq)$, $A \in \mc A(\g)$, and $b_A = \ceo{A}$, then\textup{:}}
\[
    \istate{b}{A \sm \cl{\g[A]}{b}}{\mb{\g[A]}{b}}{\mu_{\ms P} \s i_{\g}^{\leq}}.
\]

\begin{proof}
Let $u_A = \textsc{OLMP}(\g, \leq, A)$ be the imset constructed by Algorithm \ref{alg:olmp} and $M_A = \co{\g[A]}{b_A}$. Since $u_A$ is defined as the sum of semi-elementary imsets, $u_A$ is a structural imset. Additionally, \istate{b_A}{A \sm M_A}{M_A \sm b_A}{u_A} by line 30. Consider the recursive call on line 16: Note that $A'$ is ancestral because it is defined as an ancestral set minus a set that contains all of its descendants. Let $b_{A'} = \ceo{A'}$ and $R_{A'} = \pre{\g}{b_{A'}}$ and note that $R_{A'} \sub R_A$. Accordingly, each time the algorithm is called recursively, the set of preceding variables is smaller. Since these sets are finite, Algorithm \ref{alg:olmp} is guaranteed to terminate.

We show that the conditional independence statement represented by semi-elementary imset added to $u_A$ are either represented in $\mu_{\ms P} \s i_{\g}^{\leq}$ by Lemma \ref{lem:pairs_ci} or implied by conditional independence statements that have been added the $u$ on an earlier line of the algorithm. Accordingly, $\mc I(u_A) \sube \mc I(\mu_{\ms P} \s i_{\g}^{\leq})$.

Let $R_A = \pre{\g}{b_A}$, $M_{R_A} = \co{\g[R_A]}{b_A}$, $L_A^\text{min} = \ml{\g}{A}$, $N_A = M_{R_A} \sm L_A^\text{min}$, and $L_A = R_A \sm A$. Let $B_A = b_A \, \cup \, L_A^\text{min}$, $C_A = M_{R_A} \sm B_A$, $D_A = \de{\g[R_A]}{L_A^\text{min}} \sm L_A^\text{min}$, $F_A = R_A \sm M_{R_A}D_A$. 

\vskip 5mm

We proceed by induction. For the base case, let $r_1^A = \flo{B_A}$, $R_1^A = \pre{\g}{r_1^A}$, and $M_1^A = \co{\g[R_1^A]}{r_1^A}$. Let $B_1^A = B_A \, \cap \, R_1^A$, $C_1^A = C_A \, \cap \, R_1^A$, $D_1^A = D_A \, \cap \, R_1^A$, and $F_1^A = F_A \, \cap \, R_1^A$ be the sets constrained to the set of variables $R_1^A$. Note that $r_1^A \in \dis{\g[R_i^A]}{b_A}$. By Lemma \ref{lem:pairs_ci},
\[
\istate{r_1^A}{R_1^A \sm M_1^A}{M_1^A \sm r_1^A}{\mu_{\ms P} \s i_{\g}^{\leq}}.
\]
Thus line 9 is satisfied. By changing notation,
\[
\istate{r_1^A}{B_1^A C_1^A D_1^A F_1^A \sm M_1^A}{M_1^A \sm r_1^A}{\mu_{\ms P} \s i_{\g}^{\leq}}.
\]
By Lemma \ref{lem:olmp_helper} $M_1^A \sube B_1^A C_1^A$. By the decomposition and weak union semi-graphoid axioms,
\[
\istate{r_1^A}{F_1^A}{B_1^A C_1^A \sm r_1^A}{\mu_{\ms P} \s i_{\g}^{\leq}}.
\]
Thus line 10 is satisfied. Noting that $B_1^A = r_1^A$,
\[
\istate{B_1^A}{F_1^A}{C_1^A}{\mu_{\ms P} \s i_{\g}^{\leq}}.
\]
Thus line 22 is satisfied.

\vskip 5mm

Let $r_i^A = \flo{R_A \sm R_{i-1}^A}$, $R_i^A = \pre{\g}{r_i^A}$, and $M_i^A = \co{\g[R_i^A]}{r_i^A}$. Let $B_i^A = B_A \, \cap \, R_i^A$, $C_i^A = C_A \, \cap \, R_i^A$, $D_i^A = D_A \, \cap \, R_i^A$, and $F_i^A = F_A \, \cap \, R_i^A$ be the sets constrained to the set of variables $R_i^A$. By the inductive hypothesis:
\[
\istate{B_{i-1}^A}{F_{i-1}^A}{C_{i-1}^A}{\mu_{\ms P} \s i_{\g}^{\leq}}.
\]

\vskip 5mm

If $r_i^A \in \dis{\g[R_i^A]}{b_A}$, then by Lemma \ref{lem:pairs_ci},
\[
\istate{r_i^A}{R_i^A \sm M_{R_i^A}}{M_{R_i^A} \sm r_i^A}{\mu_{\ms P} \s i_{\g}^{\leq}}.
\]
Thus line 9 is satisfied. By changing notation,
\[
\istate{r_i^A}{B_i^A C_i^A D_i^A F_i^A \sm M_i^A}{M_i^A \sm r_i^A}{\mu_{\ms P} \s i_{\g}^{\leq}}.
\]
By Lemma \ref{lem:olmp_helper} $M_i^A \sube B_i^A C_i^A$. By the decomposition and weak union semi-graphoid axioms,
\[
\istate{r_i^A}{F_i^A}{B_i^A C_i^A \sm r_i^A}{\mu_{\ms P} \s i_{\g}^{\leq}}.
\]
Thus line 10 is satisfied.

\vskip 5mm

If $r_i^A \in B_A$, then $B_i^A = r_i^A \, \cup \, B_{i-1}^A$, $C_i^A = C_{i-1}^A$, $D_i^A = D_{i-1}^A$, and $F_i^A = F_{i-1}^A$. By changing notation,
\[
\istate{r_i^A}{F_i^A}{B_{i-1}^A C_i^A}{\mu_{\ms P} \s i_{\g}^{\leq}}.
\]
By changing the notation of the inductive hypothesis,
\[
\istate{B_{i-1}^A}{F_i^A}{C_i^A}{\mu_{\ms P} \s i_{\g}^{\leq}}.
\]
By the symmetry and contraction semi-graphoid axioms,
\[
\istate{B_i^A}{F_i^A}{C_i^A}{\mu_{\ms P} \s i_{\g}^{\leq}}.
\]
Thus line 22 is satisfied.

\vskip 5mm

If $r_i^A \in C_A$, then $B_i^A = B_{i-1}^A$, $C_i^A = r_i^A \, \cup \, C_{i-1}^A$, $D_i^A = D_{i-1}^A$, and $F_i^A = F_{i-1}^A$. By changing notation,
\[
\istate{r_i^A}{F_i^A}{B_i^A C_{i-1}^A}{\mu_{\ms P} \s i_{\g}^{\leq}}.
\]
By changing the notation of the inductive hypothesis,
\[
\istate{B_i^A}{F_i^A}{C_{i-1}^A}{\mu_{\ms P} \s i_{\g}^{\leq}}.
\]
By the symmetry and contraction semi-graphoid axioms,
\[
\istate{r_i^A \, \cup \, B_i^A}{F_i^A}{C_{i-1}^A}{\mu_{\ms P} \s i_{\g}^{\leq}}.
\]
Thus line 12 satisfies the lemma. By the symmetry and weak union semi-graphoid axioms,
\[
\istate{B_i^A}{F_i^A}{C_i^A}{\mu_{\ms P} \s i_{\g}^{\leq}}.
\]
Thus line 22 is satisfied.

\vskip 5mm

Else if $r_i^A \in C_A F_A$, then let $A' = R_i^A \sm D_A$, $b_{A'} = \ceo{A'}$, and $M_{A'} = \co{\g[A']}{b_{A'}}$. Note that $A'$ is ancestral because it is defined as an ancestral set minus a set that contains all of its descendants. Note that $b_{A'} = r_i^A$. Since we show that all other lines are satisfied and Algorithm \ref{alg:olmp} terminates, lines 16 is satisfied. Accordingly,
\[
\istate{b_{A'}}{A' \sm M_{A'}}{M_{A'} \sm b_{A'}}{\mu_{\ms P} \s i_{\g}^{\leq}}.
\]
By changing notation,
\[
\istate{r_i^A}{B_i^A C_i^A F_i^A \sm M_{A'}}{M_{A'} \sm r_i^A}{\mu_{\ms P} \s i_{\g}^{\leq}}.
\]
By Lemma \ref{lem:olmp_helper} $M_{A'} \sube C_i^A F_i^A$. By the symmetry, decomposition, and weak union semi-graphoid axioms,
\[
\istate{B_i^A}{r_i^A}{ C_i^A F_i^A \sm r_i^A}{\mu_{\ms P} \s i_{\g}^{\leq}}.
\]
Thus line 17 is satisfied.

\vskip 5mm

If $r_i^A \in F_A$, then $B_i^A = B_{i-1}^A$, $C_i^A = C_{i-1}^A$, $D_i^A = D_{i-1}^A$, and $F_i^A = r_i^A \, \cup \, F_{i-1}^A$. By changing notation,
\[
\istate{B_i^A}{r_i^A}{ C_i^A F_{i-1}^A}{\mu_{\ms P} \s i_{\g}^{\leq}}.
\]
By changing the notation of the inductive hypothesis,
\[
\istate{B_i^A}{F_{i-1}^A}{C_i^A}{\mu_{\ms P} \s i_{\g}^{\leq}}.
\]
By the contraction semi-graphoid axiom,
\[
\istate{B_i^A}{F_i^A}{C_i^A}{\mu_{\ms P} \s i_{\g}^{\leq}}.
\]
Thus line 22 is satisfied.

\vskip 5mm

If $r_i^A \in C_A$, then $B_i^A = B_{i-1}^A$, $C_i^A = r_i^A \, \cup \, C_{i-1}^A$, $D_i^A = D_{i-1}^A$, and $F_i^A = F_{i-1}^A$. By changing notation,
\[
\istate{B_i^A}{r_i^A}{ C_{i-1}^A F_i^A}{\mu_{\ms P} \s i_{\g}^{\leq}}.
\]
By changing the notation of the inductive hypothesis,
\[
\istate{B_i^A}{F_i^A}{C_{i-1}^A}{\mu_{\ms P} \s i_{\g}^{\leq}}.
\]
By the contraction semi-graphoid axiom,
\[
\istate{B_i^A}{r_i^A \, \cup \, F_i^A}{C_{i-1}^A}{\mu_{\ms P} \s i_{\g}^{\leq}}.
\]
Thus line 19 is satisfied. By the weak union semi-graphoid axiom,
\[
\istate{B_i^A}{F_i^A}{C_i^A}{\mu_{\ms P} \s i_{\g}^{\leq}}.
\]
Thus line 22 is satisfied.

\vskip 5mm

If $r_i^A \in D_A$, then $B_i^A = B_{i-1}^A$, $C_i^A = C_{i-1}^A$ $D_i^A = r_i^A \, \cup \, D_{i-1}^A$, and $F_i^A = F_{i-1}^A$. By changing the notation of the inductive hypothesis,
\[
\istate{B_i^A}{F_i^A}{C_i^A}{\mu_{\ms P} \s i_{\g}^{\leq}}.
\]
Thus line 22 is satisfied.

\vskip 5mm

Accordingly,
\[
\istate{B_A}{F_A}{C_A}{\mu_{\ms P} \s i_{\g}^{\leq}}.
\]
Note that $B_A = b_A \, \cup \, L_A^\text{min}$, $C_A = N_A \sm b_A$, and $F_A \sube (A \sm N_A) \, \cup \, (L_A \sm M_{R_A} D_A)$. By changing notation and the decomposition semi-graphoid axiom,
\[
\istate{b \, \cup \, L_A^\text{min}}{(A \sm N_A) \, \cup \, (L_A \sm M_{R_A} D_A)}{N_A \sm b_A}{\mu_{\ms P} \s i_{\g}^{\leq}}.
\]
By the symmetry and decomposition semi-graphoid axioms,
\[
\istate{b_A}{A \sm N_A}{N_A \sm b_A}{\mu_{\ms P} \s i_{\g}^{\leq}}.
\]
Thus line 28 is satisfied. Note that $N_A \sm b_A = (N_A \sm M_A) \, \cup \, (M_A \sm b_A)$ because $b_A \in M_A \sube N_A$. By expanding notation,
\[
\istate{b_A}{A \sm N_A}{(N_A \sm M_A) \, \cup \, (M_A \sm b)}{\mu_{\ms P} \s i_{\g}^{\leq}}
\]
By Lemma \ref{lem:pairs_ci},
\[
\istate{b_A}{N_A \sm M_A}{M_A \sm b_A}{\mu_{\ms P} \s i_{\g}^{\leq}}.
\]
Thus line 29 is satisfied. By the contraction and decomposition semi-graphoid axioms,
\[
\istate{b_A}{(A \sm N_A) \, \cup \, (N_A \sm M_A)}{M_A \sm b_A}{\mu_{\ms P} \s i_{\g}^{\leq}}.
\]
Note that $A \sm M_A \sube (A \sm N_A) \, \cup \, (N_A \sm M_A)$ because $M_A \sube N_A$. By the decomposition semi-graphoid axiom,
\[
\istate{b_A}{A \sm M_A}{M_A \sm b_A}{\mu_{\ms P} \s i_{\g}^{\leq}}.
\]
Thus line 30 is satisfied.
\end{proof}

\newpage

\noindent \textbf{Proposition \ref{prop:in_ex_indmodels}}
\textit{Let $\g = (V, E)$ be an ADMG, $\leq$ be a total order consistent with $\g[]$. If $i_{\g}^{\leq}, e_{\g}^{\leq} = \textsc{NIE}(\g,\leq)$, then\textup{:}}
\begin{enumerate}
    \item $\mc I(\mu_{\ms P} \s i_{\g}^{\leq}) = \mc I(\g)$\textup{;}
    \item $\mc I(\mu_{\ms P} \s e_{\g}^{\leq}) \sube \mc I(\g)$.
\end{enumerate}

\begin{proof}
    By Lemma \ref{lem:fact_helper}, $\mc I(\mu_{\ms P} \s i_{\g}^{\leq}) \sube \mc I(\g)$ and $\mc I(\mu_{\ms P} \s e_{\g}^{\leq}) \sube \mc I(\g)$.
    By Theorem \ref{thm:mprops} and Lemma \ref{lem:olmp}, $\mc I(\g) \sube \mc I(\mu_{\ms P} \s i_{\g}^{\leq})$.
\end{proof}

\noindent \textbf{Corollary \ref{cor:fact}}
\textit{If $\g$ be an ADMG with consistent order $\leq$, then\textup{:}}
\begin{enumerate}
    \item $\mc I(\g) \sube \mc I(P) \s[18] \Lra \s[18] \phi_{\ms P}(x)^{\top} i_{\g}^{\leq} = 0 \s[18]$ for $\mu$-almost all $x \in \mc X$\textup{;}
    \item $\mc I(\g) \sube \mc I(P) \s[18] \Ra \s[18] \phi_{\ms P}(x)^{\top} e_{\g}^{\leq} = 0 \s[18]$ for $\mu$-almost all $x \in \mc X$.
\end{enumerate}

\begin{proof}
    This directly follows from Proposition \ref{prop:in_ex_indmodels} and Theorem \ref{thm:imfact}
\end{proof}

\noindent \textbf{Theorem \ref{thm:mconn_fact}}
\textit{Let $\g = (V, E)$ be an ADMG with consistent order $\leq$. If $P$ is a positive measure, then the following are equivalent\textup{:}}
\begin{enumerate}
    \item $\log f(x) = \sum \limits_{M \in \mc M(\g)} \phi_{M}(x) - \phi_{\ms P}(x)^{\top} e_{\g}^{\leq} \s[18]$ for $\mu$-almost all $x \in \mc X$\textup{;}
    \item \istate{A}{B}{C}{\g} $\s[18] \Ra \s[18]$ \istate{A}{B}{C}{P} $\s[18]$ for all $\seq{A,B \mid C} \in \mc T(V)$.
\end{enumerate}

\begin{proof}

    In vector notation:
    \[
        \log f(x) = \phi_{\ms P}(x)^{\top} (m_{\g} - e_{\g}^{\leq}).
    \]
    
    By M{\"o}bius inversion:
    \begin{align}
        \label{eq:1}
        \log f(x) &= \phi_{\ms P}(x)^{\top} \delta_{\mc P(V)} \nonumber \\
        &= \phi_{\ms P}(x)^{\top} (m_{\g} + n_{\g}) \nonumber \\ 
        &= \phi_{\ms P}(x)^{\top} (m_{\g} + i_{\g}^{\leq} - e_{\g}^{\leq}) \nonumber \\ 
        &= \phi_{\ms P}(x)^{\top} m_{\g} - \phi_{\ms P}(x)^{\top} e_{\g}^{\leq} + \phi_{\ms P}(x)^{\top} i_{\g}^{\leq} \nonumber \\ 
        &= \sum \limits_{M \in \mc M(\g)} \phi_{M}(x) - \phi_{\ms P}(x)^{\top} e_{\g}^{\leq} + \phi_{\ms P}(x)^{\top} i_{\g}^{\leq}.
    \end{align}
    
    \noindent $(i \Ra ii)$:
    
    By setting the RHS of the antecedent and Equation \ref{eq:1} equal to each other:
    \[
        \phi_{\ms P}(x)^{\top} i_{\g}^{\leq} = 0.
    \]

    Applying the M{\"o}bius inversion:
    \begin{align*}
        \phi_{\ms P}(x)^{\top} i_{\g}^{\leq} &= \phi_{\ms P}(x)^{\top} \zeta_{\ms P} \mu_{\ms P} \s i_{\g}^{\leq} \\
        &= [\zeta_{\ms P}^{\top} \phi_{\ms P}(x)]^{\top} \mu_{\ms P} \s i_{\g}^{\leq} \\
        &= h_{\ms P}(x)^{\top} \mu_{\ms P} \s i_{\g}^{\leq}.
    \end{align*}
    
    Accordingly:
    \[
        h_{\ms P}(x)^{\top} \mu_{\ms P} \s i_{\g}^{\leq} = 0.
    \]
    
    By Lemma \ref{lem:olmp}, if $A \in \mc A(\g)$ and $b = \ceo{A}$, then:
    \[
        \istate{b}{A \sm \cl{\g[A]}{b}}{\mb{\g[A]}{b}}{\mu_{\ms P} \s i_{\g}^{\leq}}.
    \]

    By Theorem \ref{thm:imfact}:
    \[
        \istate{A}{B}{C}{\mu_{\ms P} \s i_{\g}^{\leq}} \s[18] \Ra \s[18] \istate{A}{B}{C}{P} \s[18] \text{for every} \; \seq{A,B \mid C} \in \mc T(V).
    \]
    
    Accordingly, if $A \in \mc A(\g)$ and $b = \ceo{A}$, then:
    \[
        \istate{b}{A \sm \cl{\g[A]}{b}}{\mb{\g[A]}{b}}{P}.
    \]
    
    By Theorem \ref{thm:mprops}:
    \[
        \istate{A}{B}{C}{\g} \s[18] \Ra \s[18] \istate{A}{B}{C}{P} \s[18] \text{for every} \; \seq{A,B \mid C} \in \mc T(V).
    \]
    
    \noindent $(ii \La i)$: 
    
    By Equation \ref{eq:1}:
    \[
        \log f(x) = \sum \limits_{M \in \mc M(\g)} \phi_{M}(x) - \phi_{\ms P}(x)^{\top} e_{\g}^{\leq} + \phi_{\ms P}(x)^{\top} i_{\g}^{\leq}.
    \]
    
    By Lemma \ref{cor:fact}:
    \[
        \phi_{\ms P}(x)^{\top} i_{\g}^{\leq} = 0.
    \]
    
    Accordingly:
    \[
        \log f(x) = \sum \limits_{M \in \mc M(\g)} \phi_{M}(x) - \phi_{\ms P}(x)^{\top} e_{\g}^{\leq}.
    \]
\end{proof}

\noindent \textbf{Lemma \ref{lem:consistant_score}}
\textit{Let $\g = (V, E)$ and $\g' = (V, E')$ be ADMGs and $\leq$ and $\leq'$ be total orders consistent with $\g$ and $\g'$, respectively. Furthermore, let $X$ be a collection of random variables indexed by $V$ with positive measure $P_\theta$ that admits density $f(x \mid \theta)$ with respect to dominating $\sigma$-finite product measure $\mu$.}

If $x^1, \dots, x^n \overset{iid}{\sim} f(x \mid \theta)$ and $\theta \in \Theta_{\g} \sm \Theta_{\g'}$, then:
\[
\lim_{n \ra \infty} \, \frac{1}{n} \left | \textup{BIC}_\textup{MF}(\g, \leq, x^1, \dots, x^n) - \textup{BIC}_\textup{MF}(\g', \leq', x^1, \dots, x^n) \right | = \mbb E_{P_\theta} \left[ \phi_{\ms P}(x)^{\top} i_{\g'}^{\leq} \right].
\]

If $x^1, \dots, x^n \overset{iid}{\sim} f(x \mid \theta)$ and $\theta \in \Theta_{\g} \cap \, \Theta_{\g'}$ with $|\Theta_{\g}| < |\Theta_{\g'}|$, then:
\[
\lim_{n \ra \infty} \frac{1}{\log n} \left | \textup{BIC}_\textup{MF}(\g, \leq, x^1, \dots, x^n) - \textup{BIC}_\textup{MF}(\g', \leq', x^1, \dots, x^n) \right | = \frac{|\Theta_{\g'}| - |\Theta_{\g}|}{2}.
\]

\begin{proof}
Let $\mc D = \dom{\g}$ and $\mc D' = \dom{\g'}$. If $x^1, \dots, x^n \overset{iid}{\sim} f(x \mid \theta)$ and $\theta \in \Theta_{\g} \sm \Theta_{\g'}$, then:

\begin{align*}
\lim_{n \ra \infty} & \frac{1}{n} \left | \textup{BIC}_\textup{MF}(\g', \leq' x^1, \dots, x^n) - \textup{BIC}_\textup{MF}(\g, \leq, x^1, \dots, x^n) \right | \\
&= \lim_{n \ra \infty} \frac{1}{n} \hat{\ell}_{\g}^{\leq}(\hat{\theta}_{\mc D,n}^\textup{mle} \mid x^1, \dots, x^n) - \lim_{n \ra \infty} \frac{1}{n} \hat{\ell}_{\g'}^{\leq'}(\hat{\theta}_{\mc D',n}^\textup{mle} \mid x^1, \dots, x^n) - \frac{|\Theta_{\g'}| - |\Theta_{\g}|}{2} \lim_{n \ra \infty} \frac{\log(n)}{n} \\
&= \mbb E_{P_\theta} \left[ \phi_{\ms P}(x)^{\top} \delta_{\mc P(V)} \right] - \mbb E_{P_\theta} \left[ \phi_{\ms P}(x)^{\top} \delta_{\mc M(\g')} \right] + \mbb E_{P_\theta} \left[ \phi_{\ms P}(x)^{\top} e_{\g'}^{\leq} \right] \\
&= \mbb E_{P_\theta} \left[ \phi_{\ms P}(x)^{\top} i_{\g'}^{\leq} \right].
\end{align*}

If $x^1, \dots, x^n \overset{iid}{\sim} f(x \mid \theta)$ and $\theta \in \Theta_{\g} \cap \, \Theta_{\g'}$ with $|\Theta_{\g}| < |\Theta_{\g'}|$, then:
\begin{align*}
\lim_{n \ra \infty} \, & \frac{1}{\log n} \left | \textup{BIC}_\textup{MF}(\g, \leq, x^1, \dots, x^n) - \textup{BIC}_\textup{MF}(\g', \leq', x^1, \dots, x^n) \right | \\
&= \lim_{n \ra \infty} \frac{1}{\log n} \left [ \hat{\ell}_{\g}^{\leq}(\hat{\theta}_{\mc D,n}^\textup{mle} \mid x^1, \dots, x^n) - \hat{\ell}_{\g'}^{\leq'}(\hat{\theta}_{\mc D',n}^\textup{mle} \mid x^1, \dots, x^n) \right ] - \frac{|\Theta_{\g'}| - |\Theta_{\g}|}{2} \\
&= \frac{|\Theta_{\g'}| - |\Theta_{\g}|}{2}.
\end{align*}
\end{proof}

\noindent \textbf{Proposition \ref{prop:consistant_score}}
\textit{Let $\g = (V, E)$ and $\g' = (V, E')$ be ADMGs and $\leq$ and $\leq'$ be total orders consistent with $\g$ and $\g'$, respectively. Furthermore, let $X$ be a collection of random variables indexed by $V$ with positive measure $P_{\theta}$ that admits density $f(x \mid \theta)$ with respect to dominating $\sigma$-finite product measure $\mu$. If $x^1, \dots, x^n \overset{iid}{\sim} f(x \mid \theta)$ and $\theta \in \Theta_{\g} \sm \Theta_{\g'}$ or $\theta \in \Theta_{\g} \cap \, \Theta_{\g'}$ where $|\Theta_{\g}| < |\Theta_{\g'}|$, then\textup{:}}
\[
\lim_{n \ra \infty} \textup{Pr}(\textup{BIC}_\textup{MF}(\g', \leq', x^1, \dots, x^n) < \textup{BIC}_\textup{MF}(\g, \leq, x^1, \dots, x^n)) = 1.
\]

\begin{proof}
    This immediately follows from Proposition \ref{prop:bic_consist} and Lemma \ref{lem:consistant_score}.
\end{proof}

\section{Extras}

\label{app:extras}

\subsection{Necessity of the Adjustment Term}

\label{app:ng_not_in_s}

In this section, we give an example for which our current strategy for proving that the \textit{m}-connecting factorization (without adjustment) is equivalent to the global Markov property fails by showing a case where $\mu_{\ms P} \s n_{\g} \not \in \mc S(V)$. This does not necessarily mean that the adjustment term is necessary, but does mean that if it is not, then we need a new proof strategy.

\begin{figure}[H]
    \centering
    \begin{minipage}{.45\textwidth}
        \centering
        \includegraphics[page=56]{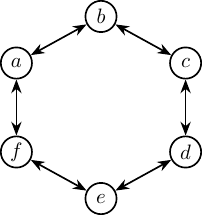}
    \end{minipage} %
    \begin{minipage}{.45\textwidth}
        \centering
        \begin{tabular}{c||c||c|c|c|c|c|c}
         & $|S|$ & 6 & 5 & 4 & 3 & 2 & 1 \\
         \hline
        $n_{\g}$ & $-$ & 0 & 0 & 9 & 14 & 9 & 0 \\
        $\delta_{v,w \mid S}$ & 2 & 0 & 0 & 1 & 2 & 1 & 0 \\
        $\delta_{v,w \mid S}$ & 1 & 0 & 0 & 0 & 1 & 1 & 0 \\
        $\delta_{v,w \mid S}$ & 0 & 0 & 0 & 0 & 0 & 1 & 0
    \end{tabular}
    \end{minipage}
    \caption{Imset values marginalized over set size.}
    \label{fig:not_struct}
\end{figure}

\noindent Let $V = {a,b,c,d,e,f}$. In Figure \ref{fig:not_struct}, we see that:
\begin{itemize}
    \item $\sum\limits_{\substack{T \in \mc P(V) \\ |T| = 6}} n_{\g}(T) = 0$; 
    \item $\sum\limits_{\substack{T \in \mc P(V) \\ |T| = 5}} n_{\g}(T) = 0$; 
    \item $\sum\limits_{\substack{T \in \mc P(V) \\ |T| = 4}} n_{\g}(T) = 9$; 
    \item $\sum\limits_{\substack{T \in \mc P(V) \\ |T| = 3}} n_{\g}(T) = 14$; 
    \item $\sum\limits_{\substack{T \in \mc P(V) \\ |T| = 2}} n_{\g}(T) = 9$;
    \item $\sum\limits_{\substack{T \in \mc P(V) \\ |T| = 1}} n_{\g}(T) = 0$.
\end{itemize}
Clearly, there is no way to pick coefficients $k_{v,w \mid S} \in \mbb Q^+$ for all $\seq{v,w \mid S} \in \mc T(V)$ such that:
\[
    n_{\g} = \sum\limits_{\seq{v,w \mid S} \in \mc T(V)} k_{v,w \mid S} \, \delta_{v,w \mid S}.
\]
Accordingly, $\mu_{\ms P} \s n_{\g} \not \in \mc S(V)$.

\subsection{Alternative In/Exclusion Algorithm}

\label{app:alt_nie}

In this section, we present an alternative version of Algorithm \ref{alg:nie} which accomplishes the same goal as Algorithm \ref{alg:nie}, but does so with less redundancy. That is, $i_{\g}^{\leq}$ and $e_{\g}^{\leq}$ are constructed such that no set identifier for any conditional term is added to both structural imsets. Importantly, we have not verified that removing redundant terms does not change the induced independence model.

\vskip 5mm

\begin{algorithm}[H]
\caption{${\textsc{Non-m-connecting imset via In/Ex-clusion NIE}(\g, \leq)}$}
\label{alg:nie_alt}
\KwIn{ADMG: $\g[] = (V,E)$, total order: $\leq$}
\KwOut{imsets: $i_{\g}^{\leq}$, $e_{\g}^{\leq}$}
Initialize imsets $i_{\g}^{\leq} = 0$ and $e_{\g}^{\leq} = 0$ \;
Let $A = V$ \;
\Repeat{$A = \es$}{
    Initialize lists $\ms I = \seq{}$ and $\ms E = \seq{}$ \;
    Let $b = \ceo{A}$ and $\ms M, \ms N = \textsc{Pairs}(\g[A],b)$ \;
    \ForEach{$J \sube \{1, \dots, |\ms M| \}$}{
        \ForEach{$K \sube J$ \textup{where} $K \neq \es$}{
            \uIf{$|J \sm K|$ \textup{is even and} $\ms N_{J,K} \not \in \ms E$}{
                Add an instance of $\ms N_{J,K}$ to $\ms I$ \;
            }
            \uElseIf{$|J \sm K|$ \textup{is even and} $\ms N_{J,K} \in \ms E$}{
                Remove an instance of $\ms N_{J,K}$ from $\ms E$ \;
            }\uElseIf{$|J \sm K|$ \textup{is odd and} $\ms N_{J,K} \not \in \ms I$}{
                Add an instance of $\ms N_{J,K}$ to $\ms E$ \;
            }
            \ElseIf{$|J \sm K|$ \textup{is odd and} $\ms N_{J,K} \in \ms I$}{
                Remove an instance of $\ms N_{J,K}$ from $\ms I$ \;
            }
        }
    }
    \ForEach{$\ms S \in \ms I$}{
        $i_{\g}^{\leq} = i_{\g}^{\leq} + \delta_{\ms S}$ \;
    }
    \ForEach{$\ms S \in \ms E$}{
        $e_{\g}^{\leq} = e_{\g}^{\leq} + \delta_{\ms S}$ \;
    }
    Remove $b$ from $A$ \;   
}
\end{algorithm}

\vskip 5mm

\subsection{Marginal Likelihood}

\label{app:bic}

In what follows, we reason about the asymptotic properties of the $\textup{BIC}_\textup{MF}$ score and its relation to the marginal likelihood. \cite{berk1972consistency} proved strong consistency for the maximum likelihood parameter estimates of exponential families under mild regularity conditions. In other wordd, if $\theta \in \Theta_{\g}$, then:
\[
\hat{\theta}_{\g,n}^\text{mle} \overset{\text{a.s.}\;}{\longrightarrow} \theta.
\]
Therefore, by the continuous mapping theorem:
\[
\ell(\hat{\theta}_{\g,n}^\text{mle} \mid x^1, \dots x^n) \overset{\text{a.s.}\;}{\longrightarrow} \ell(\theta \mid x^1, \dots x^n).
\]

\noindent Next, we note the following result due to \cite{dawid2011posterior}.

\begin{prop}[Theorem 1 \cite{dawid2011posterior}]
\label{prop:dawid}
Let $\g = (V,E)$ be an ADMG. Furthermore, let $X$ be a collection of random variables indexed by $V$ with probability measure $P_\theta$ that admits density $f(x \mid \theta)$ with respect to dominating $\sigma$-finite product measure $\nu$. If $x^1, \dots, x^n \overset{iid}{\sim} f(x \mid \theta)$, then\textup{:}
\[
\lim_{n \ra \infty} \log \frac{\int_{\hat{\theta} \in \Theta_{\g}} \prod_{i = 1}^n f(x^i \mid \hat{\theta}) \, \textup{d}\nu(\theta)}{\prod_{i = 1}^n f(x^i \mid \theta)} = - \lim_{n \ra \infty} n \inf_{\hat{\theta} \in \Theta_{\g}} \int_{x \in \mc X} \log \left[ \frac{f(x \mid \theta)}{f(x \mid \hat{\theta})} \right] \textup{d}P_\theta(x) + O_p(n^{\frac{1}{2}})
\]
where $\inf_{\hat{\theta} \in \Theta_{\g}} \int_{x \in \mc X} \log \left[ \frac{f(x \mid \theta)}{f(x \mid \hat{\theta})} \right] \textup{d}P_{\theta}(x) = 0$ if and only if $\theta \in \Theta_{\g}$.
\end{prop}

\begin{thm}
\label{thm:asympt}
Let $\g = (V, E)$ and $\g' = (V, E')$ be ADMGs and $\leq$ and $\leq'$ be total orders consistent with $\g$ and $\g'$ respectively. Furthermore, let $X$ be a collection of random variables indexed by $V$ with probability measure $P_\theta$ that admits density $f(x \mid \theta)$ with respect to dominating $\sigma$-finite product measure $\nu$. 

If $x^1, \dots, x^n \overset{iid}{\sim} f(x \mid \theta)$ and $\theta \in \Theta_{\g}$, then:
\[
\textup{BIC}_\textup{MF}(\g, \leq, x^1, \dots, x^n) = \textup{BIC}(\g, x^1, \dots, x^n) \; \text{almost surely}.
\]

If $x^1, \dots, x^n \overset{iid}{\sim} f(x \mid \theta)$ and $\theta \in \Theta_{\g'} \sm \Theta_{\g}$, then:
\[
\lim_{n \ra \infty} \frac{\left|\textup{Pr}(x^1, \dots, x^n \mid \g) - \exp \textup{BIC}_\textup{MF}(\g, \leq x^1, \dots, x^n) \right|}{\exp \textup{BIC}_\textup{MF}(\g', \leq', x^1, \dots, x^n)} = O_p(e^{-n})
\]
\end{thm}

\begin{proof}
If $\theta \in \Theta_{\g}$, then $\textup{BIC}_\textup{MF}(\g, \leq, x^1, \dots, x^n) = \textup{BIC}(\g, x^1, \dots, x^n)$ almost surely by the continuous mapping theorem and strong consistency of the maximum likelihood estimate. If $\theta \in \Theta_{\g'} \sm \Theta_{\g}$, then:
\begin{align*}
\lim_{n \ra \infty} & \log \frac{\left|\textup{Pr}(x^1, \dots, x^n \mid \g) - \exp \textup{BIC}_\textup{MF}(\g, \leq, x^1, \dots, x^n) \right|}{\exp \textup{BIC}_\textup{MF}(\g', \leq', x^1, \dots, x^n)} \\ 
&\leq \lim_{n \ra \infty} \log \frac{\max \left[ \text{Pr}(x^1, \dots, x^n \mid \g), \exp \textup{BIC}_\textup{MF}(\g, \leq, x^1, \dots, x^n) \right] }{\exp \textup{BIC}_\textup{MF}(\g', \leq', x^1, \dots, x^n)}.
\end{align*}
If $\exp \textup{BIC}_\textup{MF}(\g, \leq, x^1, \dots, x^n) > \text{Pr}(x^1, \dots, x^n \mid \g)$, then the results directly follows from Lemma \ref{lem:consistant_score}. If $\text{Pr}(x^1, \dots, x^n \mid \g) > \exp \textup{BIC}_\textup{MF}(\g, \leq, x^1, \dots, x^n)$, then the result directly follows from Proposition \ref{prop:dawid}.
\end{proof}

Accordingly, the BIC and $\textup{BIC}_\textup{MF}$ scores are almost surely equal when the probability measure is Markov with respect to the ADMG. Moreover, the difference between the log marginal likelihood and the $\textup{BIC}_\textup{MF}$ score tends towards zero at an exponential rate with respect to sample size relative to the $\textup{BIC}_\textup{MF}$ score for models where the probability measure is Markov when the probability measure is not Markov with respect to the ADMG.

% \begin{comment}

\section{Factorization of Graphs with Five Vertices}
\label{app:fac}

The graphs illustrated below are called maximally informative partial ancestral graphs (PAGs). Intuitively, a maximally informative PAG graphically represents the Markov equivalence class of a MAG (or an ADMG). A maximally informative PAG is not a mixed graph, but a graph that summarizes a set of mixed graphs. In addition to the standard set of edges used by mixed graphs, maximally informative PAGs also include edges with circle edge marks to denote ambiguity---the edge mark varies among the summarized graphs.

In what follows, we enumerate all MECs (up to vertex relabeling) for ADMGs with 5 vertices. Each MEC is represented graphically using a PAG and the corresponding m-connecting factorization is written next to the PAG. By expanding the conditional interaction information rates we can ensure that the corresponding sets are disjoint---this means that there is no need for an adjustment term in the factorization.

\begin{center}
\includegraphics[page=1]{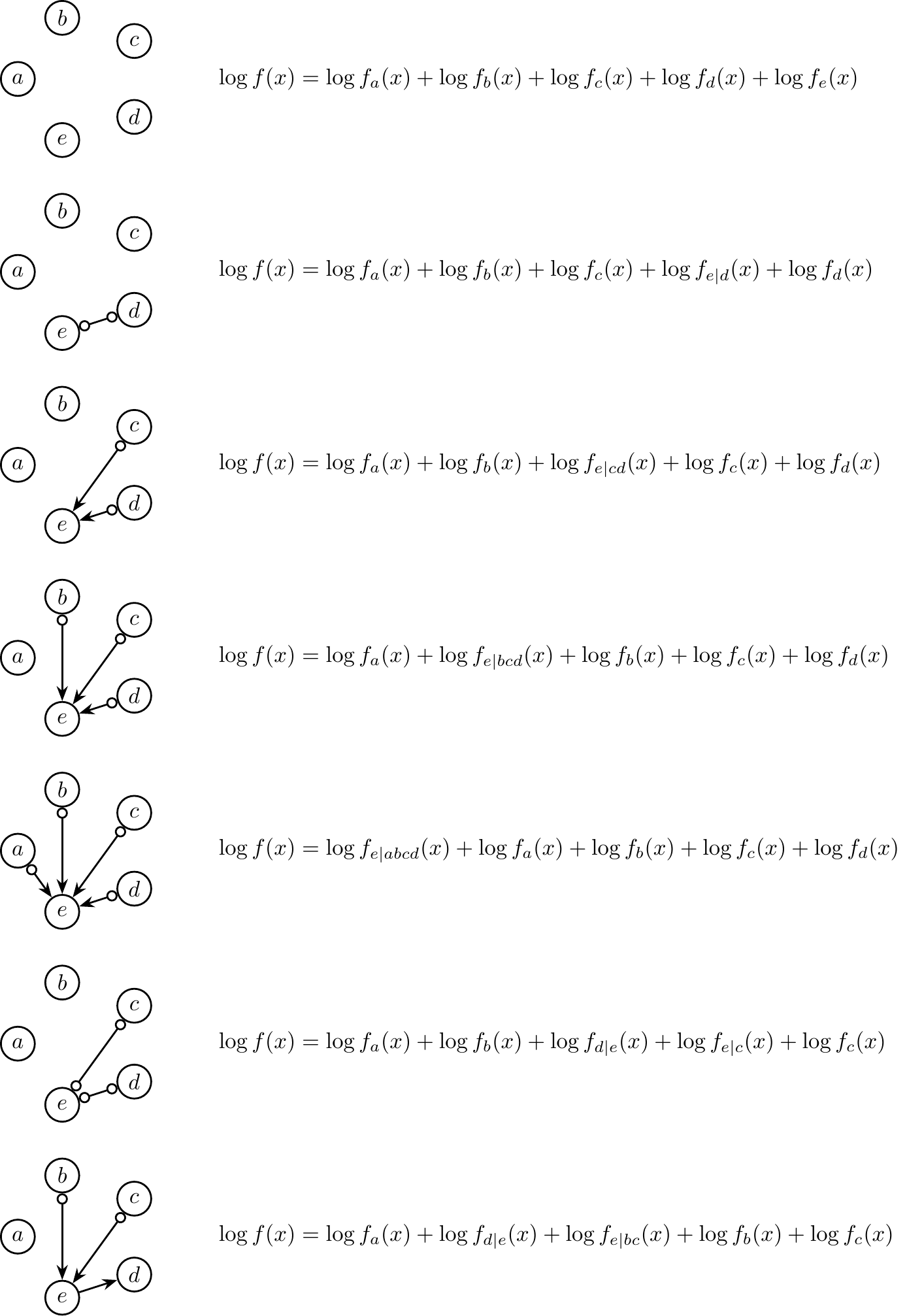}
\end{center}

\begin{center}
\includegraphics[page=2]{facts_5.pdf}
\end{center}

\begin{center}
\includegraphics[page=3]{facts_5.pdf}
\end{center}

\begin{center}
\includegraphics[page=4]{facts_5.pdf}
\end{center}

\begin{center}
\includegraphics[page=5]{facts_5.pdf}
\end{center}

\begin{center}
\includegraphics[page=6]{facts_5.pdf}
\end{center}

\begin{center}
\includegraphics[page=7]{facts_5.pdf}
\end{center}

\begin{center}
\includegraphics[page=8]{facts_5.pdf}
\end{center}

\begin{center}
\includegraphics[page=9]{facts_5.pdf}
\end{center}

\begin{center}
\includegraphics[page=10]{facts_5.pdf}
\end{center}

\begin{center}
\includegraphics[page=11]{facts_5.pdf}
\end{center}

\begin{center}
\includegraphics[page=12]{facts_5.pdf}
\end{center}

\begin{center}
\includegraphics[page=13]{facts_5.pdf}
\end{center}

\begin{center}
\includegraphics[page=14]{facts_5.pdf}
\end{center}

\begin{center}
\includegraphics[page=15]{facts_5.pdf}
\end{center}

\begin{center}
\includegraphics[page=16]{facts_5.pdf}
\end{center}

\begin{center}
\includegraphics[page=17]{facts_5.pdf}
\end{center}

\begin{center}
\includegraphics[page=18]{facts_5.pdf}
\end{center}

\begin{center}
\includegraphics[page=19]{facts_5.pdf}
\end{center}

\begin{center}
\includegraphics[page=20]{facts_5.pdf}
\end{center}

\begin{center}
\includegraphics[page=21]{facts_5.pdf}
\end{center}

\begin{center}
\includegraphics[page=22]{facts_5.pdf}
\end{center}

\begin{center}
\includegraphics[page=23]{facts_5.pdf}
\end{center}

\begin{center}
\includegraphics[page=24]{facts_5.pdf}
\end{center}

\begin{center}
\includegraphics[page=25]{facts_5.pdf}
\end{center}

\begin{center}
\includegraphics[page=26]{facts_5.pdf}
\end{center}

\begin{center}
\includegraphics[page=27]{facts_5.pdf}
\end{center}

\begin{center}
\includegraphics[page=28]{facts_5.pdf}
\end{center}

\begin{center}
\includegraphics[page=29]{facts_5.pdf}
\end{center}

\begin{center}
\includegraphics[page=30]{facts_5.pdf}
\end{center}

\begin{center}
\includegraphics[page=31]{facts_5.pdf}
\end{center}

\begin{center}
\includegraphics[page=32]{facts_5.pdf}
\end{center}

\begin{center}
\includegraphics[page=33]{facts_5.pdf}
\end{center}

\begin{center}
\includegraphics[page=34]{facts_5.pdf}
\end{center}

\begin{center}
\includegraphics[page=35]{facts_5.pdf}
\end{center}

\begin{center}
\includegraphics[page=36]{facts_5.pdf}
\end{center}

\begin{center}
\includegraphics[page=37]{facts_5.pdf}
\end{center}

\begin{center}
\includegraphics[page=38]{facts_5.pdf}
\end{center}

\begin{center}
\includegraphics[page=39]{facts_5.pdf}
\end{center}

\begin{center}
\includegraphics[page=40]{facts_5.pdf}
\end{center}

\begin{center}
\includegraphics[page=41]{facts_5.pdf}
\end{center}

\begin{center}
\includegraphics[page=42]{facts_5.pdf}
\end{center}

\begin{center}
\includegraphics[page=43]{facts_5.pdf}
\end{center}

\begin{center}
\includegraphics[page=44]{facts_5.pdf}
\end{center}

\begin{center}
\includegraphics[page=45]{facts_5.pdf}
\end{center}

% \end{comment}

\end{document}